\documentclass[a4paper,11pt]{article}
\usepackage{latexsym,amssymb,enumerate,amsmath,epsfig,amsthm,dsfont,bm}
\usepackage[margin=1in]{geometry}
\usepackage{setspace,color}
\usepackage{tikz}
\usepackage{floatrow}
\usepackage{multirow}
\usepackage{subfig}
\usepackage{graphicx}
\usepackage[ruled]{algorithm2e}
\usepackage{epstopdf}
\usepackage[multidot]{grffile}
\usepackage{comment}
\usepackage{multirow}
\newcommand{\x}{\mathbf{x}}
\newcommand{\p}{\mathbf{p}}
\newcommand{\n}{\mathbf{n}}

\newcommand{\bv}{\mathbf{v}}
\newcommand{\bc}{\mathbf{c}}
\newcommand{\mF}{\mathcal{F}}
\newcommand{\tEet}{\widetilde{E}_{2,\varepsilon}}

\newcommand{\de}{\delta_{\varepsilon}}
\DeclareMathOperator*{\argmin}{arg\,min}

\newcommand{\bflambda}{\bm{\lambda}}

\newtheorem{remark}{Remark}[section]

\newtheorem{thm}{Theorem}[section]
\newtheorem{prop}{Proposition}[section]

\begin{document}

\title{Curvature Regularized Surface Reconstruction from  Point Cloud}
\author{
Yuchen He\thanks{School of Mathematics, Georgia Institute of Technology, 686 Cherry Street, Atlanta, GA 30332-0160, USA. Email: {\bf yhe306@gatech.edu}}
\and
Sung Ha Kang\thanks{School of Mathematics, Georgia Institute of Technology, 686 Cherry Street, Atlanta, GA 30332-0160, USA. Email: {\bf kang@math.gatech.edu}}
\and
Hao Liu\thanks{The corresponding author. School of Mathematics, Georgia Institute of Technology, 686 Cherry Street, Atlanta, GA 30332-0160, USA. Email: {\bf hao.liu@math.gatech.edu}}
}

\maketitle

\begin{abstract}
We propose a variational functional with a curvature constraint to reconstruct implicit surfaces from point cloud data.  In the point cloud data, only  locations are assumed to be given, without any normal direction nor any curvature estimation.  The minimizing functional balances two terms,  the distance function from the point cloud to the surface and the mean curvature of the surface itself.   We explore both $L_1$ and $L_2$ norm for the curvature constraint.
With the added curvature constraint, the computation becomes particularly challenging.   
We propose two efficient algorithms. The first algorithm is a novel Operator Splitting Method (OSM).  It replaces the original high-order PDEs by a decoupled PDE system, which is solved by a semi-implicit method.  We also discuss an approach based on an Augmented Lagrangian Method (ALM).  The proposed model shows robustness against noise and recovers concave features and corners better compared to models without curvature constraint.  Numerical experiments in two and three dimensional data sets, noisy, and sparse data are presented to validate the model. Experiments show that the operator splitting semi-implicit method is flexible and robust.
\end{abstract}

\section{Introduction}

In industrial and scientific fields~\cite{gomes20143d,khan2018single}, surface reconstruction from point cloud data is a critical step in informative data visualization and successful high-level data processing.  Effective methods to reconstruct a continuous surface from finitely many points can reduce the burden in data transmission  and facilitate shape manipulations. One of the main goals of surface reconstruction is to render a meaningful and reliable surface which captures the geometrical features of the point cloud.
In this paper,  we use the celebrated level set function \cite{OS88} to represent the surface.  For  $d\in\mathbb{N}$, a $d$-dimensional implicit surface is represented by the set
\[ \Gamma= \{x\in\mathbb{R}^{d+1}\mid \phi(x)=0\},\]
for a level set function $\phi:\mathbb{R}^{d+1}\to\mathbb{R}$.  Implicit surfaces enjoy the flexibility in topological changes, and via $\phi$, one can easily derive and express geometric features of $\Gamma$, such as normals, mean curvature, and Gaussian curvature.

In \cite{ZOF01,zhao2000implicit}, the authors proposed the minimal surface model by interpreting the reconstructed surface as an elastic membrane attached to the given point cloud.   It finds the zero-level set surface $\Gamma$ that minimizes the following energy:
\begin{eqnarray}
  E_s(\Gamma)=\left(\int_{\Gamma} d^s(\x)d\sigma\right)^{\frac{1}{s}}=\left(\int_\Omega d^s(\x)|\nabla\phi(\x)|\delta(\phi(\x))\,d\x\right)^{\frac{1}{s}}\;.
  \label{eq.energy0}
\end{eqnarray}
Here, $d\sigma$ is the area element, $s>0$ is an exponent coefficient, $d(\x)=\inf_{\mathbf{y}\in\mathcal{D}}\{|\x-\mathbf{y}|\}$ measures the point-to-point-cloud distance, and $\mathcal{D}$ is the set of point cloud data. The energy~(\ref{eq.energy0}) minimizes the surface area weighted by the distance from surface to point cloud.

There are a number of related works using implicit surface reconstruction: a data-driven logarithmic prior for noisy data was considered in \cite{SK04}, surface tension was used to enrich the Euler-Lagrange equations in \cite{HM16}, and principal component analysis was used to reconstruct curves which are embedded in sub-manifolds in \cite{liu2017level}. Fast algorithms solving (\ref{eq.energy0}) was proposed in \cite{he2019fast}.  In \cite{LPZ13}, convexified image segmentation model with a fast algorithm was proposed for implicit surface reconstruction for point clouds. An efficient algorithm incorporating heat kernel convolution and thresholding was introdcued in \cite{wang2020efficient}. In \cite{EZL*12}, an efficient algorithm for level set method which  preserves distance function was proposed.
Open surface reconstruction using graph-cuts was proposed in \cite{wan2012reconstructing}, where reconstruction of open surface based on domain decomposition was also proposed.
 In \cite{LWQ09}, the authors proposed a variational model consisting of the distance, the normal direction, and the smoothness terms.  In \cite{lai2013ridge}, a ridge and corner preserving model based on vectorial TV regularization for surface restoration was introduced. In  \cite{LPZ13}, the authors defined the surface via a collection of anisotropic Gaussians centered at each entry of the input point cloud, and used TVG-L1 model for minimization.  A similar strategy addressing an $\ell_0$ gradient regularization model can be found in \cite{LLY*18}.
%

We propose a variational functional with a curvature constraint to improve surface reconstruction.   One such curvature constraint is the squared mean curvature, $\kappa^2$, such as Euler's elastica minimization model \cite{shen2003euler}.  In addition to image inpainting, it has been applied to denoising~\cite{Deng2019},
segmentation problem \cite{zhu2013image}, and others.
For any closed surface $\Gamma$ in $\mathbb{R}^3$, the bending energy,
\[ \int_\Gamma \kappa^2\,d\sigma, \] where $d\sigma$ denotes the surface area element, is a conformal invariant~\cite{blaschke1929topologische}, and it has a universal lower bound~\cite{willmore1968mean}:
$ \int_\Gamma \kappa^2\,d\sigma\geq 4\pi. $
%
Another curvature constraint we consider is the absolute mean curvature $|\kappa|$, i.e., $\int_\Gamma |\kappa|\,d\sigma$, which preserves sharp edges and corners in various cases, e.g.,  denoising \cite{Deng2019, zhu2012image},  and segmentation \cite{bae2017augmented}.
As a related work, in \cite{SWT*12}, graph cuts algorithm was explored for a functional with the absolute mean curvature term.
In~\cite{droske2010higher},  a variation using a function which is sensitive to large curvature was considered.
Other works used weighted mean curvature~\cite{gong2019weighted}, principle curvature~\cite{qiao2013principle}, Gaussian curvature~\cite{gong2013local}, Menger curvature~\cite{goldluecke2011introducing}, and other high-order geometrical information, e.g., conformal factor~\cite{LWCC10} and elastic ratio~\cite{schoenemann2011elastic}.

Optimizing a curvature regularized functional is a non-convex and non-linear problem. Computation of such functional is  particularly challenging.  There are a number of different approaches to design a fast and efficient algorithm, e.g., multigrid method~\cite{brito2008multigrid}, graph-cut algorithm~\cite{SWT*12},  homotopy method~\cite{yang2012homotopy} and convex relaxation~\cite{bredies2015convex,schoenemann2012linear}, just to name a few. A semi-implicit scheme  introduced in \cite{smereka2003semi}  simulates the curvature and surface diffusion motion of the interface.
One of the major class of methods is based on \textit{splitting}~\cite{glowinski1989augmented}. The common spirit of these methods is to cast the complicated primal problem into a series of more tamable subproblems, then to find the minimizer using alternative direction method.   There are various strategies to obtain such decompositions from the optimization problem.  One can derive the associated Euler-Lagrange equations, then apply operator splitting methods on the differential equations, e.g., Lie-Trotter method~\cite{trotter1959product}.   A new operator splitting algorithm was proposed for Euler's elastica model for image smoothing in \cite{Deng2019}.  One can also introduce auxiliary variables and transform the primal problem into a constrained one, then obtain a series of subproblems by alternatively optimizing one variable at a time while keeping the others fixed; e.g., augmented Lagrangian method (ALM)~\cite{bae2017augmented,tai2011fast,tai2009augmented,yashtini2016fast}. We refer the readers to \cite{glowinski2016some,glowinski2016admm,burger2016first,glowinski2019fast} for a detailed discussion on the splitting method in image processing.

In this paper,  we propose a variational functional with a curvature constraint to reconstruct implicit surfaces from point cloud data, and explore fast algorithms to solve the associated non-convex, non-linear optimization problem.  The minimizing functional balances two terms, the Euclidean distance from the point cloud to the surface and the mean curvature of the surface.
We show that the curvature term improves corner reconstruction and recovers non-convex features of the underlying shape of the point cloud data.   To avoid dealing with the high-order PDEs resulting from the gradient descent approach, we introduce a semi-implicit method to solve an easier, but equivalent, problem derived by the operator splitting method (OSM).  We also explore an ALM  method recently proposed by Bae et al.~\cite{bae2017augmented}, which reduces the number of parameters compared to other curvature regularized models.  Our approaches work effectively for 2D/3D cases, as well as for noisy and sparse point cloud.  The contributions of this paper are as follows:
 \begin{enumerate}
 \item{We propose a new minimal surface model with a curvature constraint, and develop efficient algorithms.  It is based on operator splitting and solved by a semi-implicit approach.  We also explore an augmented Lagrangian type algorithm and discuss the effects of the parameters involved.  }
\item{We explore and numerically compare results using  $L_1$ and $L_2$ norms of the mean curvature as the regularization term.  We compare the performance of the OSM approach with that of the ALM  approach.}
 \end{enumerate}

This paper is organized as follows. In Section~\ref{sec:model}, we introduce our curvature regularized minimal surface model and present fast algorithms for the proposed model.  In Section~\ref{sec:algo:split}, we introduce a new operator splitting scheme.  This approach transforms the original high-order PDE into an easier differential system, which is solved by a semi-implicit method.  In Section~\ref{sec:algo:ADMM}, we describe an ALM method for the model, followed by the numerical details in Section~\ref{sec:numD}.  Various numerical experiments are presented in Section~\ref{sec:Exp}, and analytical discussion is presented in  Section~\ref{sec:anal}.  We conclude this paper  in Section~\ref{sec:con}.

\section{Curvature Regularized Surface Reconstruction Model}\label{sec:model}

Let $\mathcal{D}$ be the set of given point cloud data, $d(\x)=\inf_{\mathbf{y}\in\mathcal{D}}\{|\x-\mathbf{y}|\}$ measures the point-to-point-cloud distance, and $d\sigma$ is the area element.
To reconstruct a surface from a given point cloud data $\mathcal{D}$, we propose the following curvature-constrained minimal surface energy:
\begin{eqnarray}
  E_s(\Gamma)=\left(\int_{\Gamma} |d(\x)|^{s} d\sigma\right)^{\frac{1}{s}}+\eta \left(\int_{\Gamma} |\kappa(\x)|^{s}d\sigma\right)^{\frac{1}{s}}\;.
  \label{eq.energy}
\end{eqnarray}
Here, $\kappa$ is the mean curvature of $\Gamma$, and the exponent coefficient $s>0$ is a constant integer.  We explore the cases when $s=1$ and $s=2$.  The first term, the surface integral of the distance from point cloud to the surface, signifies the fidelity of reconstruction.   It moves the surface $\Gamma$ towards the point cloud.
The second term, which is the integral of the surface mean curvature along the reconstructed surface, is the regularization.  This induces regularized geometric features for $\Gamma$ independent to the point cloud location.  Geometric features can include sharp corners, smooth corners, or straight segments, depending on the choice of $s$.  The parameter $\eta>0$ controls the influence of the curvature regularization.   When $\eta=0$, the model (\ref{eq.energy}) degenerates to the minimal surface model proposed in \cite{ZOF01}.
\begin{remark} The $1/s$ power in (\ref{eq.energy}) comes from the original model (\ref{eq.energy0}).  If one takes the distance function as the potential function of the point cloud, then the energy is an $L_s$ norm of the potential on $\Gamma$.  Based on this, we add the regularization term related to the mean curvature of $\Gamma$, and the $1/s$ power is used to keep the two terms in the same format.  One may remove this power to get a simpler energy and  still get the same minimizer when $\eta$ is chosen appropriately.
\end{remark}

We employ the implicit surface representation~\cite{osher2004level} to rewrite the energy~\eqref{eq.energy}, and the level set function $\phi$ is defined such that $\Gamma$ is its zero level set:
\[
  \phi(\x) \mbox{ is } \begin{cases}
    >0,& \mbox{ if } \x \mbox{ is outside of }\Gamma,\\
    =0,& \mbox{ if } \x \mbox{ is on }\Gamma,\\
    <0, & \mbox{ if } \x \mbox{ is inside } \Gamma.
  \end{cases}
\]
Using $\phi$,  the functional $E_s(\Gamma)$ restricted to the surface $\Gamma$ can be expressed as
\begin{align}
  E_s(\phi)=\left(\int_{\Omega} |d(\x)|^s \delta(\phi)|\nabla\phi|d\x\right)^{\frac{1}{s}}+\eta \left(\int_{\Omega} |\kappa(\x)|^s\delta(\phi)|\nabla\phi|d\x\right)^{\frac{1}{s}},
  \label{eq.energy.level}
\end{align}
where $\delta(\phi)=H'(\phi)$ is the Dirac Delta function with $H$ being the Heaviside step function:
$H (\phi)=1$ if $\phi>0$, and 0 otherwise.   We use the smooth approximation \cite{bae2017augmented} for practical computation
\begin{equation}
	H_\varepsilon(\phi) = \frac{1}{2}+\frac{1}{\pi}\arctan\left(\frac{\phi}{\varepsilon}\right) \;\; \text{ and } \;\;
	\delta_{\varepsilon}(\phi)= H^\prime_\varepsilon(\phi) = \frac{\varepsilon}{\pi(\varepsilon^2+\phi^2)} \;,
\label{eq.H.eps}
\end{equation}
with $\varepsilon>0$,  which is a constant controlling the smoothness.  For any point $\x$ on $\Gamma$, its mean curvature can be computed as,
$\kappa(\x)=\nabla\cdot \left(\frac{\nabla \phi(\x)}{|\nabla\phi(\x)|}\right)$.
Putting these together,  we focus on the smoothed energy
\begin{eqnarray}
  E_{s,\varepsilon}(\phi)=\left(\int_{\Omega} |d(\x)|^s \frac{\varepsilon}{\pi(\varepsilon^2+\phi^2)}|\nabla\phi|d\x\right)^{\frac{1}{s}}+\eta \left(\int_{\Omega} \left|\nabla\cdot \left(\frac{\nabla \phi}{|\nabla\phi|}\right)\right|^s\frac{\varepsilon}{\pi(\varepsilon^2+\phi^2)}|\nabla\phi|d\x\right)^{\frac{1}{s}},
  \label{eq.energy.2}
\end{eqnarray}
and the reconstructed surface is defined as the minimizer of the energy (\ref{eq.energy.2}), i.e.,
\[
  \Gamma=\{\x| \psi(\x)=0\} \;\;\; \mbox{ for } \;\; \psi=\argmin_{\phi} E_{s,\varepsilon}(\phi).
\]
Here   $\psi$ represents the optimal level set function.

\subsection{Operator Splitting Method}\label{sec:algo:split}

One of our main challenges is that $E_{s}$ in (\ref{eq.energy.level}) is highly nonlinear in terms of $\phi$, and the corresponding Euler-Lagrange equation is a high-order nonlinear PDE. See (\ref{eq.vard}) and (\ref{eq.vark}) in Section \ref{sec:anal}. To circumvent this difficulty, we propose a new operator splitting strategy, which leads to an equivalent differential equation system that is much easier to solve.

We follow the direction of gradient flow; however,  we first decouple the data fidelity term and the curvature regularization term, then minimize the simplified functional via its gradient flow.
For $s=2$, using (\ref{eq.energy.2}), we rewrite the energy~(\ref{eq.energy}) in the following equivalent form:
\begin{equation}
  \begin{cases}
    \tEet(\phi)=\displaystyle{\left(\int_{\Omega} d^2(\x) \delta_{\varepsilon}(\phi)|\nabla\phi|d\x\right)^{\frac{1}{2}}+\eta \left(\int_{\Omega} q^2(\x)\delta_{\varepsilon}(\phi)|\nabla\phi|d\x\right)^{\frac{1}{2}}},\\
    \displaystyle{ q=\nabla\cdot\frac{\nabla\phi}{|\nabla\phi|} },
  \end{cases}
  \label{eq.splitEner}
\end{equation}
with the notation $\delta_{\varepsilon}(\phi)$ as in (\ref{eq.H.eps}).
We then compute the variation of $\tEet(\phi)$ with respect to $\phi$ \cite{ZOF01}. For $\forall v\in H^2$, $H^2$ denoting the Sobolev space, we have
\begin{eqnarray*}
 \Bigg \langle \frac{\partial \tEet(\phi)}{\partial\phi},v \Bigg \rangle
 &= & -\int_{\Omega}\frac{1}{2}\de(\phi)\left[ \int_{\Omega} d^2(\x)\de(\phi)|\nabla\phi|d\x\right]^{-1/2} \nabla\cdot \left[ d^2(\x)\frac{\nabla\phi}{|\nabla \phi|}\right]vd\x \nonumber\\
 & - & \eta \int_{\Omega}\frac{1}{2}\de(\phi)\left[ \int_{\Omega} q^2(\x)\de(\phi)|\nabla\phi|d\x\right]^{-1/2} \nabla\cdot \left[ q^2(\x)\frac{\nabla\phi}{|\nabla \phi|}\right]vd\x.
\end{eqnarray*}
If $\psi$ is a minimizer of $\tEet$, it satisfies the optimality condition
\begin{equation}
\displaystyle{    \Bigg \langle \frac{\partial \tEet(\psi)}{\partial\phi},v  \Bigg \rangle=0, \;\;\;
    q-\nabla\cdot\frac{\nabla\psi}{|\nabla\psi|}=0,}  \;\;\;  \forall v\in H^2.
  \label{eq.euler}
\end{equation}
To solve for $\psi$, we associate (\ref{eq.euler}) with the initial value problem
\begin{equation}
  \begin{cases}
   \displaystyle{ \frac{\partial \phi}{\partial t}=f(d,\phi) \nabla\cdot \left[ d^2(\x)\frac{\nabla\phi}{|\nabla \phi|}\right] +\eta f(q,\phi) \nabla\cdot \left[ q^2(\x)\frac{\nabla\phi}{|\nabla \phi|}\right],}\\
    \displaystyle{ \frac{\partial q}{\partial t}+\gamma\left(q-\nabla\cdot\frac{\nabla\phi}{|\nabla\phi|}\right)=0,}
  \end{cases}
  \label{eq.ivp}
\end{equation}
with
$$
f(d,\phi)=\frac{1}{2}\de(\phi)\left[ \int_{\Omega} d^2(\x)\de(\phi)|\nabla\phi|d\x\right]^{-1/2}.
$$
The steady state of (\ref{eq.ivp}) is a minimizer of $\tEet$ in (\ref{eq.splitEner}). On the right hand side of the first equation in (\ref{eq.ivp}), the two terms are of the same form, only differing by $\eta$, $d$ and $q$. The first term is the driving velocity to minimize the squared distance from the surface to the given data. The second term is the driving velocity to minimize the squared curvature along the reconstructed surface. The parameter $\eta$ controls the trade-off between these two terms.


We adopt the Lie type of operator splitting and  refer the readers to \cite{glowinski2017splitting} for a complete discussion of different splitting schemes. Given $\{\phi^k,q^k\}$ at the $k$-th step, we update $\{\phi^{k+1},q^{k+1}\}$ in two fractional steps. In particular, for $k>0$,  we update the variables through $\{\phi^k,q^k\}\rightarrow \{\phi^{k+1/2},q^{k+1/2}\}\rightarrow\{\phi^{k+1},q^{k+1}\}$ as follows:

\underline{\emph{Fractional step 1}}: Solve
\begin{equation}
\begin{cases}
 \displaystyle{  \frac{\partial \phi}{\partial t}=f(d,\phi) \nabla\cdot \left[ d^2(\x)\frac{\nabla\phi}{|\nabla \phi|}\right] +\eta f(q,\phi) \nabla\cdot \left[ q^2(\x)\frac{\nabla\phi}{|\nabla \phi|}\right] \mbox{ on } \Omega\times[t^k,t^{k+1}],}\\
\displaystyle{   \frac{\partial q}{\partial t}=0 \mbox{ on } \Omega\times[t^k,t^{k+1}],}\\
  \phi(t^k)=\phi^k,q(t^k)=q^k
  \label{eq.split.1}
  \end{cases}
\end{equation}
and set $\phi^{k+1/2}=\phi(t^{k+1}),q^{k+1/2}=q(t^{k+1})$. \\

\underline{\emph{Fractional step 2}}: Solve
\begin{equation}
\begin{cases}
\displaystyle{   \frac{\partial \phi}{\partial t}=0 \mbox{ on }\Omega\times[t^k,t^{k+1}]\;,}\\
 \displaystyle{  \frac{\partial q}{\partial t}+\gamma\left(q-\nabla\cdot\frac{\nabla\phi^{k+1/2}}{|\nabla\phi^{k+1/2}|}\right)=0 \mbox{ on } \Omega\times[t^k,t^{k+1}]\;,}\\
  \phi(t^k)=\phi^{k+1/2}, q(t^k)=q^{k+1/2}
  \end{cases}
  \label{eq.split.2}
\end{equation}
and set $\phi^{k+1}=\phi(t^{k+1}),q^{k+1}=q(t^{k+1})$.
\vspace{0.5cm}

We have two subproblems (\ref{eq.split.1}) and (\ref{eq.split.2}) to address. There is no difficulty to solve (\ref{eq.split.2}), since we have the closed form solution
\[
  q=e^{\gamma\Delta t}q^{k+1/2}+(1-e^{\gamma\Delta t})\nabla\cdot\frac{\nabla\phi^{k+1/2}}{|\nabla\phi^{k+1/2}|}\;.
\]
To solve (\ref{eq.split.1}) for $\phi^{k+1/2}$, the simplest way is to use the explicit scheme as the following:
\[
  \frac{\phi^{k+1/2}-\phi^k}{\Delta t}=f(d,\phi^k) \nabla\cdot \left[ d^2(\x)\frac{\nabla\phi^k}{|\nabla \phi^k|}\right] +\eta f(q^k,\phi^k) \nabla\cdot \left[ (q^k)^2(\x)\frac{\nabla\phi^k}{|\nabla \phi^k|}\right]\;.
\]
However, due to the stability consideration, one needs to choose a very small time step of order $O(h^2)$ where $h$ is the spatial step size.  To relax the time step constraint,  for some $\alpha>0$, we add $-\alpha\Delta\phi$ on both sides of (\ref{eq.split.1}), as in \cite{smereka2003semi}, to get
\begin{equation}
   \frac{\partial \phi}{\partial t}-\alpha\Delta\phi=-\alpha\Delta\phi+ f(d,\phi) \nabla\cdot \left[ d^2(\x)\frac{\nabla\phi}{|\nabla \phi|}\right] +\eta f(q,\phi) \nabla\cdot \left[ q^2(\x)\frac{\nabla\phi}{|\nabla \phi|}\right].
   \label{eq.split.1add}
\end{equation}
We discretize (\ref{eq.split.1add}) in time semi-implicitly:
\[
  \frac{\phi^{k+1/2}-\phi^k}{\Delta t}-\alpha\Delta\phi^{k+1/2}=-\alpha\Delta\phi^k+f(d,\phi^k) \nabla\cdot \left[ d^2(\x)\frac{\nabla\phi^k}{|\nabla \phi^k|}\right] +\eta f(q^k,\phi^k) \nabla\cdot \left[ (q^k)^2(\x)\frac{\nabla\phi^k}{|\nabla \phi^k|}\right].
\]
We fix $\alpha=1$ in this paper. This equation is a Laplacian equation of $\phi^{k+1/2}$ and can be solved efficiently by fast Fourier transformation (FFT).
The updating formula is summarized as
\begin{equation}
  \begin{cases}
   \displaystyle{  \frac{\phi^{k+1}-\phi^k}{\Delta t}-\alpha\Delta\phi^{k+1}=-\alpha\Delta\phi^{k}+f(d,\phi^k) \nabla\cdot \left[ d^2(\x)\frac{\nabla\phi^k}{|\nabla \phi^k|}\right] +\eta f(q^k,\phi^k) \nabla\cdot \left[ (q^k)^2(\x)\frac{\nabla\phi^k}{|\nabla \phi^k|}\right],}\\
    \displaystyle{    q^{k+1}=e^{\gamma\Delta t}q^{k}+(1-e^{\gamma\Delta t})\nabla\cdot\frac{\nabla\phi^{k+1}}{|\nabla\phi^{k+1}|}\;.}
  \end{cases}
  \label{eq.split}
\end{equation}

To solve (\ref{eq.split}), we need the initial condition $(\phi^0,q^0)$. We choose $\phi^0$ to be a signed distance function whose zero level set encloses all data. $q^0$ is assigned as $\nabla\cdot ((\nabla\phi^0)/|\nabla\phi^0|)$.
The algorithm of OSM with $s=2$ is stated in Algorithm \ref{alg:OSM}.
In Algorithm \ref{alg:OSM}, the reinitialization is used to keep $\phi$ to be a signed-distance function near its zero level set. The details of reinitialization are discussed in Section \ref{sec:numD}.
\begin{algorithm}
	\SetKwInOut{KwIni}{Initialization}
	\KwIni{$d$, $\phi^0, q^0$.}

	\While {not converge}{
		Update $\{\phi^{k+1},q^{k+1}\}$ by solving (\ref{eq.split}).\\
        Reinitialize $\phi^{k+1}$.
        }

	\KwOut{$\phi^{k}$.}
	\caption{Operator Splitting Method (OSM) for $s=2$.}\label{alg:OSM}
\end{algorithm}

In the following, we give a brief derivation of OSM for $s=1$. With $s=1$, $E_{1,\varepsilon}$ has the same minimizer as
\[
  \begin{cases}
   \displaystyle{ \widetilde{E}_{1,\varepsilon}(\phi)=\int_{\Omega} d(\x) \delta_{\varepsilon}(\phi)|\nabla\phi|d\x+\eta \int_{\Omega} |q(\x)|\delta_{\varepsilon}(\phi)|\nabla\phi|d\x\;,}\\
    \displaystyle {q=\nabla\cdot\frac{\nabla\phi}{|\nabla\phi|}}\;,
  \end{cases}
\]
whose
corresponding gradient flow initial value problem is
\begin{equation}
  \begin{cases}
   \displaystyle{  \frac{\partial \phi}{\partial t}= \delta_{\varepsilon}(\phi)\nabla\cdot \left[ d(\x)\frac{\nabla\phi}{|\nabla \phi|}\right] +\eta  \delta_{\varepsilon}(\phi)\nabla\cdot \left[ |q(\x)|\frac{\nabla\phi}{|\nabla \phi|}\right],}\\
 \displaystyle{    \frac{\partial q}{\partial t}+\gamma\left(q-\nabla\cdot\frac{\nabla\phi}{|\nabla\phi|}\right)=0\;.}
  \end{cases}
  \label{eq.ivp1}
\end{equation}
With the Lie type splitting in time and introducing the term $-\alpha\Delta\phi$ on both sides in the first equation of (\ref{eq.ivp1}), we get the updating formula
\begin{equation}
  \begin{cases}
  \displaystyle{   \frac{\phi^{k+1}-\phi^k}{\Delta t}-\alpha\Delta\phi^{k+1}=-\alpha\Delta\phi^k+ \delta_{\varepsilon}(\phi)\nabla\cdot \left[ d(\x)\frac{\nabla\phi^k}{|\nabla \phi^k|}\right] +\eta \delta_{\varepsilon}(\phi) \nabla\cdot \left[ |q^k(\x)|\frac{\nabla\phi^k}{|\nabla \phi^k|}\right],}\\
\displaystyle{     q^{k+1}=e^{\gamma\Delta t}q^{k}+(1-e^{\gamma\Delta t})\nabla\cdot\frac{\nabla\phi^{k+1}}{|\nabla\phi^{k+1}|}\;.}
  \end{cases}
  \label{eq.split1}
\end{equation}

We note that (\ref{eq.splitEner}) can be solved similarly by $\kappa$TV method proposed in \cite{yashtini2016fast} for image inpainting problem.  One advantage of OSM in this paper is its simplicity and having less parameters: once the model parameter, i.e., $\eta$, is fixed, there is only one parameter, the artificial time step, to tune. Numerical experiments show that there is a wide range of the time step we can choose.

\subsection{Augmented Lagrangian Method}\label{sec:algo:ADMM}
We present another efficient algorithm to find the minimizer of (\ref{eq.energy.2}) with $s=1$. The reason for only focusing on $s=1$ in this case is that we can take advantage of the shrinkage operator. We first introduce three new variables: $\p=\nabla \phi$, $\n=\nabla \phi/|\nabla \phi|$ and $q=\nabla \cdot (\nabla\phi/|\nabla \phi|)$.   Finding the minimizer of $E_{1,\varepsilon}$ is equivalent to solving
\begin{align}
  \underset{\phi,\p,\n,q}{\min} \int \varepsilon \frac{(d(\x)+\eta|q|)|\p|}{\pi(\varepsilon^2+\phi^2)} d\x  \;\;\;
  \mbox{with }\;\; \p=\nabla \phi,\;\;  \n=\nabla \phi/|\nabla \phi|, \;\;  q=\nabla \cdot (\nabla\phi/|\nabla \phi|)\;.\nonumber
  \label{eq.admm.0}
\end{align}

This can be addressed via alternating direction method of multipliers by introducing Lagrange multipliers $\bflambda_1,\lambda_2,\bflambda_3$. The associated Augmented Lagrangian functional is
\begin{equation}
\begin{array}{rl}
&\mathcal{L}(\phi,q,\p,\n,\bflambda_1,\lambda_2,\bflambda_3) \\ & \\
& \displaystyle{=\int_\Omega \frac{\varepsilon(d+\eta|q|)|\mathbf{p}|}{\pi(\varepsilon^2+\phi^2)}\,dx
+\frac{r_1}{2}\int_\Omega|\mathbf{p}-\nabla\phi|^2\,dx+\int_\Omega\bflambda_1\cdot(\mathbf{p}-\nabla\phi)\,dx}\\
&\\
& \displaystyle{+\frac{r_2}{2}\int_\Omega(q-\nabla\cdot\n)^2\,dx+\int_\Omega\lambda_2(q-\nabla\cdot\n)\,dx
+\frac{r_3}{2}\int_\Omega||\p|\n-\p|^2\,dx+\int_\Omega\bflambda_3\cdot(|\p|\n-\p)\;}
\end{array}
\label{eq:almModel}
\end{equation}
where $\p,\n,\bflambda_1,\bflambda_3$ are vectors, $\phi,q,\lambda_2$ are scalars, $r_1,r_2,r_3$ are fixed constants. To find the saddle point of $\mathcal{L}$, we update each variable in an alternative manner. In each iteration, for each variable, we minimize the corresponding functional while keeping other variables fixed. After all variables are updated, we update Lagrange multipliers. This procedure is repeated until we achieve a steady state. In each iteration, we have four subproblems to minimize:
\begin{align}
 \mathcal{E}_1(\phi)&=\int_\Omega \frac{\varepsilon(d+\eta|q|)|\mathbf{p}|}{\pi(\varepsilon^2+\phi^2)}dx +\frac{r_1}{2}\int_\Omega|\mathbf{p}-\nabla\phi|^2dx\;,  +\int_\Omega\bflambda_1\cdot(\mathbf{p}-\nabla\phi)dx\label{eq.sub1}\\
  \mathcal{E}_2(q)&=\int_\Omega \frac{\varepsilon(d+\eta|q|)|\mathbf{p}|}{\pi(\varepsilon^2+\phi^2)}\,dx +\frac{r_2}{2}\int_\Omega(q-\nabla\cdot\n)^2\,dx +\int_\Omega\lambda_2(q-\nabla\cdot\n)\,dx\;,\label{eq.sub2}\\
  \mathcal{E}_3(\p)&=\int_\Omega \frac{\varepsilon(d+\eta|q|)|\mathbf{p}|}{\pi(\varepsilon^2+\phi^2)}\,dx +\frac{r_1}{2}\int_\Omega|\mathbf{p}-\nabla\phi|^2\,dx+\int_\Omega\bflambda_1\cdot(\mathbf{p}-\nabla\phi)\,dx\;,\nonumber\\
  &+\frac{r_3}{2}\int_\Omega||\p|\n-\p|^2\,dx+\int_\Omega\bflambda_3\cdot(|\p|\n-\p)\;,\label{eq.sub3}\\
  \mathcal{E}_4(\n)&=\frac{r_2}{2}\int_\Omega(q-\nabla\cdot\n)^2\,dx +\int_\Omega\lambda_2(q-\nabla\cdot\n)\,dx +\frac{r_3}{2}\int_\Omega||\p|\n-\p|^2\,dx\nonumber\\
  &+\int_\Omega\bflambda_3\cdot(|\p|\n-\p)\;.\label{eq.sub4}
\end{align}
After those four variables being updated correspondingly, Lagrange multipliers are updated as
\[ \bflambda_1 \gets \bflambda_1+r_1(\p-\nabla\phi), \;\;\;
\lambda_2 \gets \lambda_2+r_2(q-\nabla\cdot\n),\;\;\;
\bflambda_3 \gets \bflambda_3+r_3(\n|\p|-\p)\;.
\]

These subproblems can be solved efficiently as described in the following.

\noindent\underline{\emph{Subproblem of $\phi$:}} For $\mathcal{E}_1(\phi)$ in (\ref{eq.sub1}), the corresponding Euler-Lagrange equation is:
\[
-r_1\Delta\phi+\beta\phi=\beta\phi+(d+\eta|q|)\frac{2\varepsilon|\p|\phi}{\pi(\varepsilon^2+\phi^2)^2}	-\nabla\cdot(r_1\p+\bflambda_1)
\;,\]
where $\beta>0$ is a frozen coefficient. We discretize the time as follows
\begin{align}
-r_1\Delta\phi^{k+1}+\beta\phi^{k+1}=\beta\phi^k+(d+\eta|q^k|)\frac{2\varepsilon|\p^k|\phi^k}{\pi(\varepsilon^2+(\phi^k)^2)^2}	-\nabla\cdot(r_1\p^k+\bflambda_1^k)\;.
\label{eq.sub1.2}
\end{align}
This is the Laplacian equation of $\phi^{k+1}$, and we efficiently solve it  by FFT.\\

\noindent\underline{\emph{Subproblem of $q$:}}  In (\ref{eq.sub2}), $\mathcal{E}_2(\phi)$ can be written as:
\[
\mathcal{E}_2(q) = \int_\Omega\frac{\eta\varepsilon|\p|}{\pi(\varepsilon^2+\phi^2)}|q|+\frac{r_2}{2}\left( q-\left(\nabla\cdot\n-\frac{\lambda_2}{r_2}\right)\right)^2\,dx+C\;,	
\]
where $C$ is independent of $q$. Then, the minimizer can be found via the shrinkage operator
\begin{align}
\arg\min_q\mathcal{E}_2(q)= \max\left\{0,1-\frac{\eta\varepsilon|\p|}{r_2\pi(\varepsilon^2+\phi^2)|q^*|}\right\}q^*\;,\label{eq.shrink}
\end{align}
with $q^*=\nabla\cdot \n-\lambda_2/r_2$. \\

\noindent\underline{\emph{Subproblem of $\p$:}}  In (\ref{eq.sub3}), $\mathcal{E}_3(\phi)$ can be rewritten as
\begin{align*}
\mathcal{E}_3(\p)&=\int_\Omega\underbrace{\left[(d+\eta|q|)\frac{\varepsilon}{\pi(\varepsilon^2+\phi^2)} +\bflambda_3\cdot\n\right]}_{\omega}|\p|+\underbrace{\frac{r_1+r_3(1+|\n|^2)}{2}}_{\mu} \bigg|\p-\underbrace{\frac{\bflambda_3+r_1\nabla\phi-\bflambda_1}{r_1+r_3(1+|\n|^2)}}_{\mathbf{a}}\bigg|^2\nonumber\\
&-\int_\Omega \underbrace{r_3\n}_{\bm{\nu}}\cdot\p|\p|+\widetilde{C}\;,
\end{align*}
where $\widetilde{C}$ is independent of $\p$. This $\mathcal{E}_3(\p)$ can be simplified as
\[
  \mathcal{E}_3(\p)=\int_\Omega \omega|\p| + \frac{\mu}{2}|\p-\mathbf{a}|^2 -\bm{\nu}\cdot\p|\p|+\widetilde{C}\;.
\]
Following the idea of Theorem 2 in \cite{bae2017augmented}, we can minimize this energy efficiently.
\begin{thm}
\label{thm.p}
Assume that $\mu>2|\bm{\nu}|$. Let $\theta$ be the angle between $\mathbf{a}$ and the minimum vector of $\mathcal{E}_3(\p)$, and $\alpha$ is the angle between $\mathbf{a}$ and $\nu$. Then the following arguments hold:
\begin{itemize}
	\item if $\omega\geq \mu|\mathbf{a}|$, then $\arg\min_{\p}\mathcal{E}_3(\p)=\mathbf{0}$;
	\item if $\omega<\mu|\mathbf{a}|$:
	\begin{enumerate}
		\item if $\mathbf{a}=\bm{\nu}=\mathbf{0}$, then $\arg\min_\p\mathcal{E}_3(\p)=\begin{cases}
\mathbf{0}\;,\;\mbox{when }\omega\geq0,\\
\text{any vector of length}~-\omega/\mu\;,\;\mbox{when }\omega<0;	
\end{cases}
$
	\item if $\mathbf{a}\neq \mathbf{0},\bm{\nu}=\mathbf{0}$, $\arg\min_\p\mathcal{E}_3(\p)=(1-\frac{\omega}{\mu|\mathbf{a}|})\mathbf{a}$;
	\item if $\mathbf{a}=\mathbf{0},\bm{\nu}\neq\mathbf{0}$,  $\arg\min_\p\mathcal{E}_3(\p)=\frac{\omega}{\mu-2|\bm{\nu}|} \frac{\bm{\nu}}{|\bm{\nu}|}$;
	\item if $\mathbf{a}\neq\mathbf{0},\bm{\nu}\neq \mathbf{0}$, the angles $\theta$ and $\alpha$ satisfy the equation:
	\begin{align}
	\mu^2|\mathbf{a}|\sin\theta+\mu|\bm{\nu}||\mathbf{a}|\sin\theta\cos(\theta-\alpha) +\omega|\bm{\nu}|\sin(\theta-\alpha) +\mu|\mathbf{a}||\bm{\nu}|\sin\alpha=0\;,\label{eq.theta}
	\end{align}
	and $\arg\min_\p\mathcal{E}_3(\p)=\frac{[\mu(\mathbf{b}\cdot\mathbf{a})-\omega]\mathbf{b}}{\mu+2\bm{\nu}\cdot\mathbf{b}}$ with $\mathbf{b}$ being a unit vector satisfying:
	\begin{align*}
	\mathbf{b}=\frac{1}{|\mathbf{a}|}\begin{bmatrix}
	\cos\widetilde{\theta}&-\sin\widetilde{\theta}\\
	\sin\widetilde{\theta}&\cos\widetilde{\theta}
\end{bmatrix}\mathbf{a}\;,
	\end{align*}
 and $\widetilde{\theta}=\theta$ if $\det[\bm{\nu}~\mathbf{a}]\geq 0$, $\widetilde{\theta}=-\theta$ if $\det[\bm{\nu}~\mathbf{a}]< 0$. Here $[\bm{\nu}~\mathbf{a}]$ denotes the $2\times 2$ matrix with the vector $\bm{\nu}$ and $\mathbf{a}$ being the first and second column respectively.
	\end{enumerate}
\end{itemize}
\end{thm}
Note that the condition in Theorem \ref{thm.p} is always satisfied since $\mu=\frac{r_1+r_3(1+|\n|^2)}{2},\nu=r_3\n$ and $ r_3(1+|\n|^2)\geq 2r_3|\n|$ for any $\n$.  From (\ref{eq.theta}), $\theta$ is solved by Newton's method. \\

\noindent\underline{\emph{Subproblem of $\n$.}} For $\mathcal{E}_4(\phi)$ in (\ref{eq.sub4}), the Euler-Lagrange equation  is:
\begin{equation}
-r_2\nabla(\nabla\cdot\n)+D\n=(D-r_3|\p|^2)\n-\nabla(r_2q+\lambda_2)-(\bflambda_3-r_3\p)|\p|\;,
\label{eq.sub4.1}
\end{equation}
where $D=\max_{x\in\Omega}(r_3|\p|^2+\beta_2)$ and $\beta_2$ is a small positive number. We discretize in time as
\begin{align}
-r_2\nabla(\nabla\cdot\n^{n+1})+D\n^{n+1}=(D-r_3|\p^n|^2)\n^n-\nabla(r_2q^n+\lambda_2^n)-(\bflambda_3^n-r_3\p^n)|\p^n|\;\label{eq.sub4.2}	
\end{align}
which can be solved efficiently by FFT.

For the initial condition, we use the same $\phi^0$ and $q^0$ as that in OSM. For other variables, we use $\p^0=\nabla\phi^0, \n^0=\p^0/|\p^0|, \bflambda_1^0=\bflambda_3^0=\mathbf{0}, \lambda_2^0=0$.

The outline of augmented Lagrangian is summarized in Algorithm \ref{alg:ALM}.

\begin{algorithm}
	\SetKwInOut{KwIni}{Initialization}
	\KwIni{$d$, $\phi^0,q^0,\p^0,\n^0,\bflambda_1^0,\lambda_2^0,\bflambda_3^0$.}
	\While {not converge}{
        \textbf{Update variables}\\
		Update $\phi^{k+1}=\arg\min_{\phi}\mathcal{L}(\phi,q^k,\p^k,\n^k,\bflambda_1^k,\lambda_2^k,\bflambda_3^k)$ by solving (\ref{eq.sub1.2}).\\
        Update $q^{k+1}=\arg\min_{q}\mathcal{L}(\phi^k,q,\p^k,\n^k,\bflambda_1^k,\lambda_2^k,\bflambda_3^k)$ by solving (\ref{eq.shrink}).\\
		Update $\p^{k+1}=\arg\min_{\p}\mathcal{L}(\phi^k,q^k,\p,\n^k,\bflambda_1^k,\lambda_2^k,\bflambda_3^k)$ according to Theorem \ref{thm.p}.\\
		Update $\n^{k+1}=\arg\min_{\n}\mathcal{L}(\phi^k,q^k,\p^k,\n,\bflambda_1^k,\lambda_2^k,\bflambda_3^k)$ by solving (\ref{eq.sub4.2}). \\
        Reinitialize $\phi^{k+1}$.\\
\textbf{Update Lagrange multipliers}:
   \begin{align*}
&\bflambda_1^{k+1} = \bflambda_1^k+r_1(\p^{k+1}-\nabla\phi^{k+1}),\\
&\lambda_2^{k+1} = \lambda_2^k+r_2(q^{k+1}-\nabla\cdot\n^{k+1}),\\
&\bflambda_3^{k+1} = \bflambda_3^k+r_3(\n^{k+1}|\p^{k+1}|-\p^{k+1}).
\end{align*}}
	\KwOut{$\phi^{k}$.}
	\caption{Augmented Lagrangian Method (ALM) for $s=1$.}
	\label{alg:ALM}
\end{algorithm}

\subsection{Numerical Implementation Details}\label{sec:numD}

For a rectangular domain $\Omega=[0,M]\times[0,N]\in\mathds{R}^2$ with $M,N$ being positive integers, we discretize it by a Cartesian grid with $\Delta x=\Delta y=1$. For any function $u$ (resp. $\bv=(v^1,v^2)^T$) defined on $\Omega$, we use $u_{i,j}$ (resp. $\bv_{i,j}=(v^1_{i,j},v^2_{i,j})^T$) to denote $u(i\Delta x,j\Delta y)$ (resp. $\bv(i\Delta x,j\Delta y)^T$) for $0\leq i\leq M, 0\leq j\leq N$. Denote the standard forward and backward difference as
\begin{eqnarray*}
&&\partial_1^-u_{i,j}=\begin{cases}u_{i,j}-u_{i-1,j},~1<i\leq M;\\
u_{1,j}-u_{M,j},~i=1.	
\end{cases} \;\;\;
 \partial_1^+u_{i,j}=\begin{cases}
u_{i+1,j}-u_{i,j},~1\leq i<M-1;\\
u_{1,j}-u_{M,j},~i=M.	
\end{cases}
\\
&&\partial_2^-u_{i,j}=\begin{cases}u_{i,j}-u_{i,j-1},~1<j\leq N\;;\\
u_{i,1}-u_{i,N},~j=1.	
\end{cases} \;\;\;
 \partial_2^+u_{i,j}=\begin{cases}u_{i,j+1}-u_{i,j},1\leq j<N-1;\\
	u_{i,1}-u_{i,N},~j=N.
\end{cases}
\end{eqnarray*}
The gradient, divergence and the Laplacian operators are approximated as follows:
\begin{align*}
&\nabla u_{i,j}=((\partial_1^-u_{i,j}+\partial_1^+u_{i,j})/2, (\partial_2^-u_{i,j}+\partial_2^+u_{i,j})/2)\;, \\
&\nabla\cdot\mathbf{v}_{i,j} = (\partial_1^+v_{i,j}^1 + \partial_1^-v_{i,j}^1)/2 + (\partial_2^+v_{i,j}^2 + \partial_2^-v_{i,j}^2)/2\;,\\
&\Delta u_{i,j}= \partial_1^+u_{i,j}-\partial_1^-u_{i,j}+ \partial_2^+u_{i,j}-\partial_2^-u_{i,j}\;.
\end{align*}
Denote the discrete Fourier transform and its inverse by $\mF$ and $\mF^{-1}$, respectively. For a function $u$, we have
\begin{eqnarray*}
	&\mF (u)(i\pm1,j)=e^{\pm2\pi\sqrt{-1}(i-1)/M}\mF (u)(i,j), &\mF (u)(i,j\pm 1)=e^{\pm2\pi\sqrt{-1}(j-1)/N}\mF (u)(i,j)\;,
\end{eqnarray*}
which gives rise to
\begin{align*}
  &\mF (\partial_1^-u)(i,j)=(1-e^{-2\pi\sqrt{-1}(i-1)/M})\mF (u)(i,j)\;,
\end{align*}
and $\mF (\partial_1^+u)(i,j),\mF (\partial_2^-u)(i,j)$ and $\mF (\partial_2^+u)(i,j)$ can be computed similarly.
Both OSM and ALM use FFT to enhance the computational efficiency.
The first equation of (\ref{eq.split1}) and the first equation of (\ref{eq.sub1.2}) belong to the same class:
\begin{equation}
  -a\Delta u+bu=c\;,
  \label{eq.laplacian}
\end{equation}
for a function $u$ with constants $a,b>0$ and constant $c$. Using the Fourier transform, we have
\begin{equation*}
\mF (\Delta u)(i,j)=\left[2\cos(\pi\sqrt{-1}(i-1)/M)+2\cos(\pi\sqrt{-1}(j-1)/N)-4\right]\mF u(i,j)\;.
\end{equation*}
Then (\ref{eq.laplacian}) can be solved by
\[
u=\mF^{-1}\left( \frac{\mF(c)}{b-a\left(2\cos(\pi\sqrt{-1}(i-1)/M)+2\cos(\pi\sqrt{-1}(j-1)/N)-4\right)} \right)\;.
\]
The equation (\ref{eq.sub4.1}) is in the form of
\begin{equation}
  -a\nabla(\nabla\cdot \bv)+b\bv=\bc\;,
  \label{eq.disn.1}
\end{equation}
for some vector valued function $\bv=(v^1,v^2)^T$ where $a,b$ are constant positive scalars and $\bc=(c^1,c^2)^T$ is a vector valued constant. After distretization, (\ref{eq.disn.1}) can be written as
\begin{equation}
\begin{cases}
-a\nabla(\partial_1^+\partial_1^-v^1+\partial_1^+\partial_2^-v^2)+bv^1=c^1,\\
-a\nabla(\partial_2^+\partial_1^-v^1+\partial_2^+\partial_2^-v^2)+bv^2=c^2.	
\end{cases}
\label{eq.disn.2}
\end{equation}
Applying the discrete Fourier transform on both sides of (\ref{eq.disn.2}), we get
\begin{align}
A
\begin{bmatrix}
\mathcal{F}(v^1)(i,j)\\
\mathcal{F}(v^2)(i,j)	
\end{bmatrix}
=\begin{bmatrix}
\mathcal{F}(c^1)(i,j)\\
\mathcal{F}(c^2)(i,j)	
\end{bmatrix}\;,
\label{eq.disn.3}
\end{align}
with
$$
A=\begin{bmatrix}
	b-a(e^{\sqrt{-1}(i-1)/M}-1)(1-e^{-\sqrt{-1}(i-1)/M})&-a(e^{\sqrt{-1}(i-1)/M}-1)(1-e^{-\sqrt{-1}(j-1)/N})\\
-a(e^{\sqrt{-1}(j-1)/N}-1)(1-e^{-\sqrt{-1}(i-1)/M})&b-a(e^{\sqrt{-1}(j-1)/N}-1)(1-e^{-\sqrt{-1}(j-1)/N})	
\end{bmatrix}.
$$
Hence, $\bv$ can be computed by first solving (\ref{eq.disn.3}) for $(\mF(v^1),\mF(v^2))^T$ and then apply inverse Fourier transform.

For any $\x\in \Omega$, $d(\x)$ is the distance from $\x$ to  the collection of the point cloud $\mathcal{D}$, and it can be computed by solving the Eikonal equation
\begin{equation}
  \begin{cases}
    |\nabla d|=1\;,\\
    d(\x)=0\;, \;\;\;  \forall \x\in \mathcal{D}.
  \end{cases}
  \label{eq.eikonal}
\end{equation}
The simplest monotonic scheme to discretize (\ref{eq.eikonal}) is the Lax-Friedrich scheme which leads to the updating formula \cite{kao2004lax}:
\begin{eqnarray*}
  &&d_{i,j}^{n+1}=\frac{1}{2}\Bigg(1-|\nabla d^n_{i,j}|+ \frac{d^n_{i+1,j}+d^n_{i-1,j}}{2} +  \frac{d^n_{i,j+1}+d^n_{i,j-1}}{2}\Bigg)\;.
\end{eqnarray*}
This is updated with a fast sweeping method.

To make our algorithm robust, when updating $\phi$, we reinitialize the level set to be a signed distance function via solving
\[
  \phi_{\tau}+\mbox{sign}(\phi)(1-|\nabla \phi|)=0.
 \]
In practice, after each iteration, we only solve this PDE for a few iterations.   For three dimensional space, we use a simple extension of the two dimensional case.

\begin{remark}
In three-dimensional surface reconstruction problems, the point cloud can be large and the computer memory is limited.  One can consider a narrow tube which encloses the point cloud, and assume the reconstructed surface lies inside this tube during the evolution.  Adopting the local level set method such as \cite{peng1999pde}, only the values of the level set function on grid points inside the tube need to be stored.  ALM and OSM can be applied under the local level set method framework except for solving the Laplace equations. Under this framework, one can derive a corresponding linear system for each Laplace equation which can be solved efficiently by the conjugate gradient method.
\end{remark}
\section{Numerical Results and Comparisons}\label{sec:Exp}

We present numerical results of the proposed model (\ref{eq.energy}). Without specification, when OSM is used, we refer to Algorithm \ref{alg:OSM} for model (\ref{eq.splitEner}) with $\gamma=10,\alpha=1$ and $\Delta t=50$; when ALM is used, we refer to Algorithm \ref{alg:ALM} for model (\ref{eq:almModel}) with $\beta=0.1$.
For a fixed $\eta$, OSM only has one (the time step) parameter, whereas ALM has three parameters ($r_1,r_2,r_3$). We use domain $[0,100]^2$ for two dimensional problems and $[0,50]^3$ for three dimensional problems. In all examples, $\varepsilon=1$ is used.

In this section, we first consider the effect of parameters for ALM.  Second, we compare the performance of ALM and OSM.  We find that using OSM with $s=2$ gives the best results.  We conclude this section by several examples to further explore the performance of OSM when $s=2$.

\subsection{Choice of  Parameters for  ALM  Method}
In the case of ALM,  the choice and combinations of the parameters are delicate.    When $r_1$ or $r_2$ is increased, the reconstruction becomes closer to the point cloud.  In Figure~\ref{fig.r1}, we fix $r_2=10$, $r_3=3$, and $\eta=2$, and let $r_1$ vary.  With increased $r_1$, ALM renders the curve closer to the point cloud.  In Figure~\ref{fig.r2}, we fix $r_1=10$, $r_3=3$, and $\eta=2$, and let $r_2$ vary; larger $r_2$ induces better reconstruction. For this example,  $r_3$ has  little influence, yet, with a large $r_3$, the results may become unstable or divergent.
\begin{figure}
\centering
\includegraphics[scale=0.3]{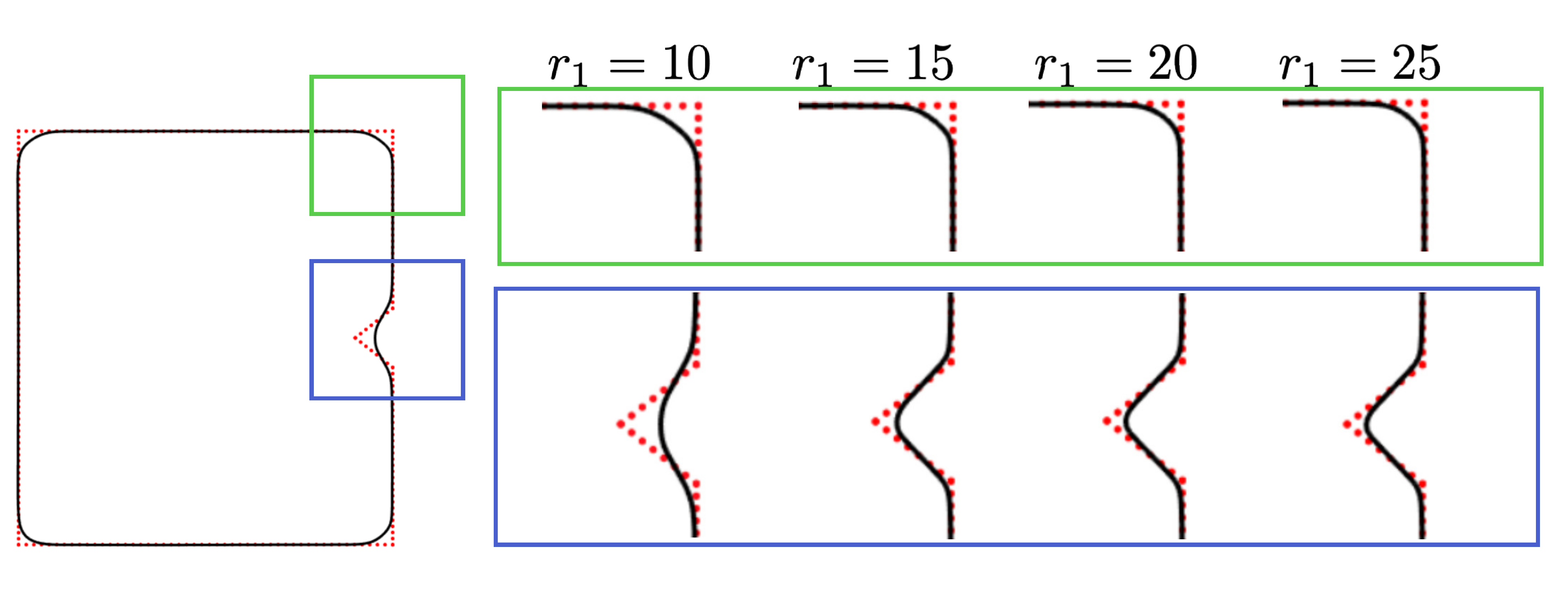}
\caption{Effect of $r_1$ in ALM. For fixed $r_2=10$, $r_3=3$, and $\eta=2$, increasing $r_1$ induces better reconstruction on the concave part.}\label{fig.r1}
\end{figure}
\begin{figure}
\centering
\includegraphics[scale=0.3]{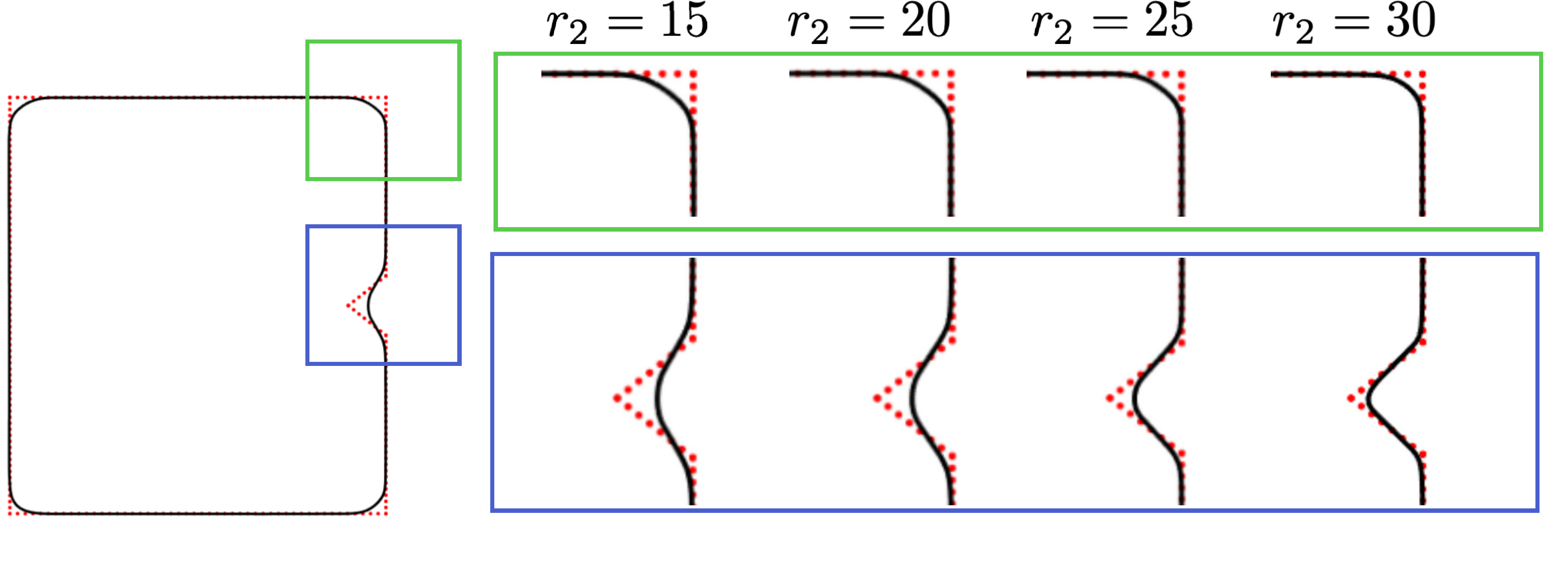}
\caption{Effect of $r_2$ in ALM. For fixed $r_1=10$, $r_3=3$, and $\eta=2$, increasing $r_2$ induces better reconstruction on the concave part.}\label{fig.r2}
\end{figure}

In Figure~\ref{fig.rho}, we fix $r_1=15$, $r_2=10$, and $r_3=3$, and increase $\eta$.  With more influence on the curvature,  as $\eta$ is increased from $0$ to $1$, the indent on the rectangle is better reconstructed. When we increase $\eta$ from $1$ to $5$, we see that, the indent is preserved, yet the tip of the wedge does not extend inward as much as in the case where $\eta=1$. This is because a large value of $\eta$ encourages both small mean curvature and short curve length.  We find that increasing $\eta$ also helps to avoid oscillation during the iteration.  In Figure~\ref{fig.rho}, we plot the energy curves corresponding to the these cases. Before the $100$-th iteration, these curves are indistinguishable; however, after the $100$-th iteration, larger values of $\eta$ suppress the oscillation of the energy curve, which gives a more stable convergence.
\begin{figure}
\centering
\includegraphics[width=0.8\textwidth]{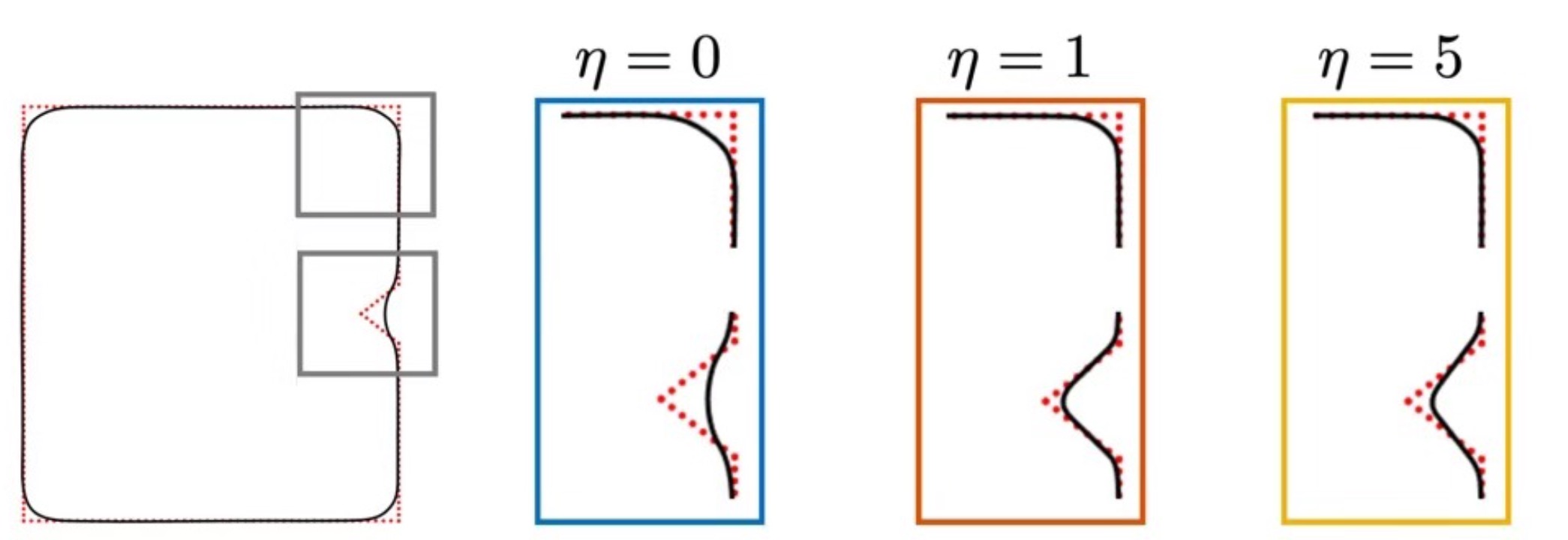}	\\
\includegraphics[width=\textwidth]{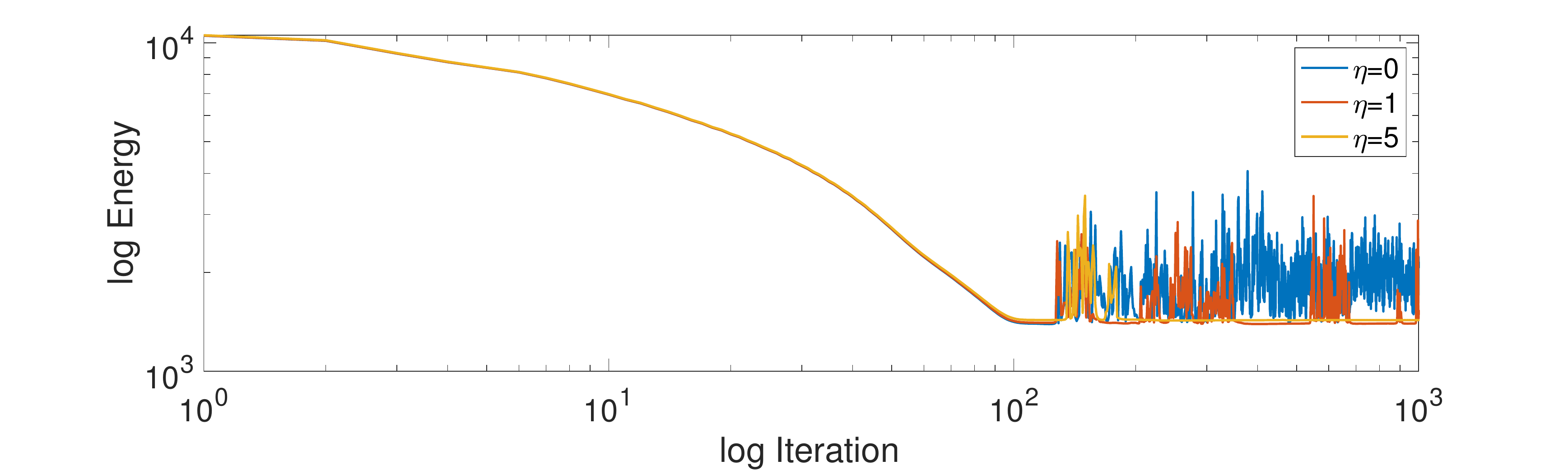}	
\caption{Effect of $\eta$ in ALM. Here $r_1=15$, $r_2=10$, and $r_3=3$ are fixed. Increasing $\eta$ induces reconstruction of the concave wedge. Although in cases, the energy curves are identical before the $100$-th iteration,  larger $\eta$ suppresses the oscillation of the energy curve: yellow line ($\eta=5$) is  more stable compared to red ($\eta=1$) or blue ($\eta =0$). }~\label{fig.rho}
\end{figure}

Figure~\ref{fig.noise_ALM} shows the robustness against noise.   The point cloud is sampled from a circle in $\Omega=[0,200]^2$, and Gaussian noise with standard deviation 2 is added to the data. Using ALM with $r_1=15,r_2=10, r_3=3$, we test $\eta= 0, 1$ and 10.  Figure~\ref{fig.noise_ALM} (a) shows similar performances, while the zoomed-in results in Figure~\ref{fig.noise_ALM} (b) show that the larger the $\eta$  the smoother the result becomes.
\begin{figure}
\centering
\begin{tabular}{cc}
(a) & (b) \\
\includegraphics[scale=0.4]{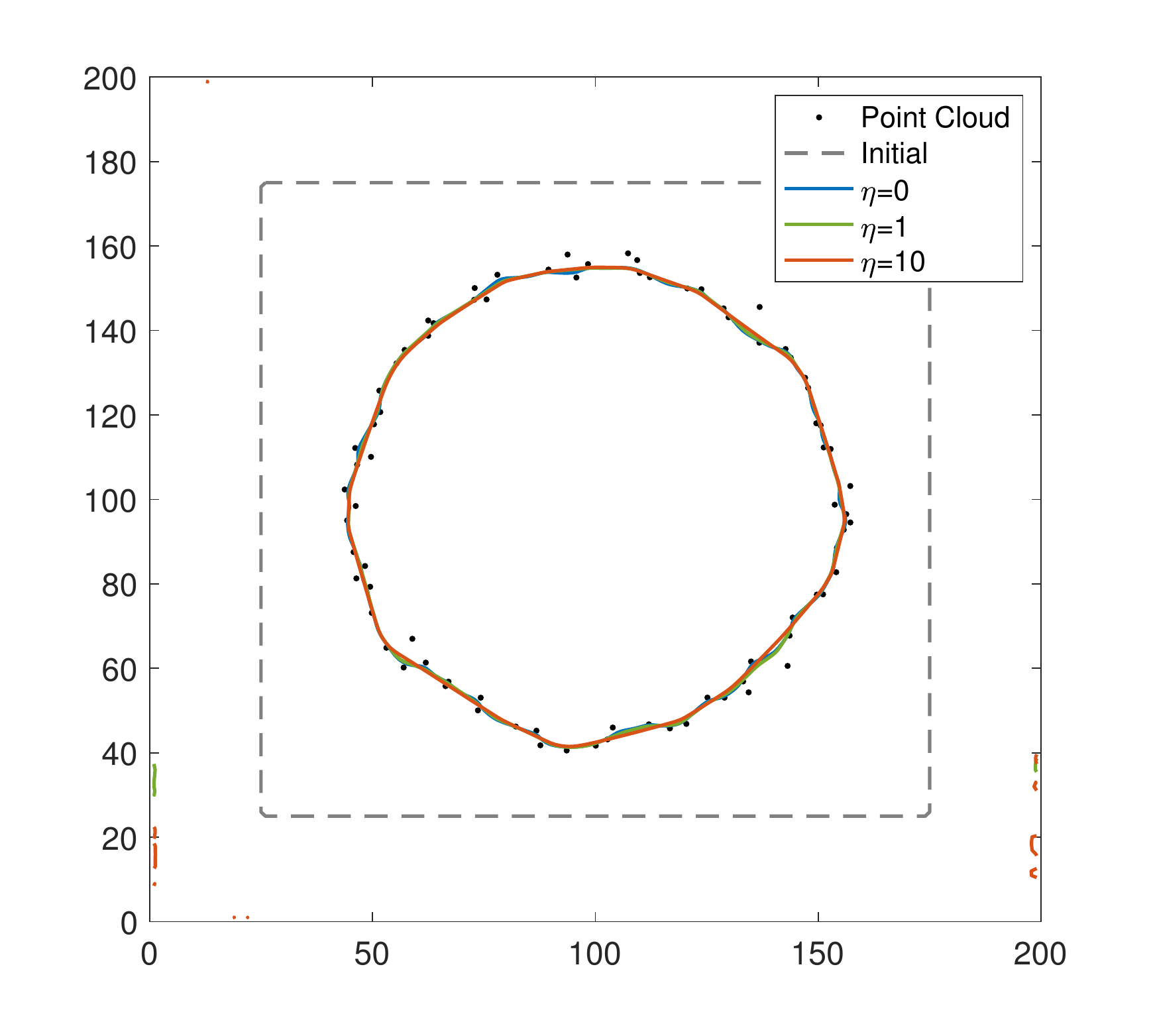} &
\includegraphics[scale=0.4]{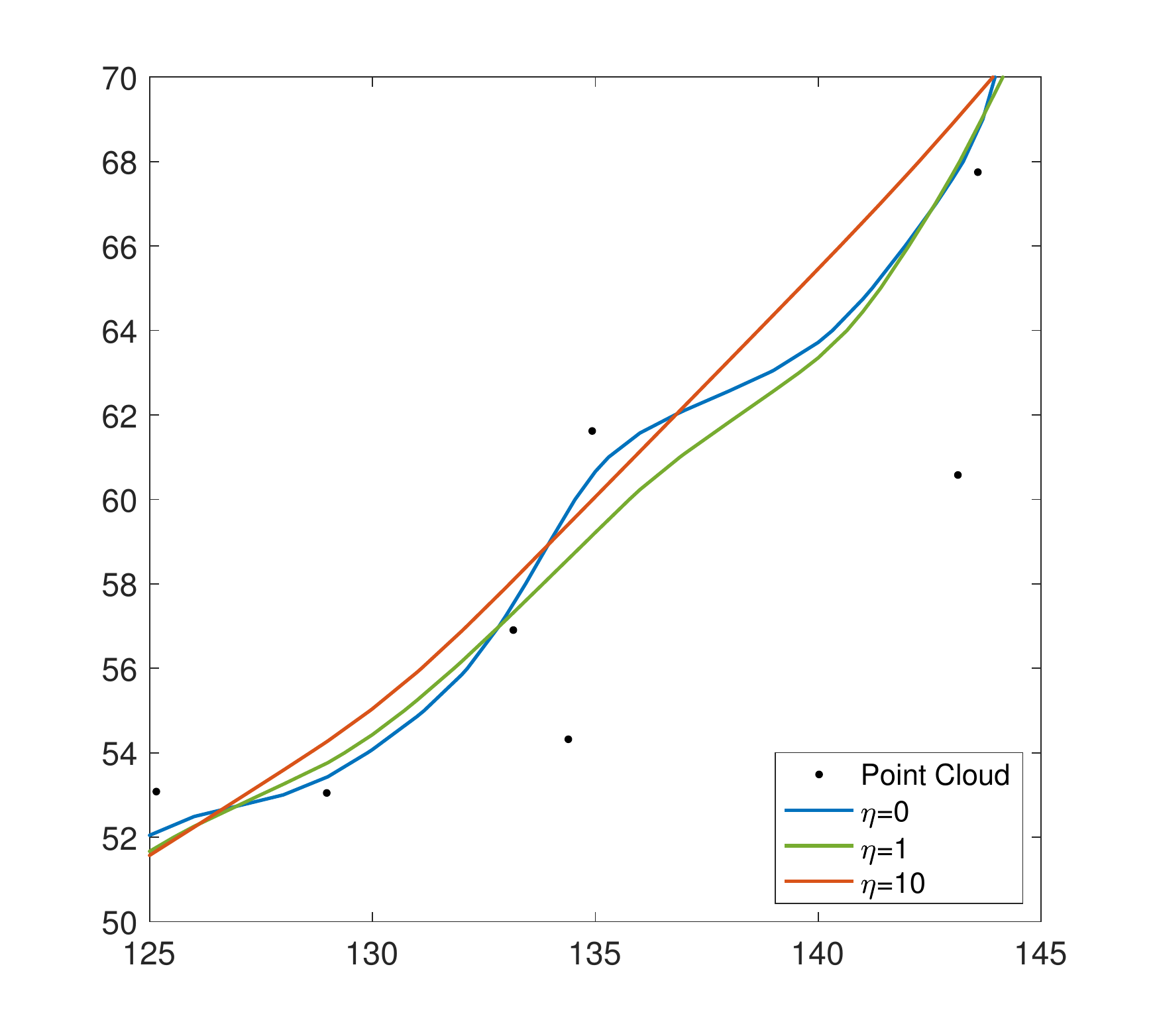}
\end{tabular}
\caption{(a) Results by ALM with noisy data and $r_1=15,r_2=10,r_3=3$. (b) shows a zoom-in of the right-bottom  of (a).  The noise is additive Gaussian with standard deviation 2. As $\eta$ increases, the curve becomes less oscillatory.}\label{fig.noise_ALM}
\end{figure}

\subsection{Comparison between OSM and ALM}

Figure \ref{fig.energycomp} (a) shows the energy convergence comparison between OSM (with $s=2$) and ALM ($s=1$). In ALM, $r_1=15,r_2=10,r_3=3$ is used.   Convergence to the steady state is faster for OSM. The reconstructed curves are shown in Figure \ref{fig.energycomp} (b) and (c).   ALM  with $s=1$ prefers to shorten the length, since it allows sharp corners.  There is a balance between the distance term and the regularization term in the functional.  OSM performs better in preserving corners, while the results extrude out a little bit at all corners.   In this case of $s=2$, the reconstruction is more circular, since sharp corners are not allowed.
\begin{figure}
\centering
\centering
  \begin{tabular}{ccc}
  (a) & (b) & (c)\\
  \includegraphics[width=0.32\textwidth]{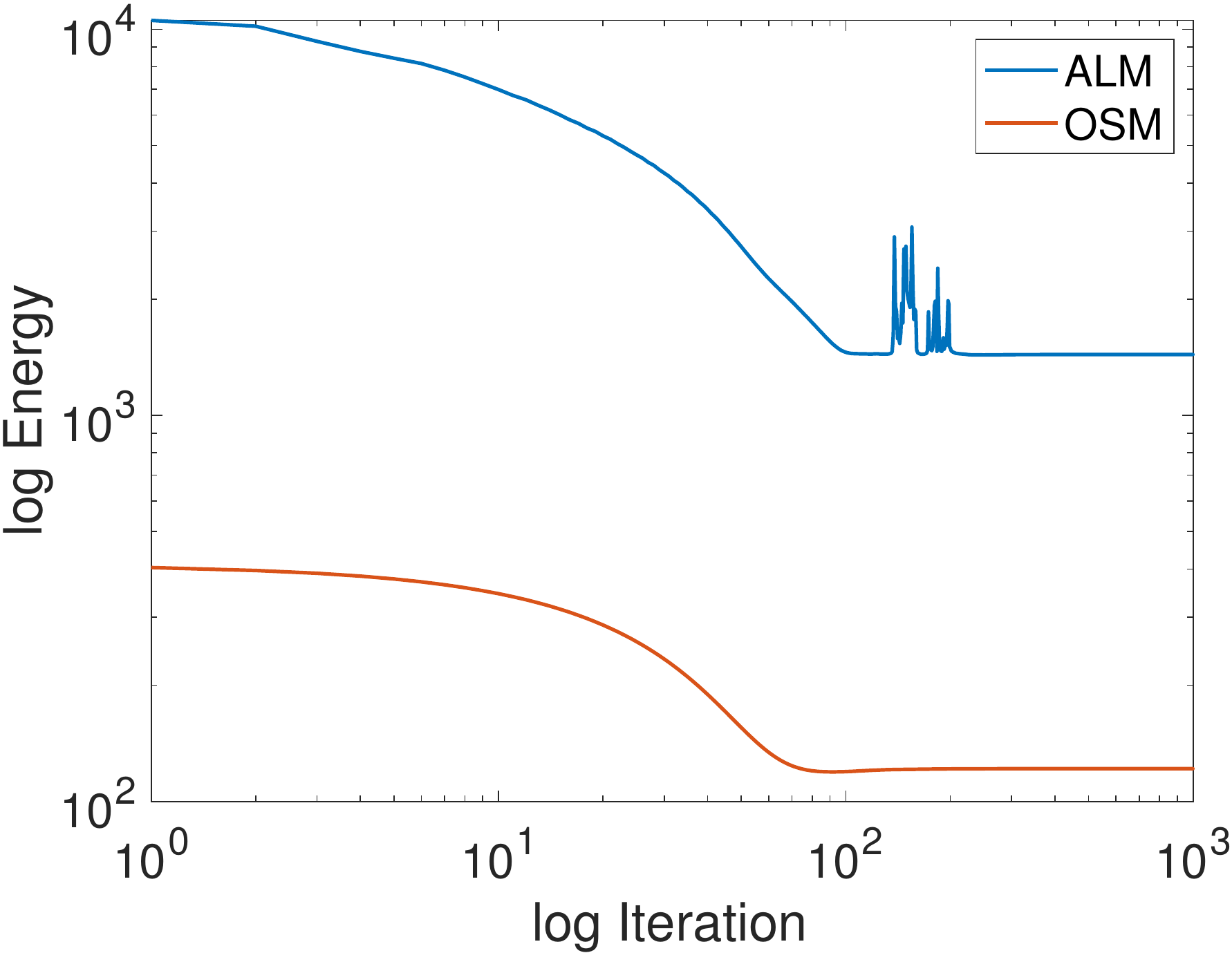}&
  \includegraphics[width=0.2\textwidth]{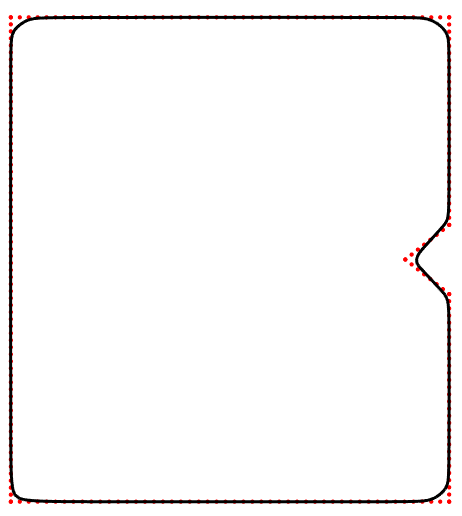}&
  \includegraphics[width=0.2\textwidth]{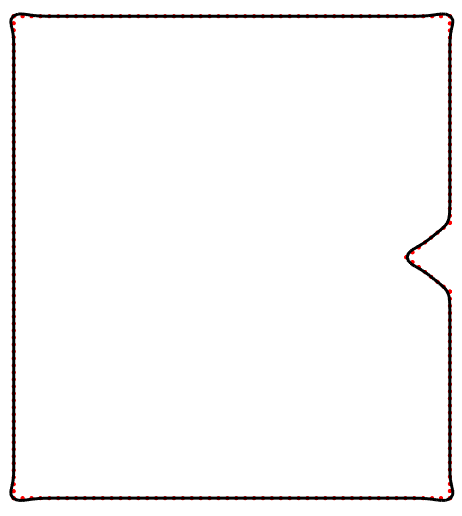}
  \end{tabular}

\caption{With $\eta=2$, comparison between OSM  with $s=2$ and ALM. In ALM, $r_1=15,r_2=10,r_3=3$ is used. Convergence to the steady state is faster for OSM.   The reconstructed curves are shown in (b) for ALM and in (c) for OSM: ALM may shorten the curve, while OSM can extrude a corner to make a circular reconstruction.  }\label{fig.energycomp}
\end{figure}

Using OSM, we fix $s=2$ for the distance term in (\ref{eq.energy}) and explore the difference between using (I) no curvature term $\eta=0$, (II) $L_1$ norm of the mean curvature, and (III) $L_2$ norm of the mean curvature.  Figure \ref{fig.model} (a) shows the comparison between  $\eta=0$ (green curve), $s=1$ (red curve), and $s=2$ (blue curve).  The blue curve (OSM with $s=2$) in (a) is presented separately in Figure \ref{fig.model} (b).
OSM using $s=2$ gives the best result capturing the structure of the underlying surface more accurately.
\begin{figure}
  \centering
  \begin{tabular}{cc}
  (a) & (b) \\
  \includegraphics[width=0.35\textwidth]{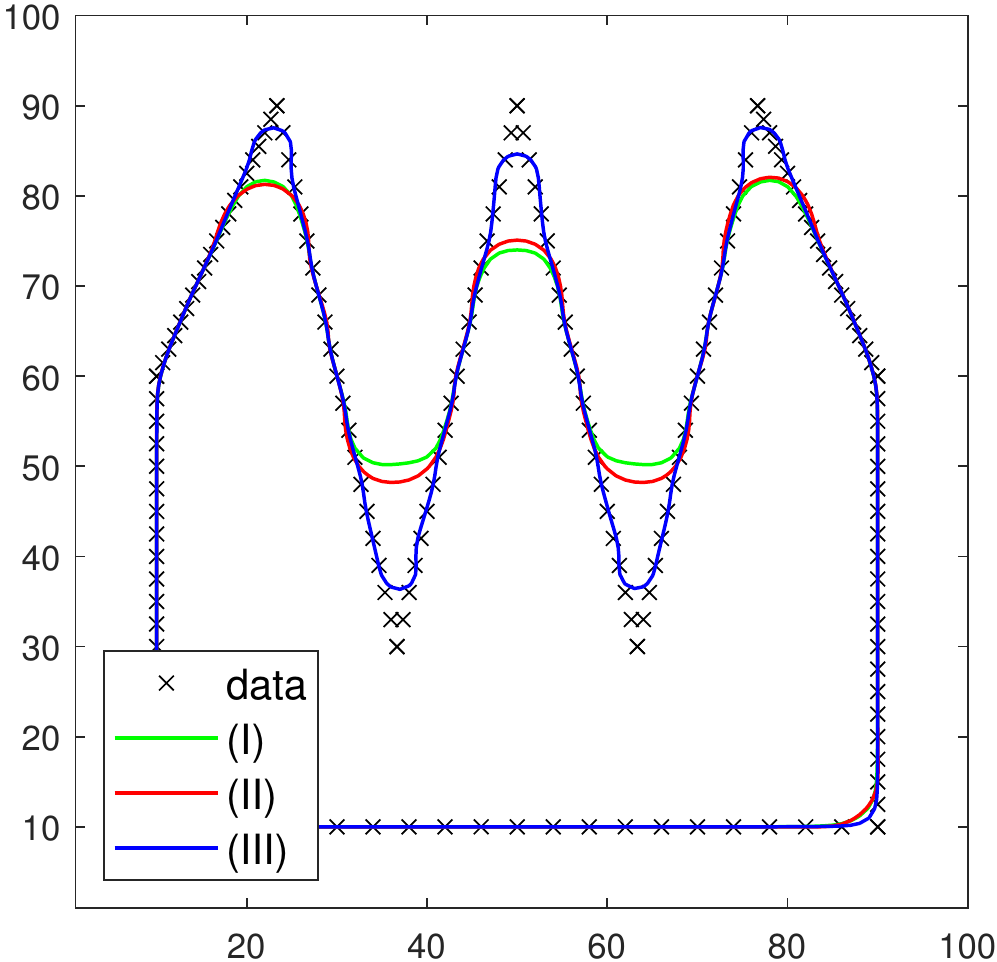}&
  \includegraphics[width=0.35\textwidth]{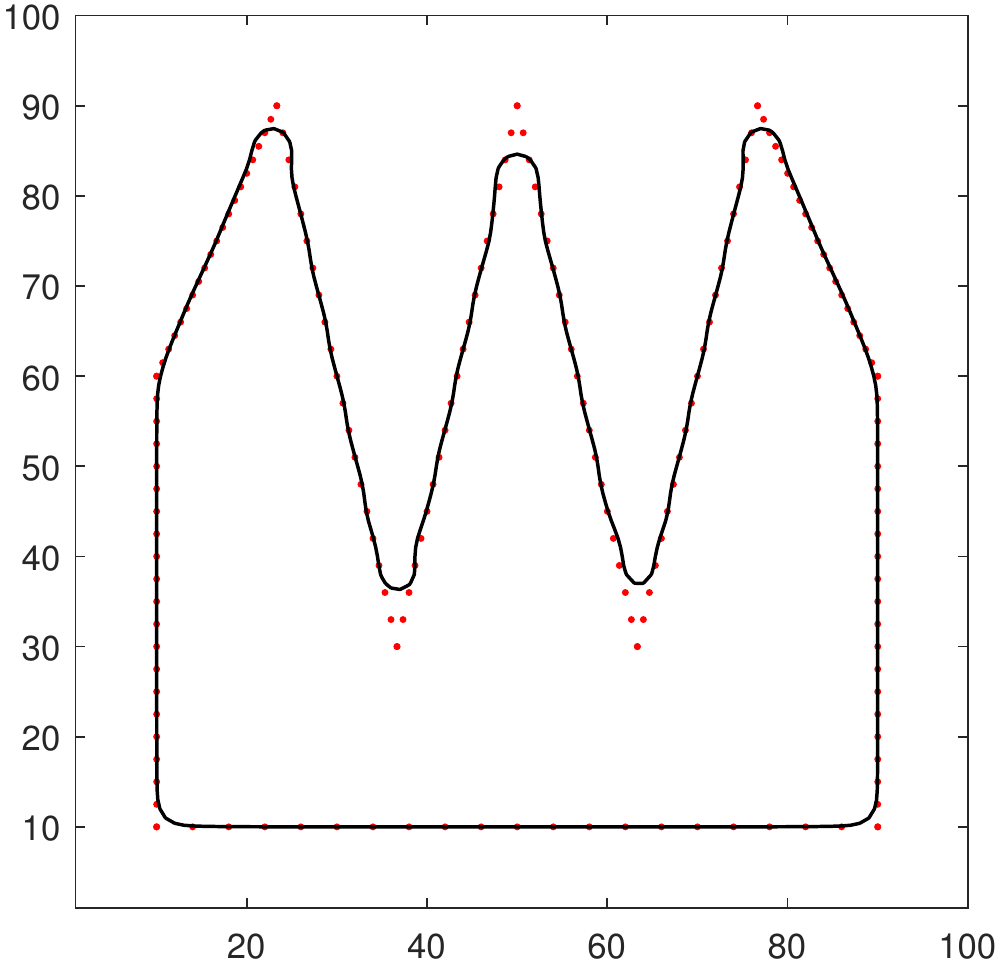}
  \end{tabular}
  \caption{ (a) By OSM with $s=2$ for the distance term in (\ref{eq.energy}), the comparison between (I) $\eta=0$ (green curve), (II) $s=1$ (red curve), and (III) $s=2$ (blue curve) for the curvature term. (b)  OSM with $\eta=2.5$ for $s=2$ in the model (\ref{eq.splitEner}). This is the blue curve in (a).  OSM using $s=2$ gives the best result in terms of capturing the structure of the underlying surface more accurately.}\label{fig.model}
\end{figure}
The rest of the numerical experiments use OSM with $s=2$,  which gives more stable results with less number of parameters and faster convergence with smaller minimized energies.

\subsection{Effect of curvature constraint: OSM with $s=2$}
\label{sec.ex2d}

Figure \ref{fig.ex2d} shows the comparison between the algorithm from \cite{EZL*12} (the first row),  $\eta = 0$ (the second row), and
OSM with $\eta>0$ (the third row) for different surfaces.  In the algorithm from \cite{EZL*12}, $r_1=r_2=8,r_3=r_4=3$ are used and the reinitialization is used to post-process the surface. For the boomerang shape in column (d),  the surface constructed by \cite{EZL*12} first shrinks to a point and then disappears. From this comparison, OSM with the curvature regularization provides the best results. The performance of the algorithm from \cite{EZL*12} is similar to that of OSM without curvature regularization.

With the curvature term,  the shape of the underlying surface is better captured in the third row.
The interior triangle shape of Figure \ref{fig.ex2d}~(a), the first column, is better captured with the curvature term.  Without the curvature term,  the reconstructed curve does not  move further inside, since the prominent part has attained the balance between the curve length and its distance to the two sides of the triangle.  Notice that the prominent part in the second row has non-zero curvature.  With a positive $\eta$, this balance is broken and the prominent part will further move towards the upper vertex of the triangle.   Also, for the cases with sharp corners in Figure \ref{fig.ex2d}~(c)-(d), our model with the curvature term improves the results and recover the underlying shape better.
\begin{figure}
  \begin{tabular}{c c c c}
    (a) & (b) & (c) & (d) \\
     \includegraphics[width=0.22\textwidth]{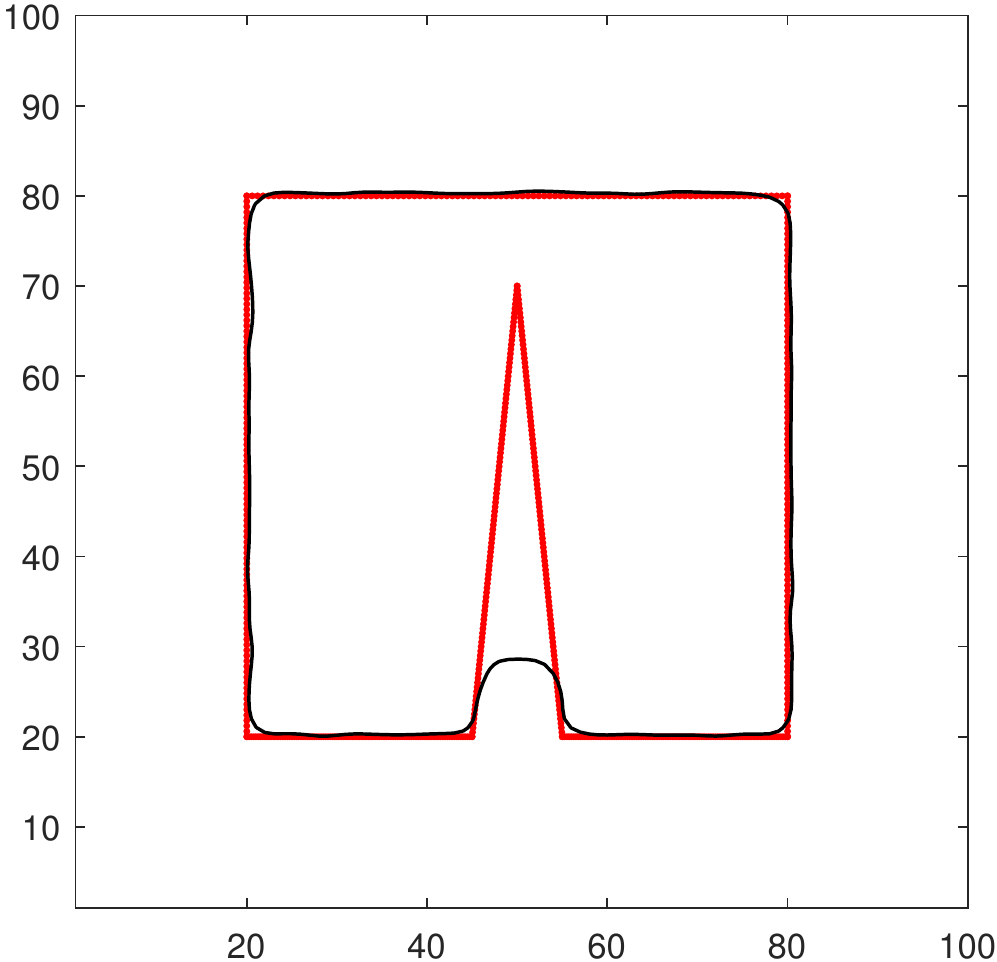}&
    \includegraphics[width=0.22\textwidth]{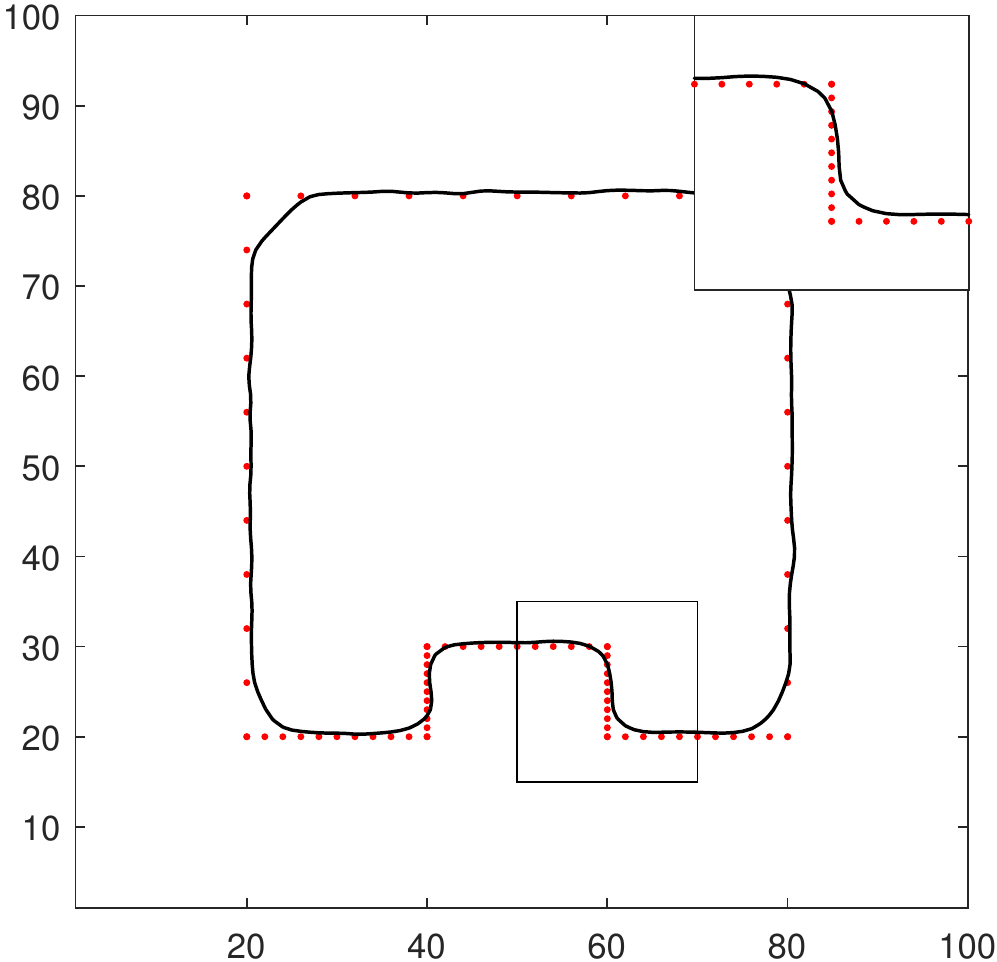}&
    \includegraphics[width=0.22\textwidth]{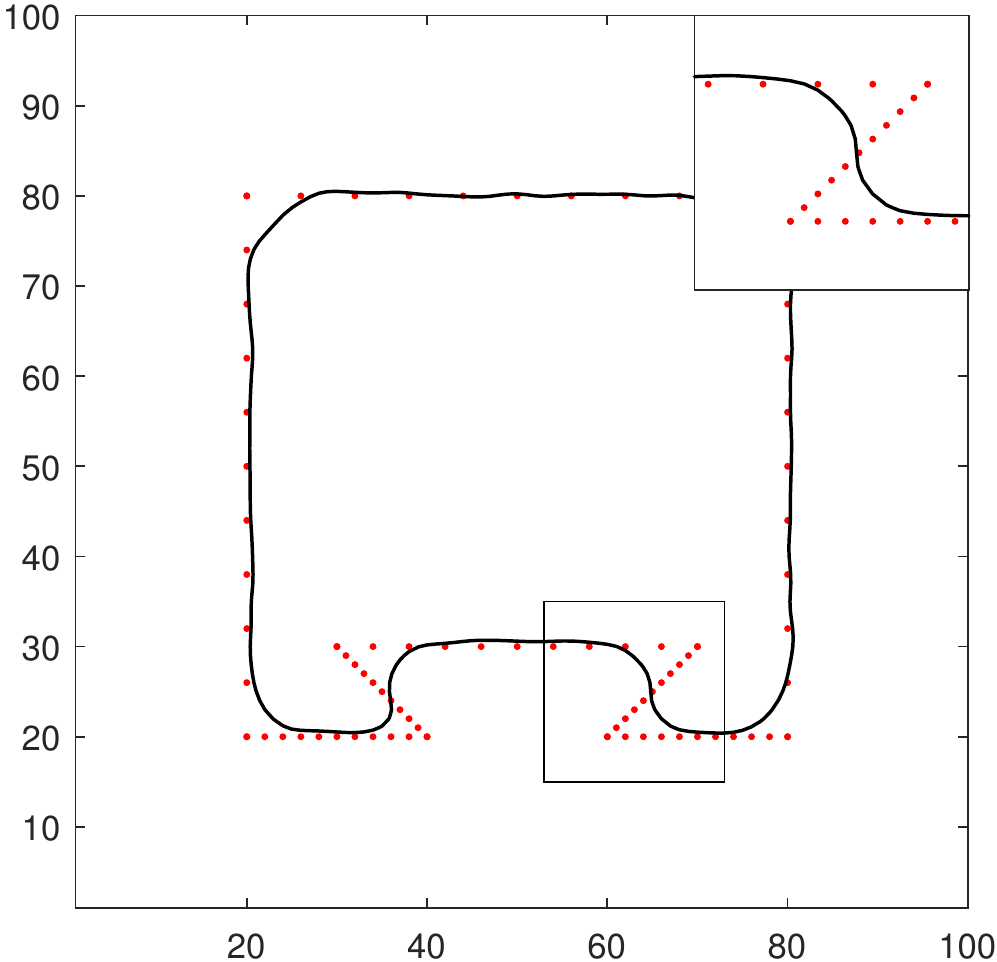}& \\
    \includegraphics[width=0.222\textwidth]{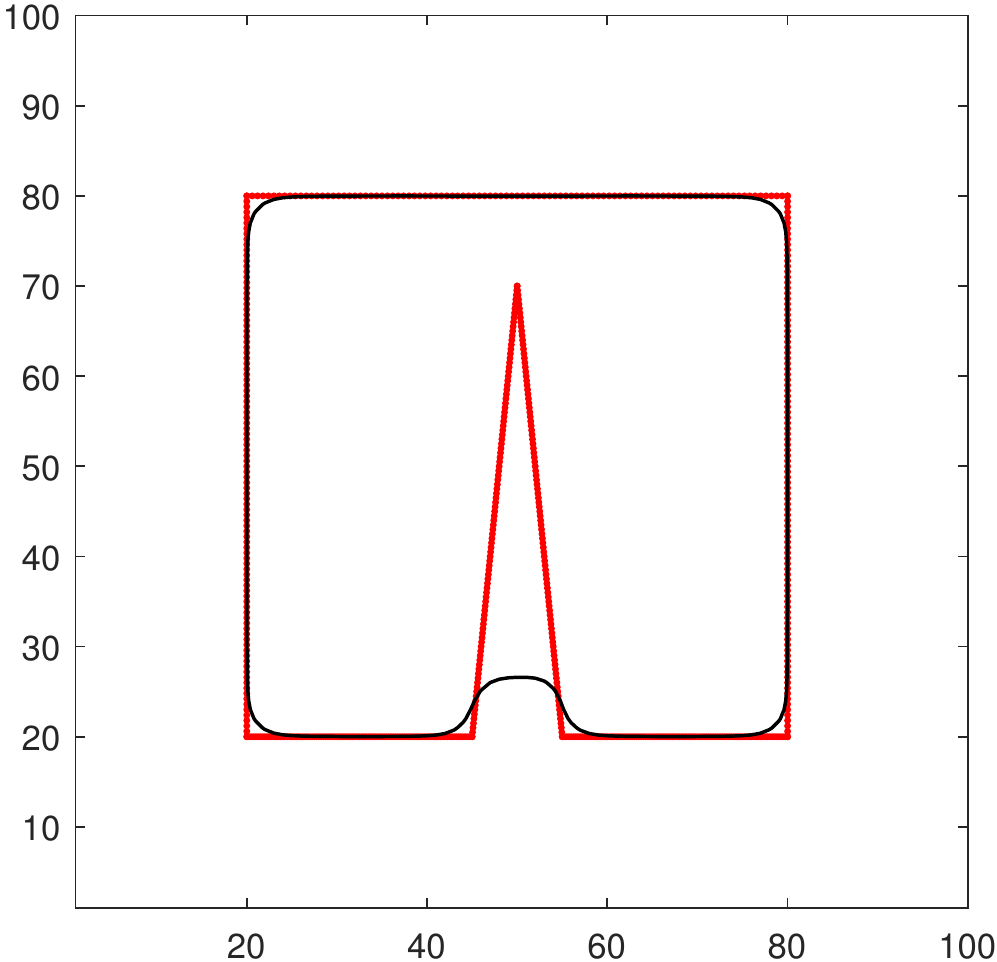}&
    \includegraphics[width=0.22\textwidth]{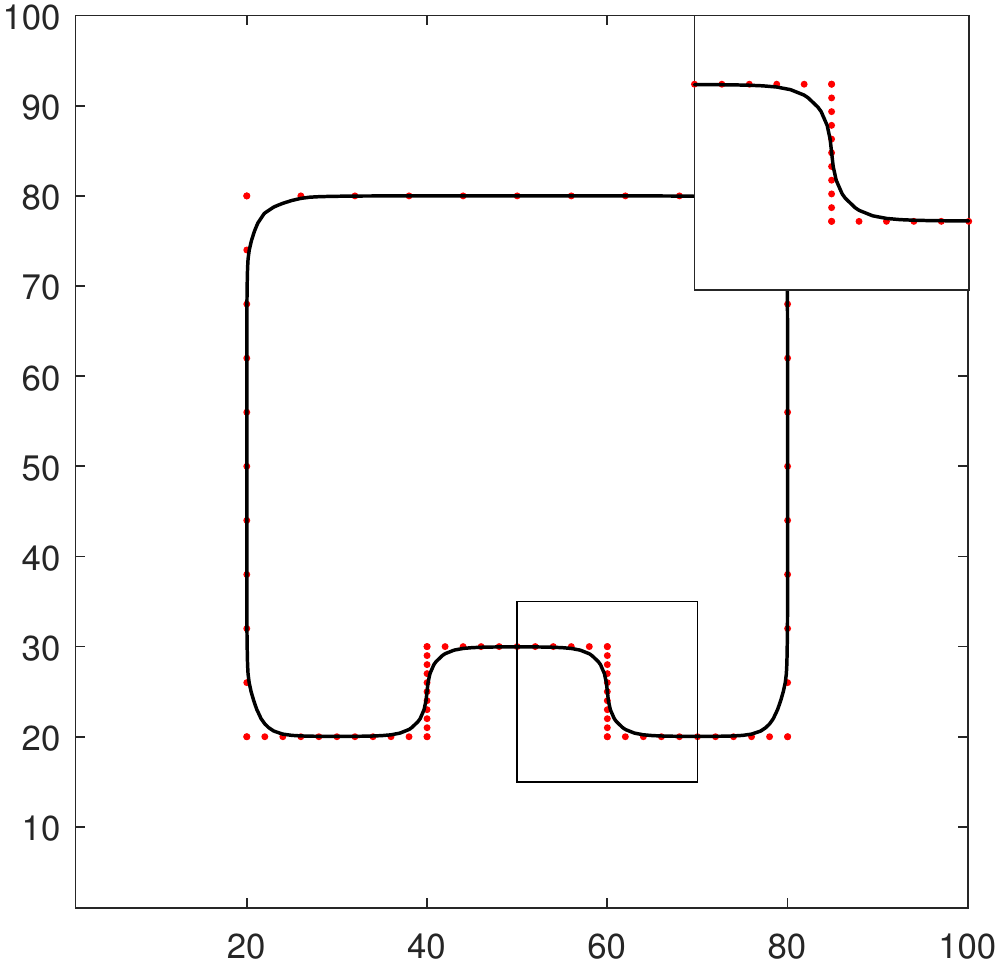}&
    \includegraphics[width=0.22\textwidth]{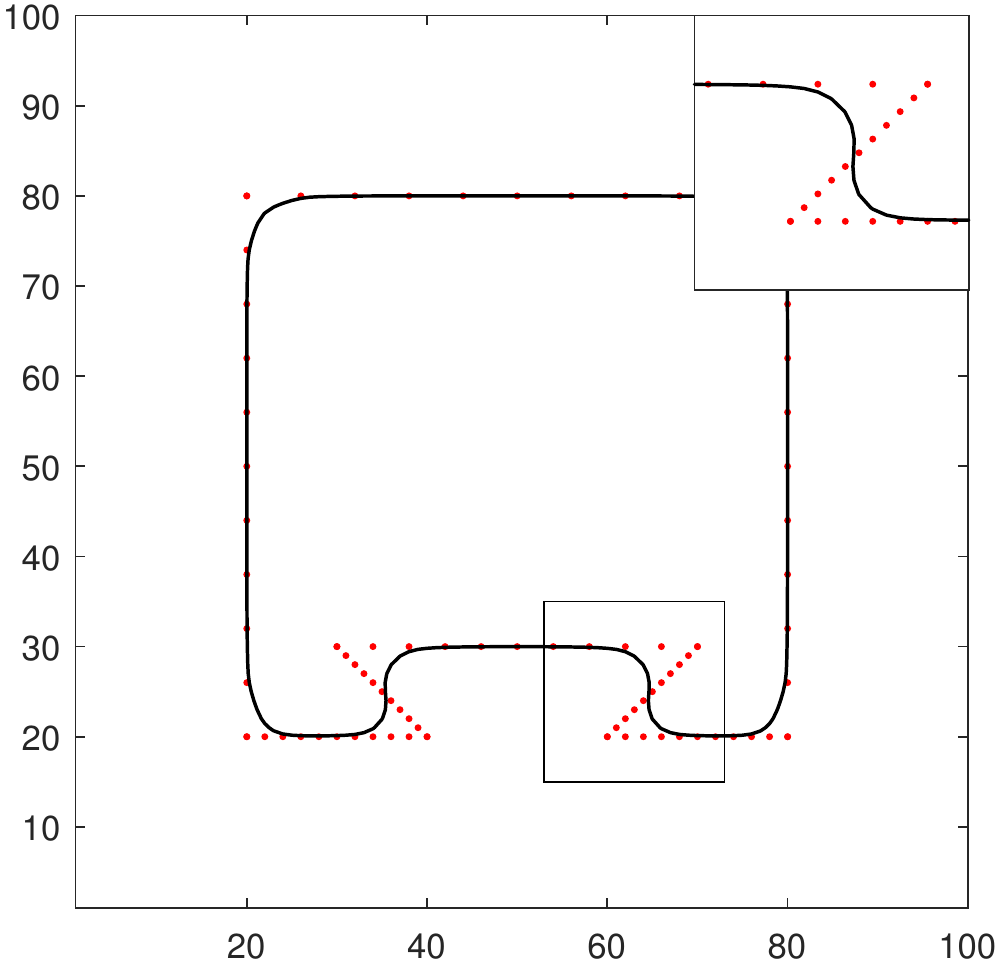}&
    \includegraphics[width=0.22\textwidth]{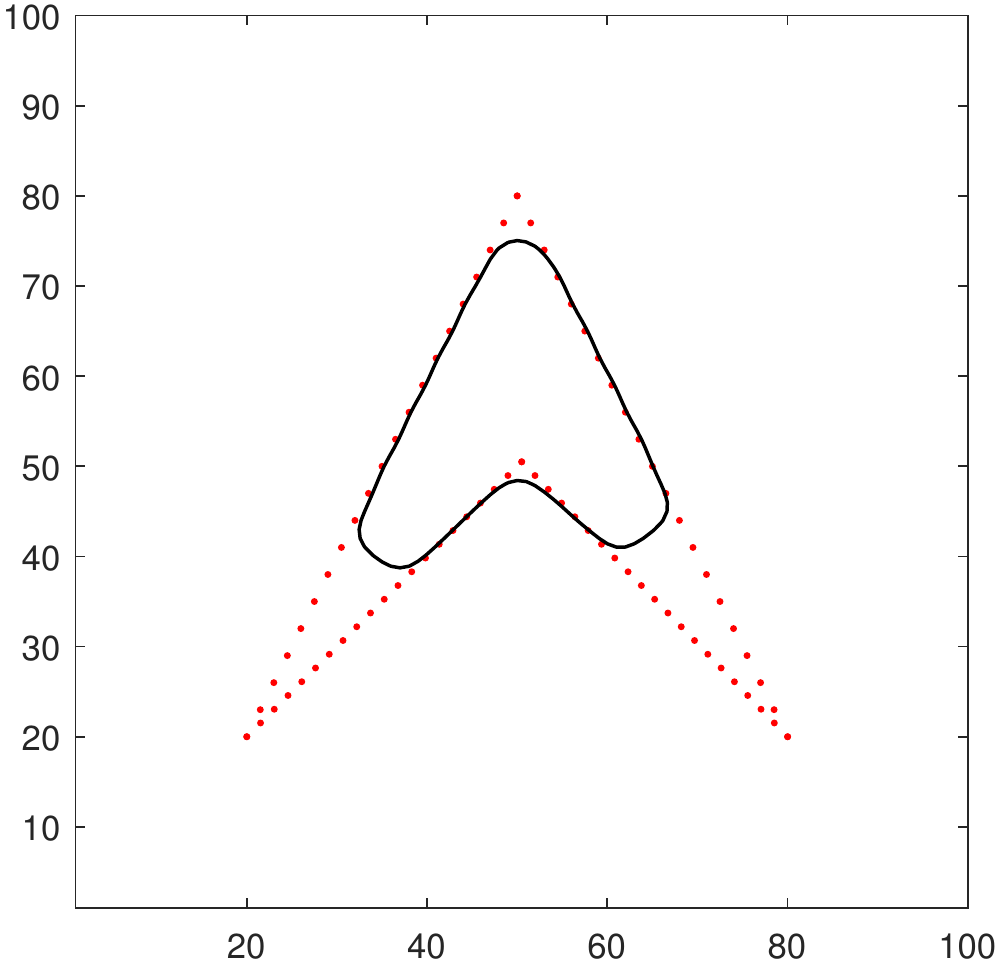}\\
    \includegraphics[width=0.22\textwidth]{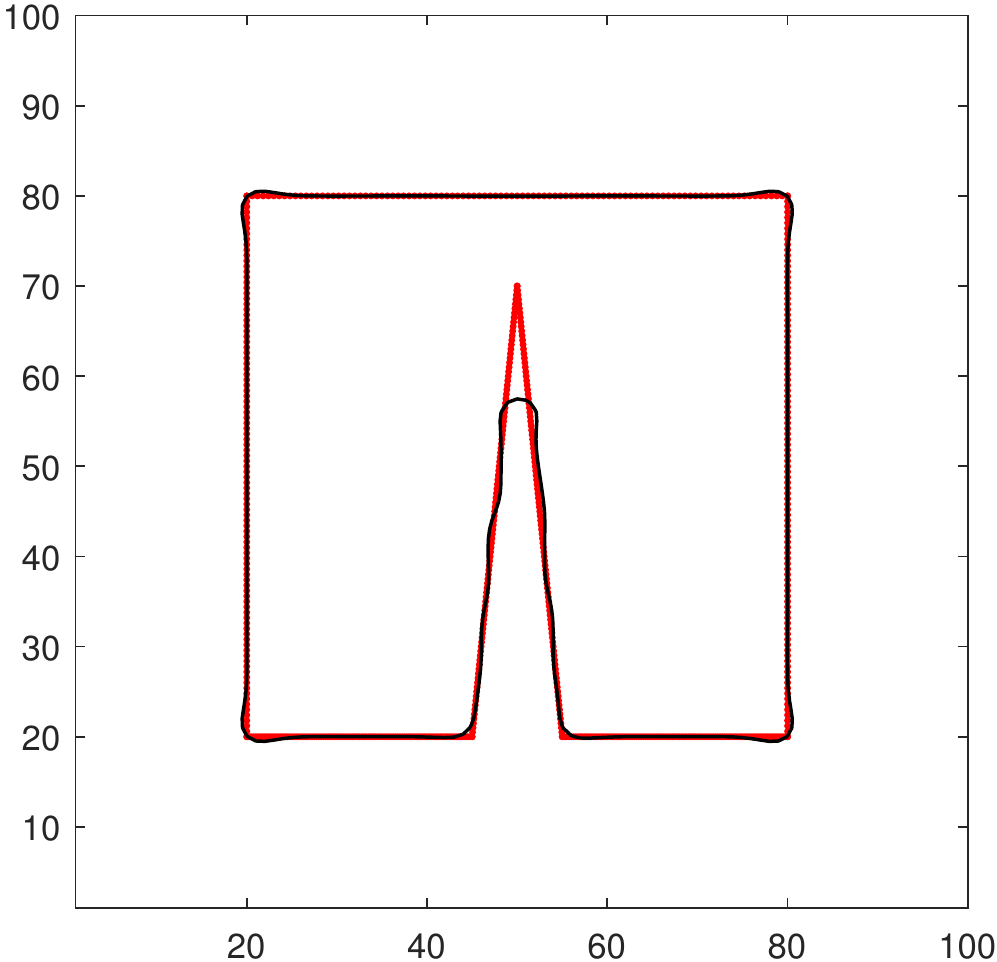}&
    \includegraphics[width=0.22\textwidth]{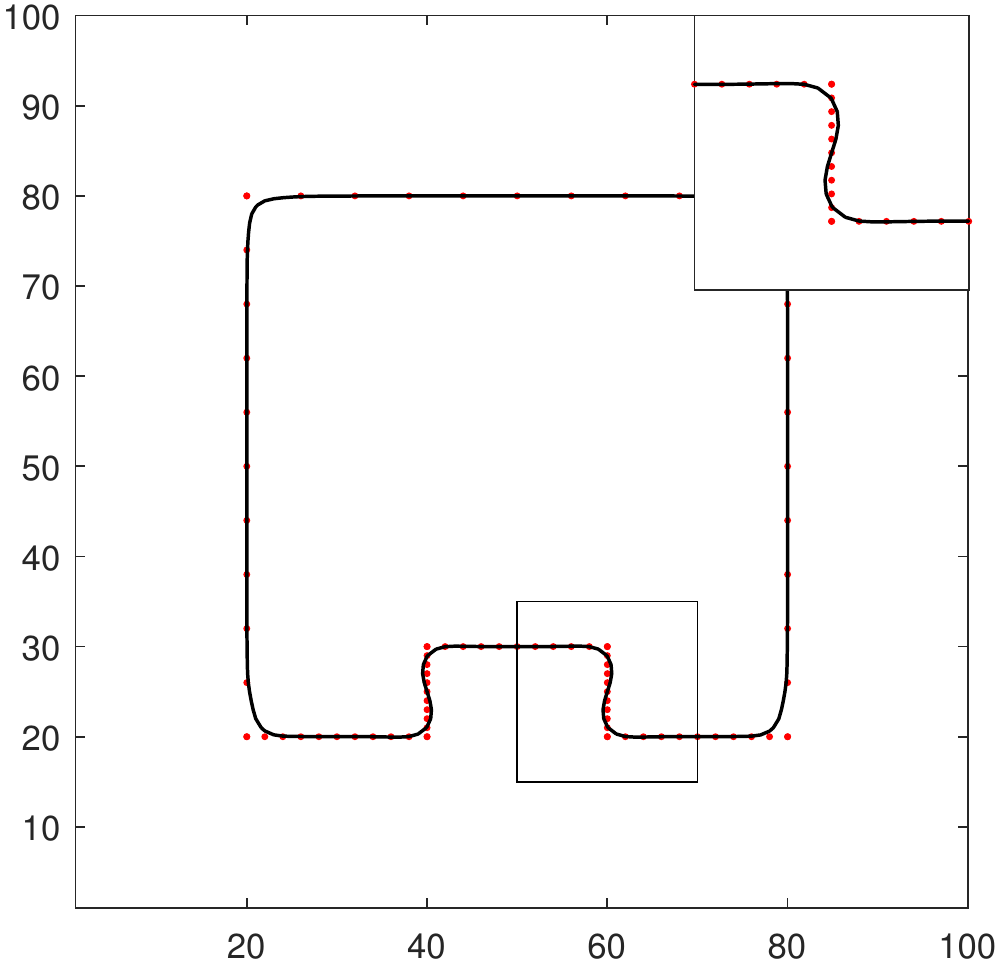}&
    \includegraphics[width=0.22\textwidth]{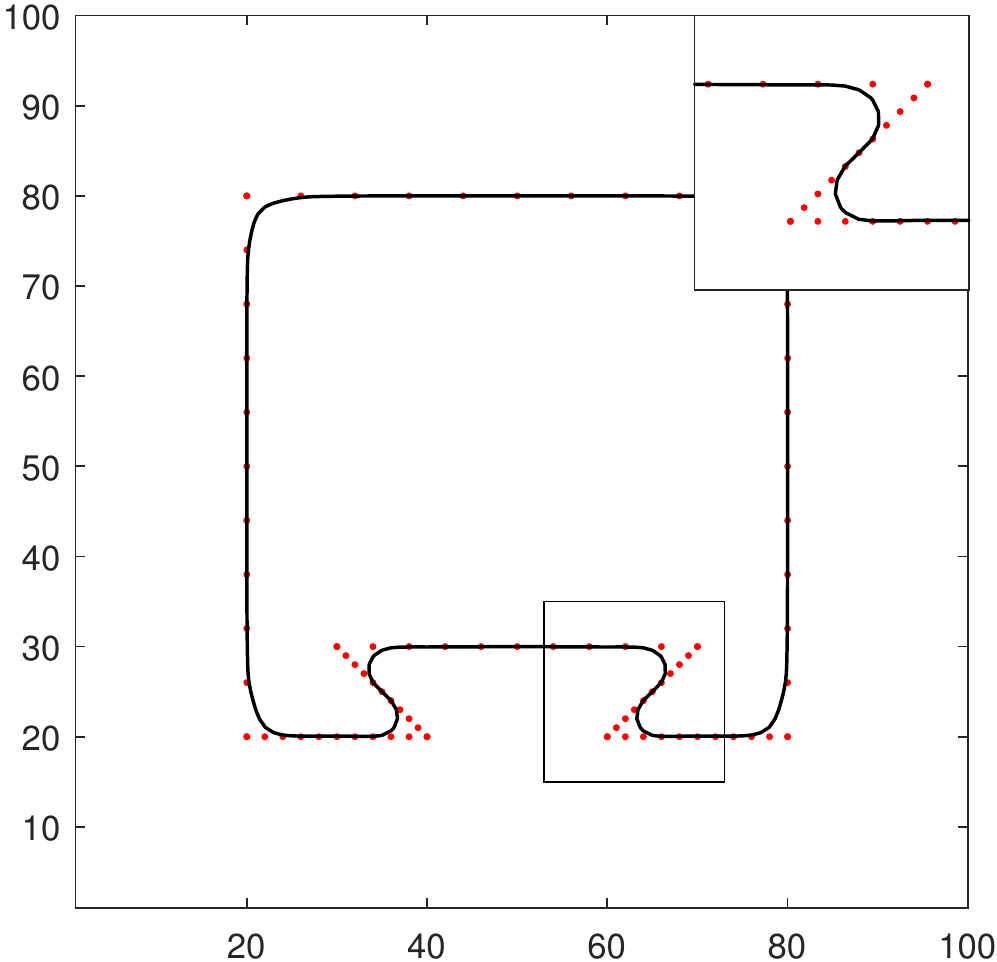}&
    \includegraphics[width=0.22\textwidth]{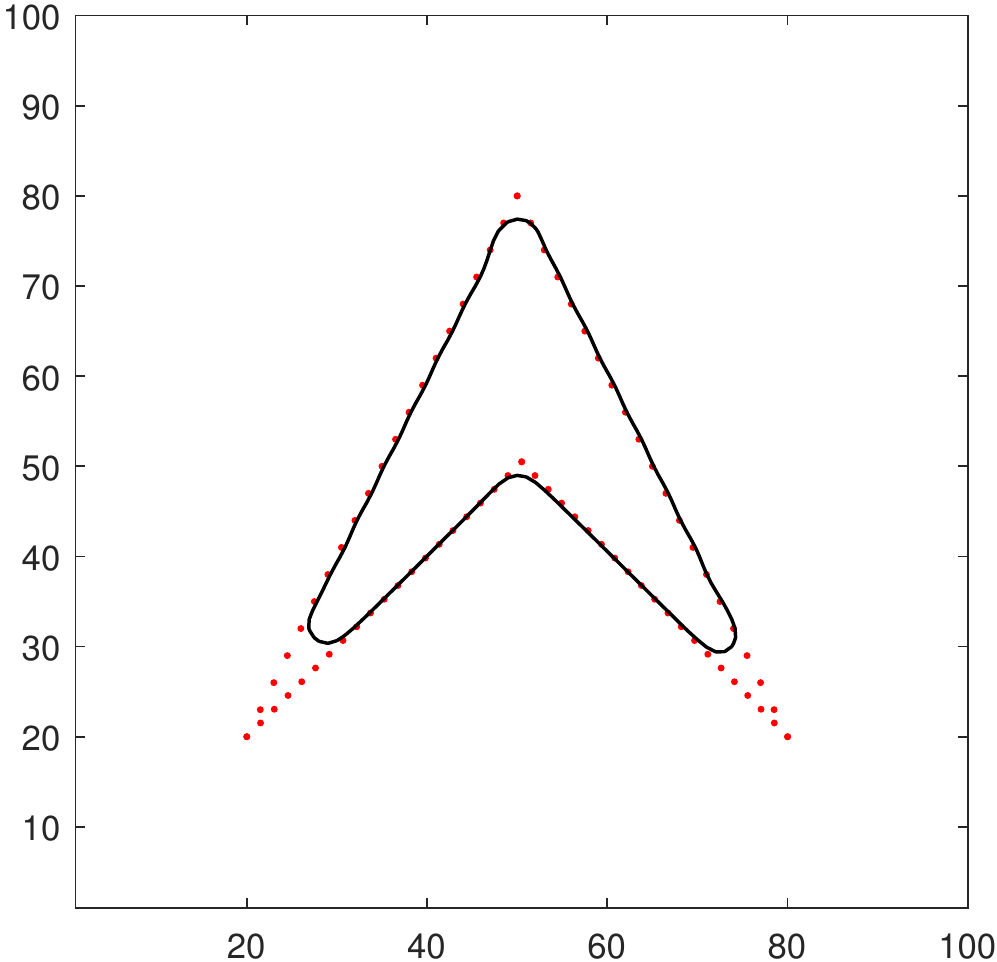}
  \end{tabular}
\caption{Comparison between the algorithm from \cite{EZL*12} and OSM with or without curvature constraints:  The first row results are by the algorithm proposed in \cite{EZL*12} with $r_1=r_2=8,r_3=r_4=3$. The second row results are by OSM without any curvature term, $\eta =0$.  The third row results are by OSM with curvature constraint ($s=2$): (a)  $\eta=3$, (b) $\eta=2$, (c) $\eta=1$, and (d) $\eta=2$.  The shape of the underlying surface are more accurately captured using our proposed model with the curvature constraint. }\label{fig.ex2d}
\end{figure}

As $\eta$ increases, different effect can be shown, see  Figure \ref{fig.model.eta}.
As one increases $\eta$, the two sharp corners are recovered better. However, if $\eta$ is too large, like $3$ and $4$ in this example, the corners get more circular.  As $\eta$, the weight of curvature term, gets larger, the reconstructed curve becomes even more circular to avoid large curvature.
\begin{figure}
  \centering
      \begin{tabular}{ccc}
  (a) & (b) & (c) \\
  \includegraphics[width=0.23\textwidth]{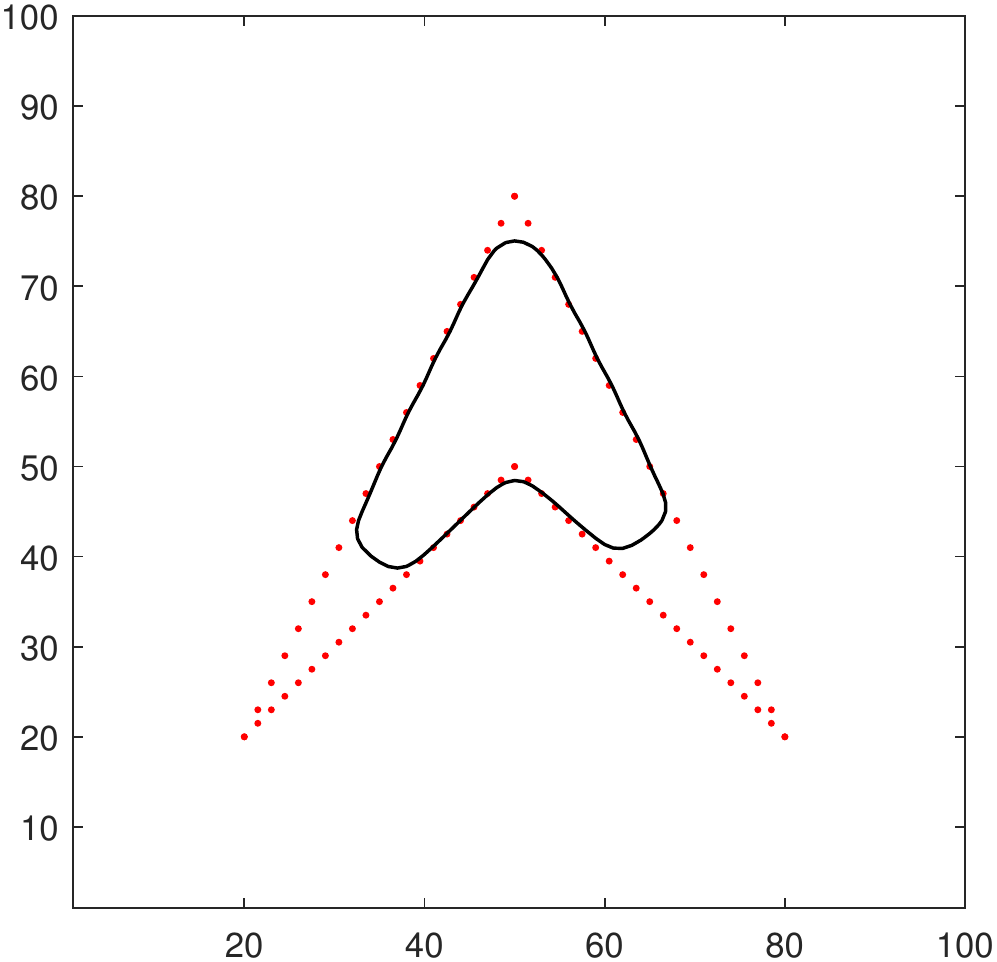} &
  \includegraphics[width=0.23\textwidth]{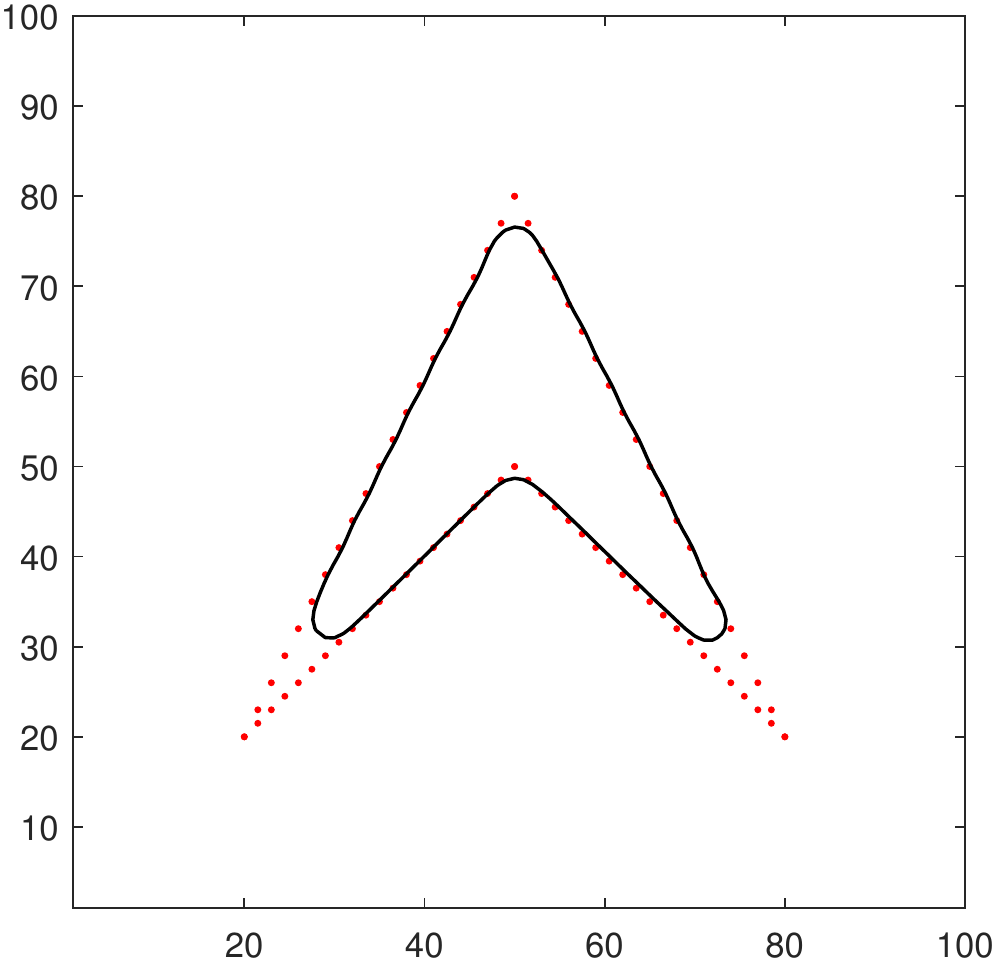} &
  \includegraphics[width=0.23\textwidth]{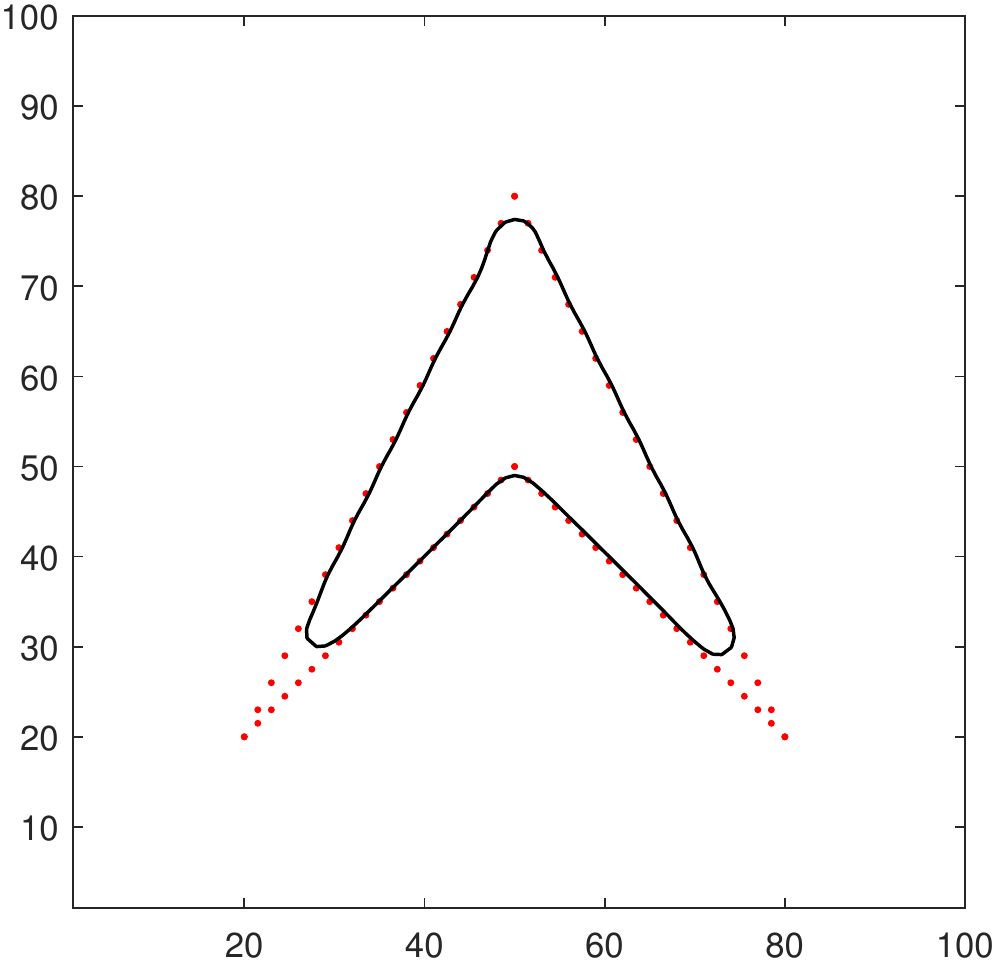}\\
 & (d) & (e) \\
 &
  \includegraphics[width=0.23\textwidth]{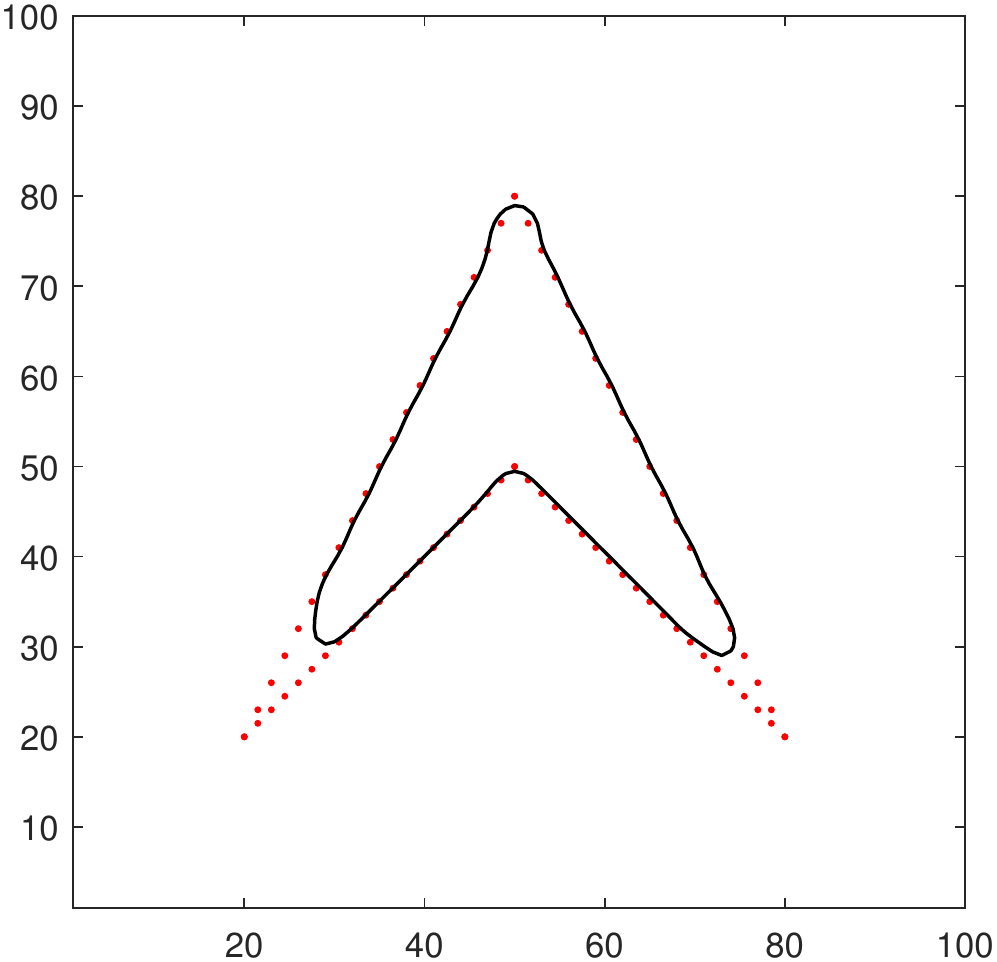}&
  \includegraphics[width=0.23\textwidth]{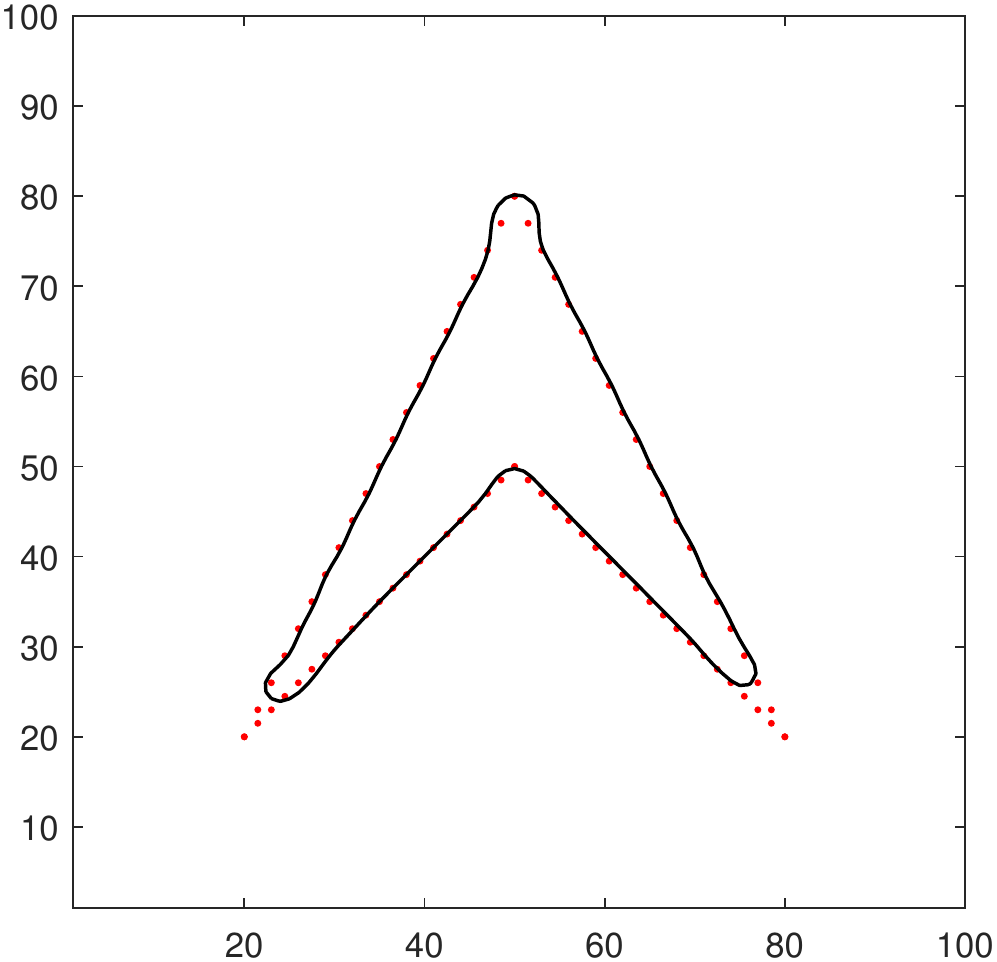}
   \end{tabular}
  \caption{Effect of $\eta$ in OSM. (a) $\eta=0$. (b) $\eta=1$. (c) $\eta=2$. (d) $\eta=3$. (e) $\eta=4$.  As $\eta$ increases, the two sharp corners are recovered better.  As $\eta$ gets larger,  the corners get more circular  to avoid large curvature.  }\label{fig.model.eta}
\end{figure}


Figure \ref{fig.model.sparse} shows results when the given point cloud is sparse. The point cloud for (a) and (b) is a boomerang shape and  that for (c) and (d) is a sparse square with an indent at the bottom.  With sparse boomerang data, just using the distance term ($\eta=0$) can recover very limited part of the given point cloud; see Figure \ref{fig.model.sparse} (a).  With  $\eta=4$, We recover the general shape of boomerang in Figure \ref{fig.model.sparse} (b).  For the sparse square shape, only eight corners and one point on each side are given in Figure \ref{fig.model.sparse} (c)-(d).  While for $\eta=0$, the bottom part of the recovered shape is smooth.  With $\eta=2$, the rectangle shape is clearly recovered   showing corners in the bottom area in Figure \ref{fig.model.sparse} (d).
\begin{figure}
  \centering
  \begin{tabular}{cccc}
  (a) & (b) & (c) & (d) \\
\includegraphics[width=0.22\textwidth]{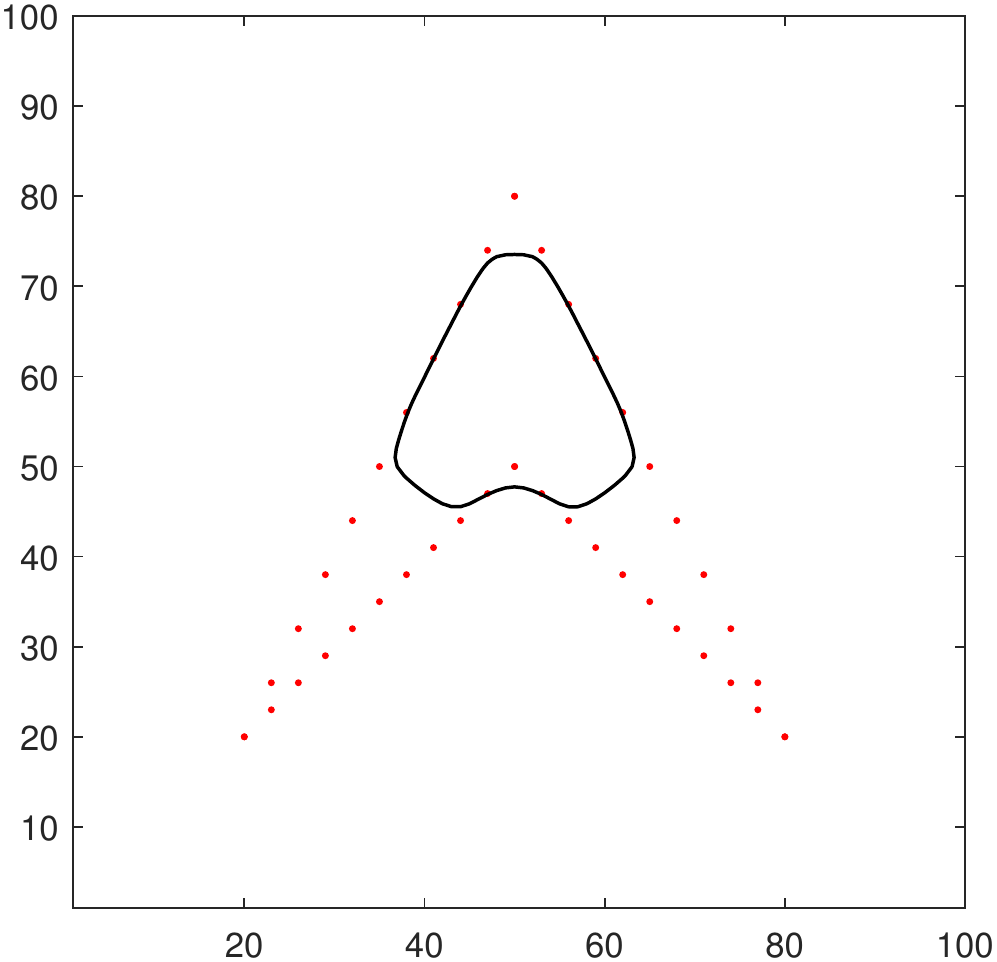} &
\includegraphics[width=0.22\textwidth]{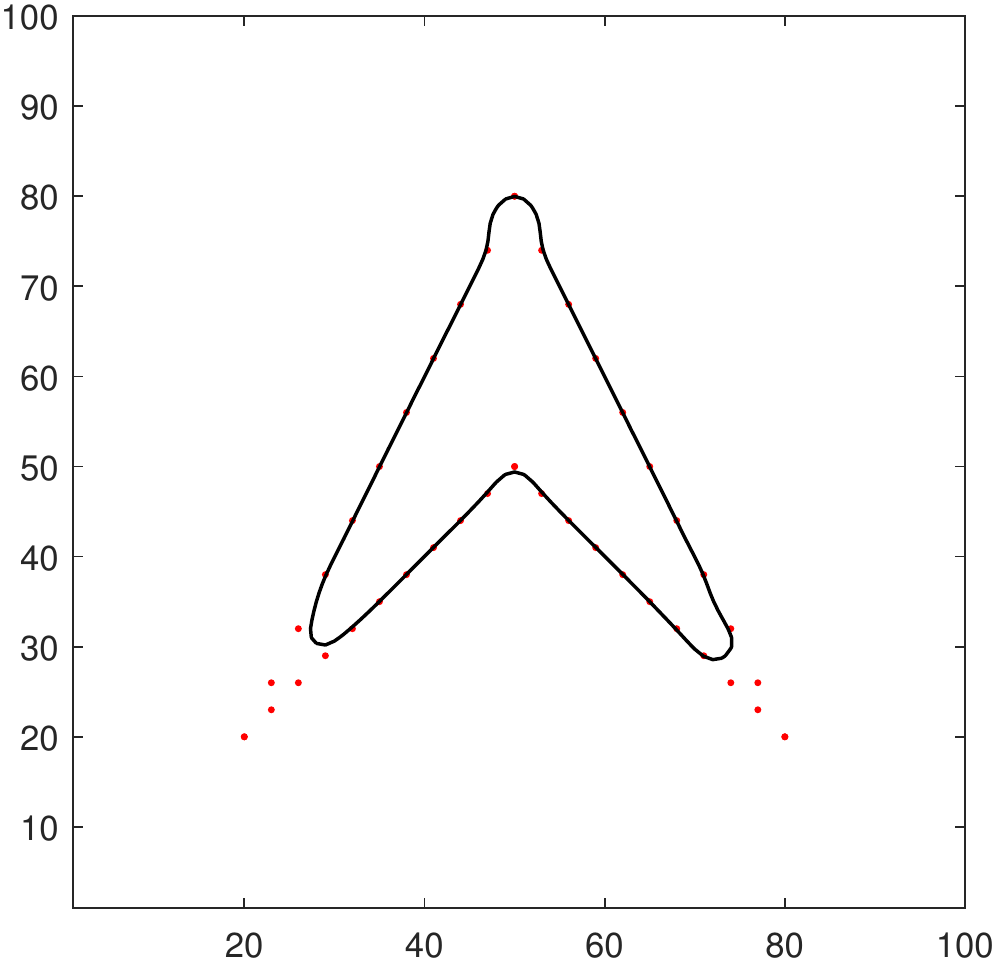} &
\includegraphics[width=0.22\textwidth]{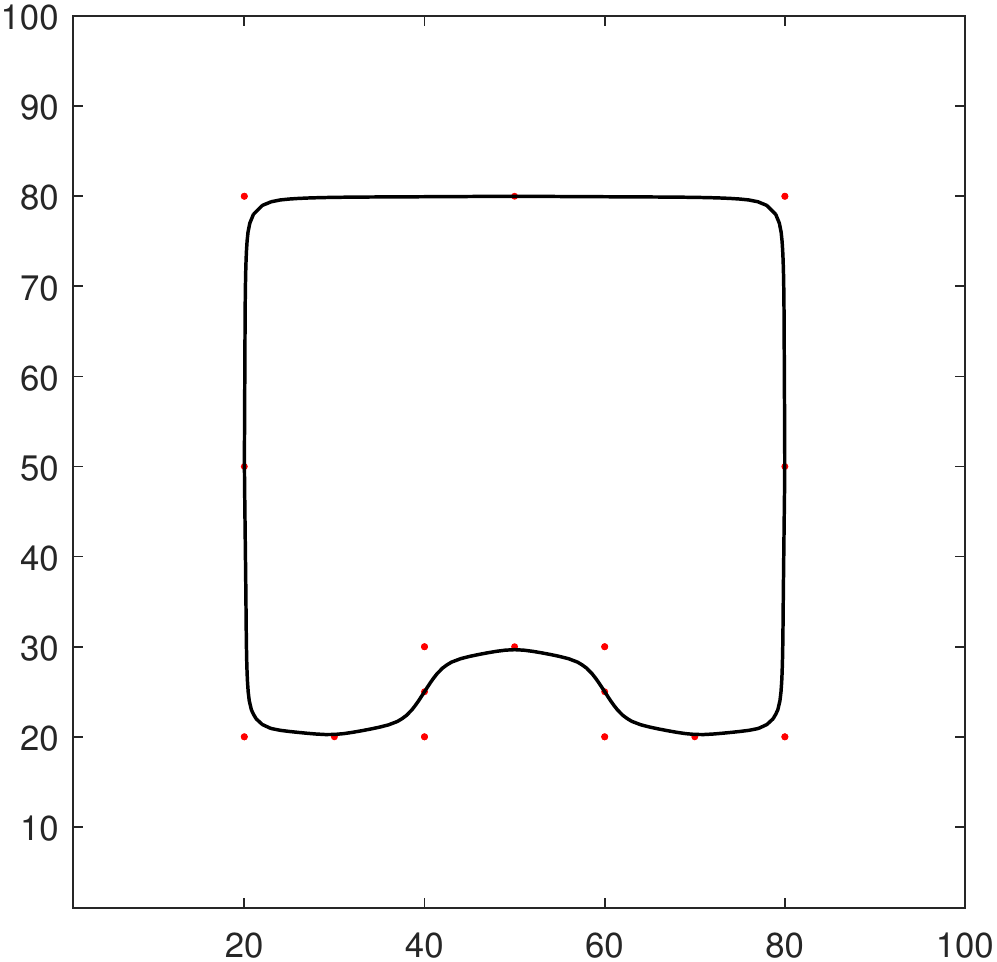} &
\includegraphics[width=0.22\textwidth]{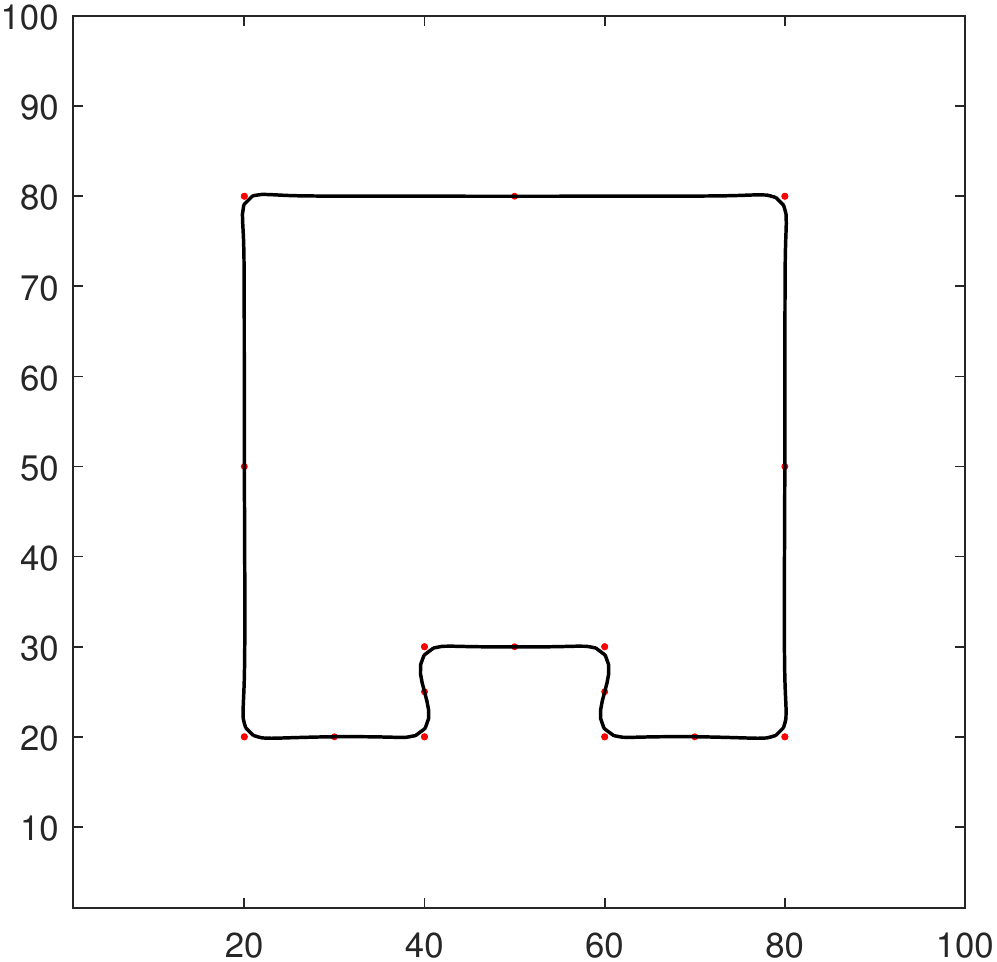}
\end{tabular}
  \caption{By OSM, sparse data results with or without curvature constraints.   (a) $\eta=0$, and (b) $\eta=4$ for a point cloud sampled from a Boomerang shape.   (c) $\eta=0$ and (d) $\eta=2$, for a sparse square shape where only eight corners and one point on each side are given.   For both examples, with curvature constraint, the recovery is more accurate and sharper. }\label{fig.model.sparse}
\end{figure}
Figure \ref{fig.model.extreme} shows the case where even less number of points are given. See Figure \ref{fig.model.extreme} (a).   Only two points around each corner are given.  Figure \ref{fig.model.extreme} (b) and (c) show results with $\eta=1$  and $\eta=1.5$, respectively.  Even with extremely sparse data, curvature constraint model can reconstruct the corners well. We observe that when the given point cloud is non-uniform, or data are missing in some region, our algorithm yields results with straight edges between distant points and smooth corners.  These are due to the fact that the curvature along a straight edge is zero, and smooth corners have smaller curvature than sharp corners for $s=2$ in the discrete setting.
\begin{figure}
  \centering
  \begin{tabular}{cccc}
  (a) & (b) & (c) \\
\includegraphics[width=0.25\textwidth]{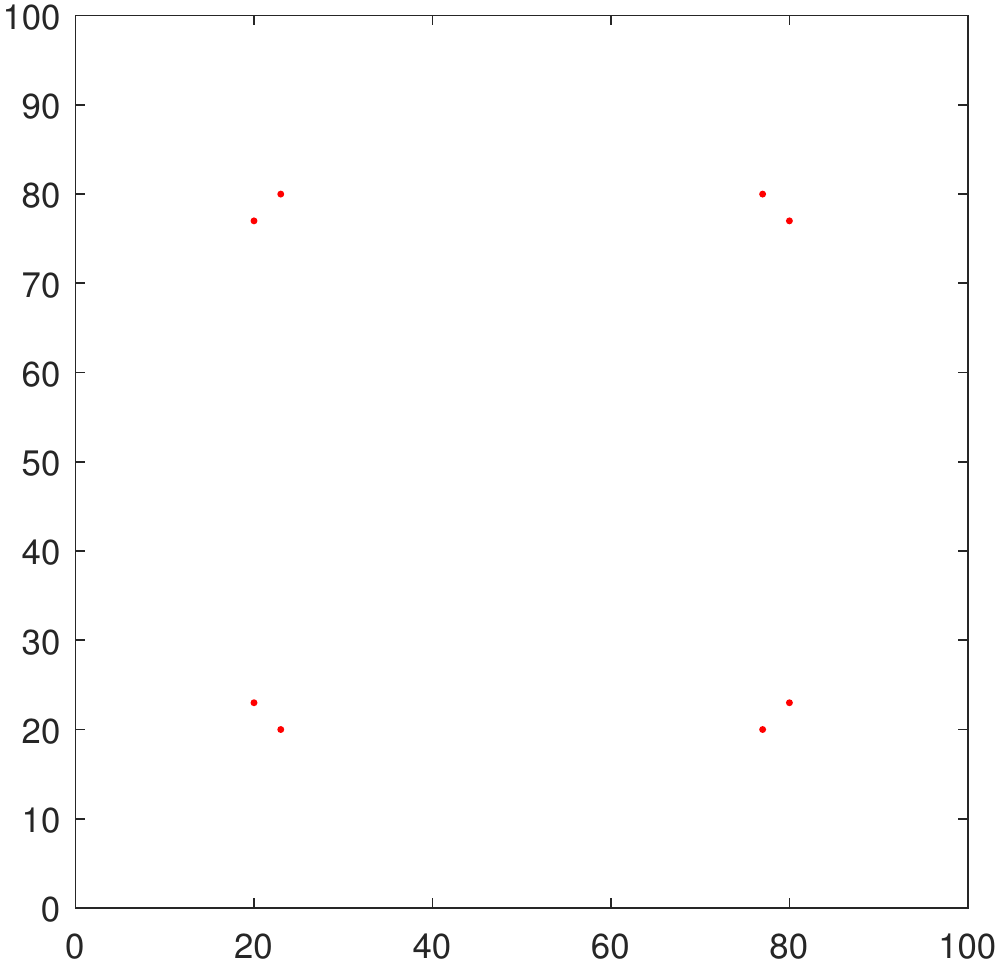} &
\includegraphics[width=0.25\textwidth]{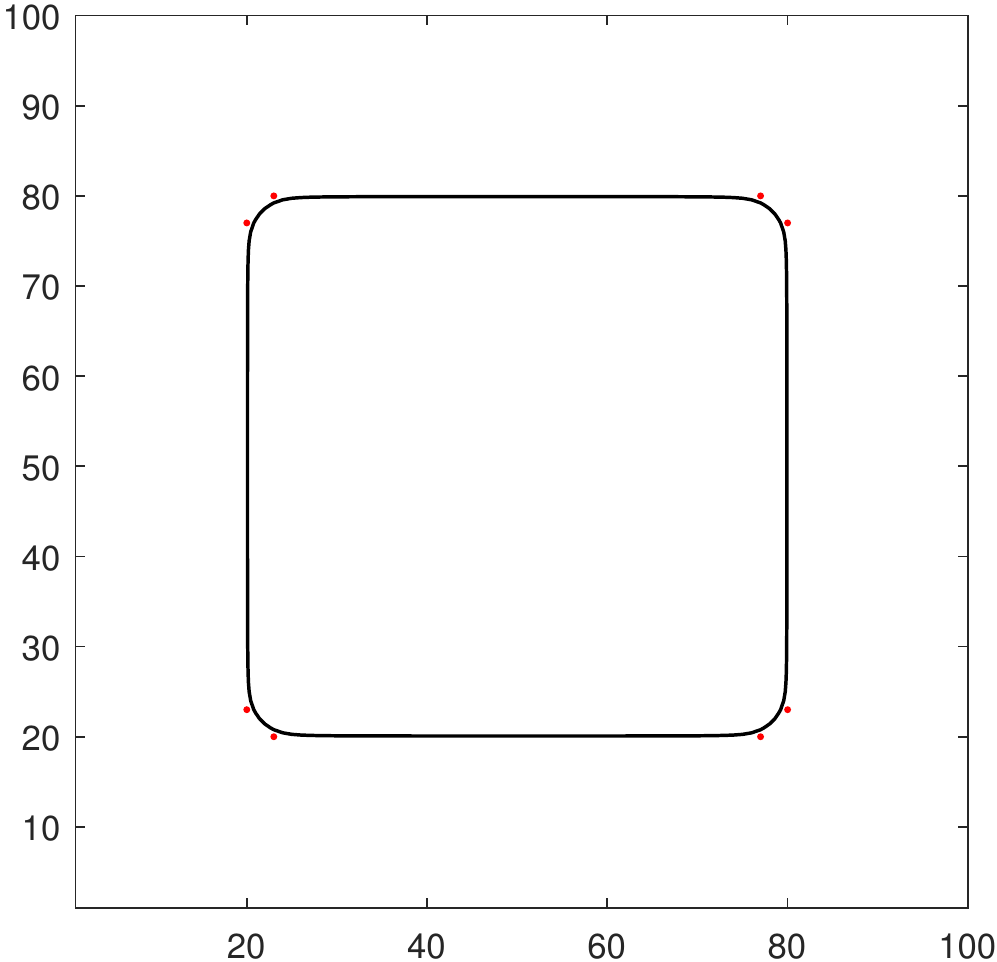} &
\includegraphics[width=0.25\textwidth]{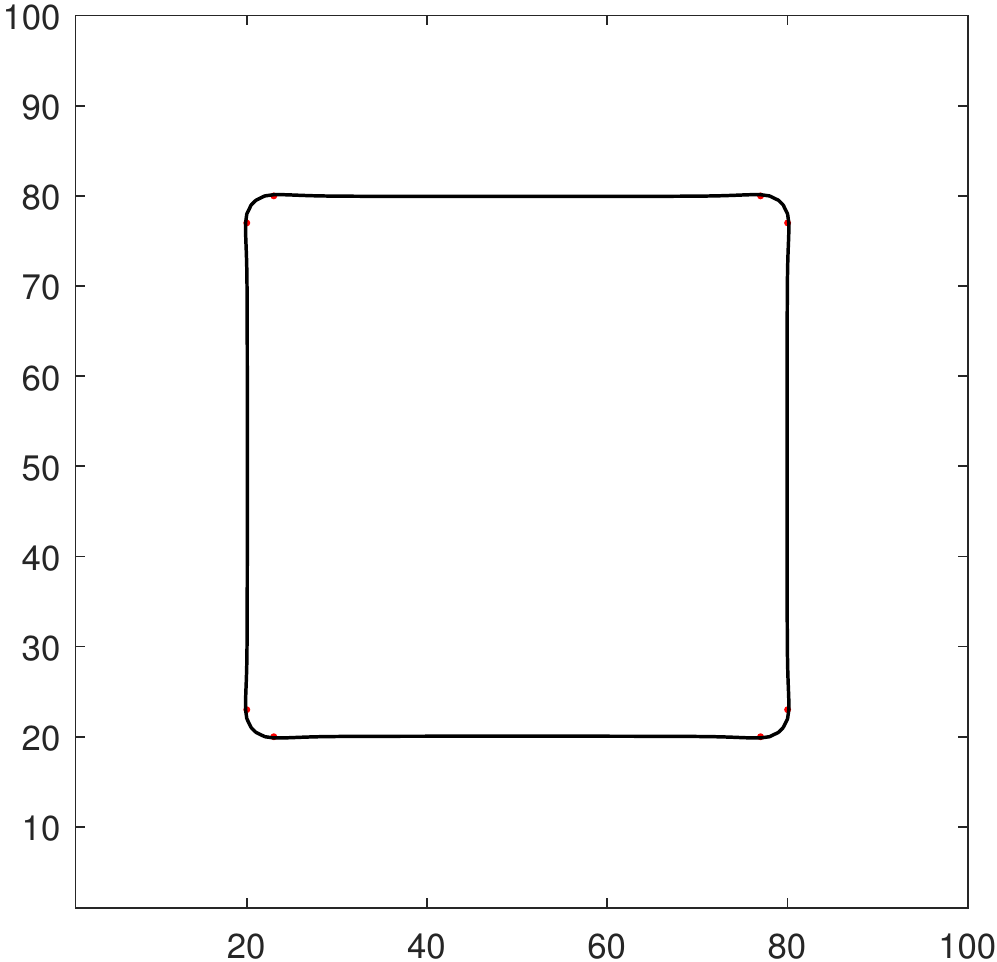}
\end{tabular}
  \caption{By OSM, extremely sparse data: (a) Given data.  (b) The  recovered result with $\eta=1$  and (c) with $\eta=1.5$.  Even with extremely sparse data, curvature constraint model can reconstruct the square corners well.}\label{fig.model.extreme}
\end{figure}

The next experiment is for the noisy boomerang data, where Gaussian noise with standard deviation 1 is added to the locations of the point cloud. The results with $\eta=0,1,2$ are shown in Figure \ref{fig.model.noise}.  As $\eta$ gets larger, the two lower corners get recovered better.  Even with noisy data, OSM shows a strong competence  of recovering the sharp corners.
\begin{figure}
  \centering
    \begin{tabular}{cccc}
  (a) & (b) & (c) \\
  \includegraphics[width=0.3\textwidth]{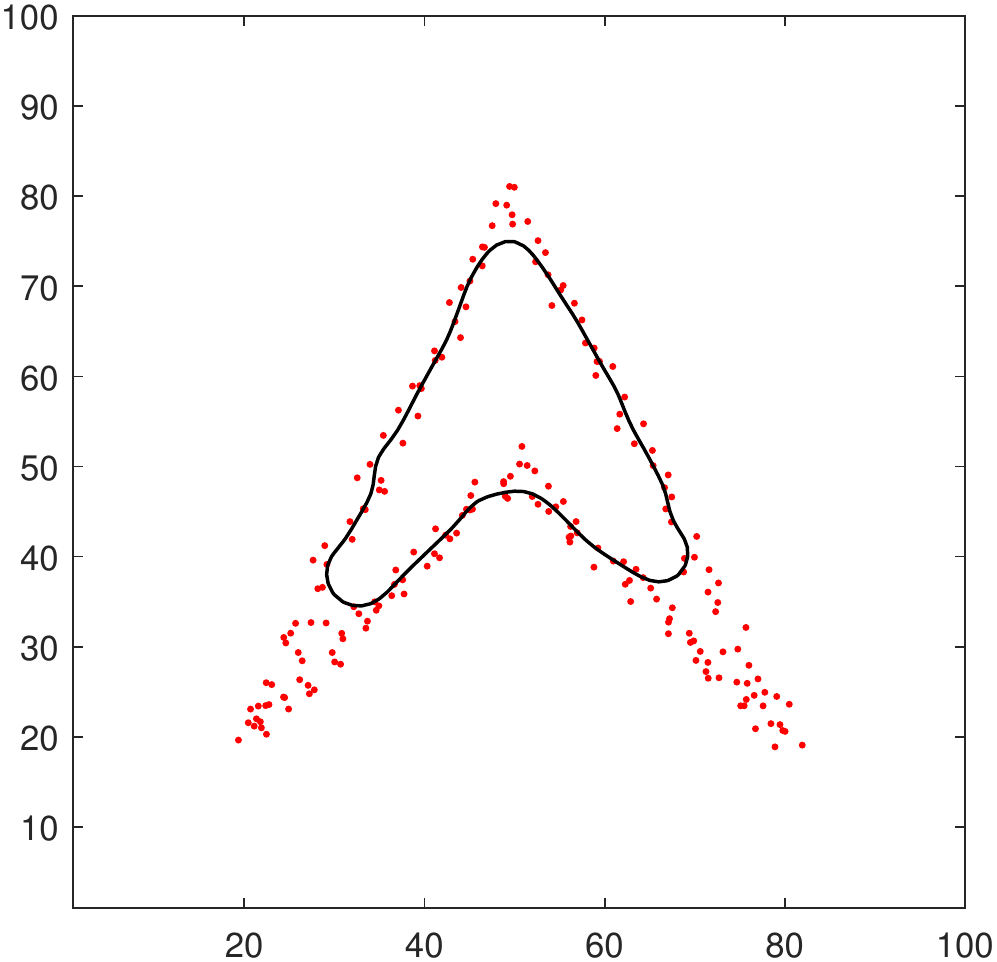} &
  \includegraphics[width=0.3\textwidth]{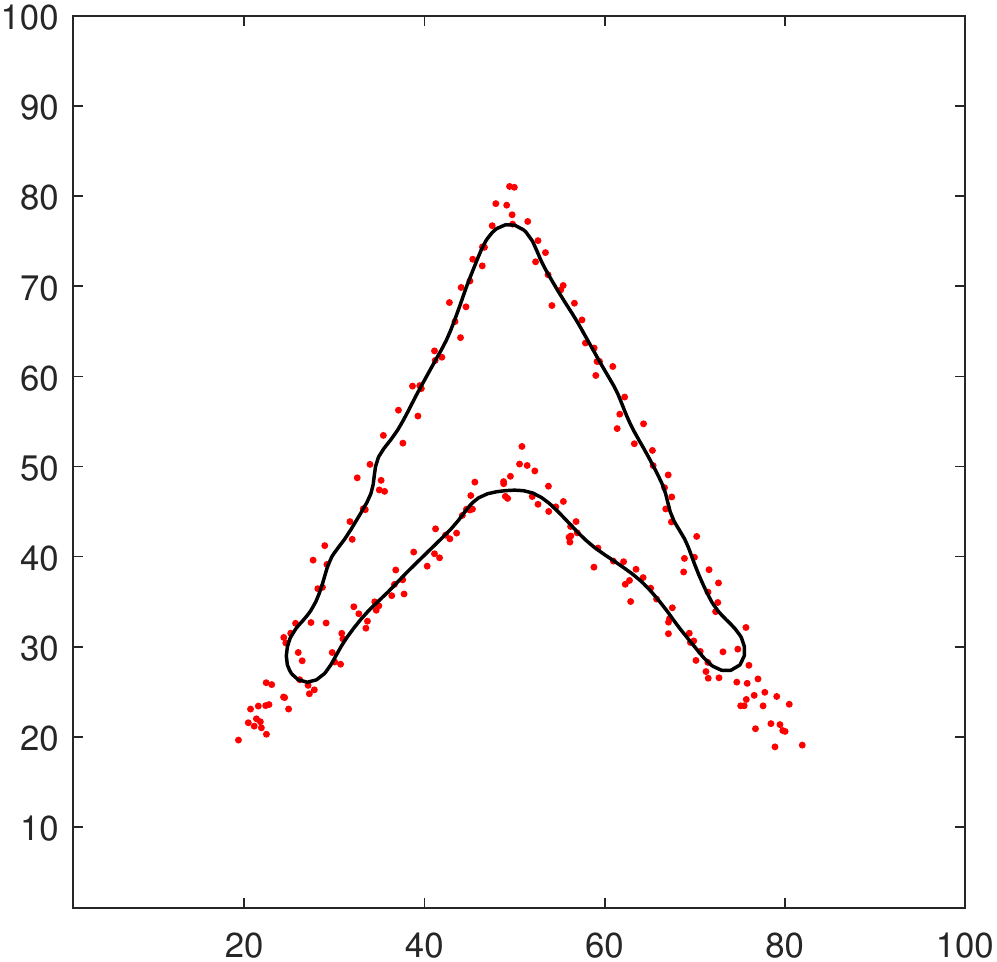}&
  \includegraphics[width=0.3\textwidth]{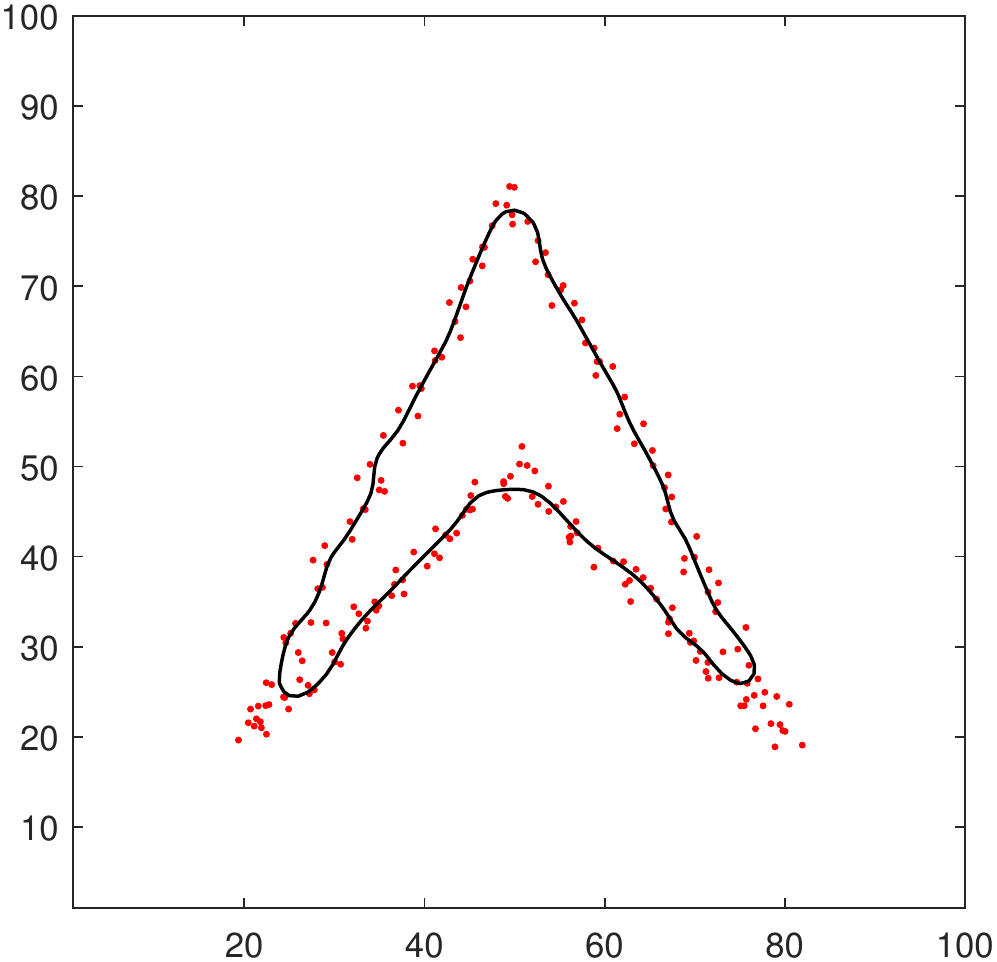}
   \end{tabular}
  \caption{By OSM, reconstruction with noisy data: (a) $\eta=0$, (b) $\eta=1$, (c) $\eta=2$. The noise is Gaussian with standard deviation 1. As $\eta$ gets larger, the two lower corners are better recovered.}\label{fig.model.noise}
\end{figure}

\subsection{Three Dimensional Examples}
We conclude this section with experiments of reconstruction of surfaces in three dimensional space. We use OSM with $s=2$ to reconstruct the pyramid, the yoyo, and the ice cream cone, whose point clouds are shown in Figure \ref{fig.data3d}. The data in these examples are concentrated within a cube $[0,50]^3$. The pyramid has a relatively simple geometry structure: it is convex and its surface only consists of five planes. We can use a large time step $\Delta t=500$. For the yoyo and the ice-cream cone we use $\Delta t=100$, since the underlying surfaces have more details, e.g.,  the neck of the yoyo and the upper concave part of the ice cream cone.
\begin{figure}
  \centering
      \begin{tabular}{cccc}
  (a) & (b) & (c) \\
  \includegraphics[height=0.25\textwidth]{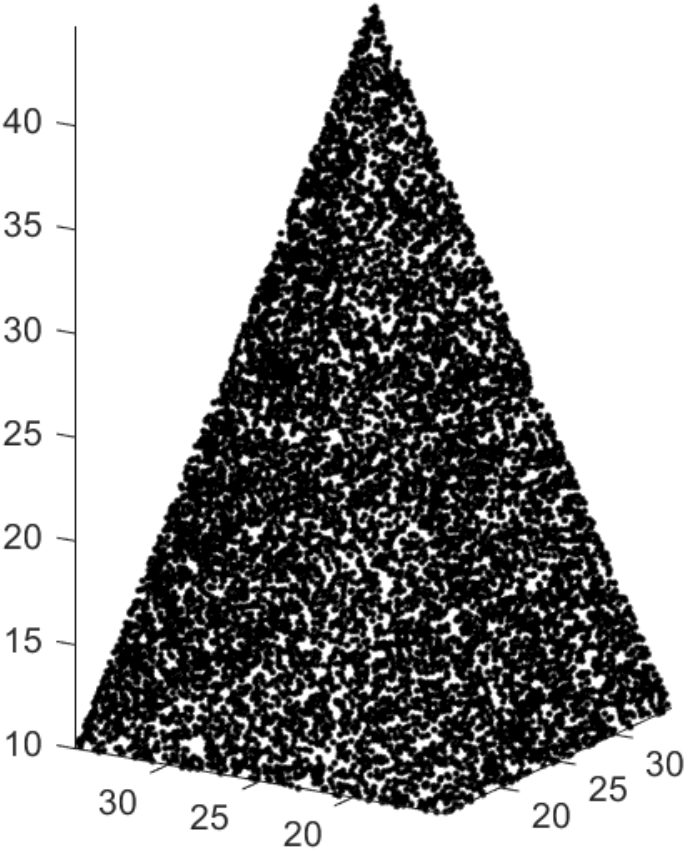} &
  \includegraphics[height=0.25\textwidth]{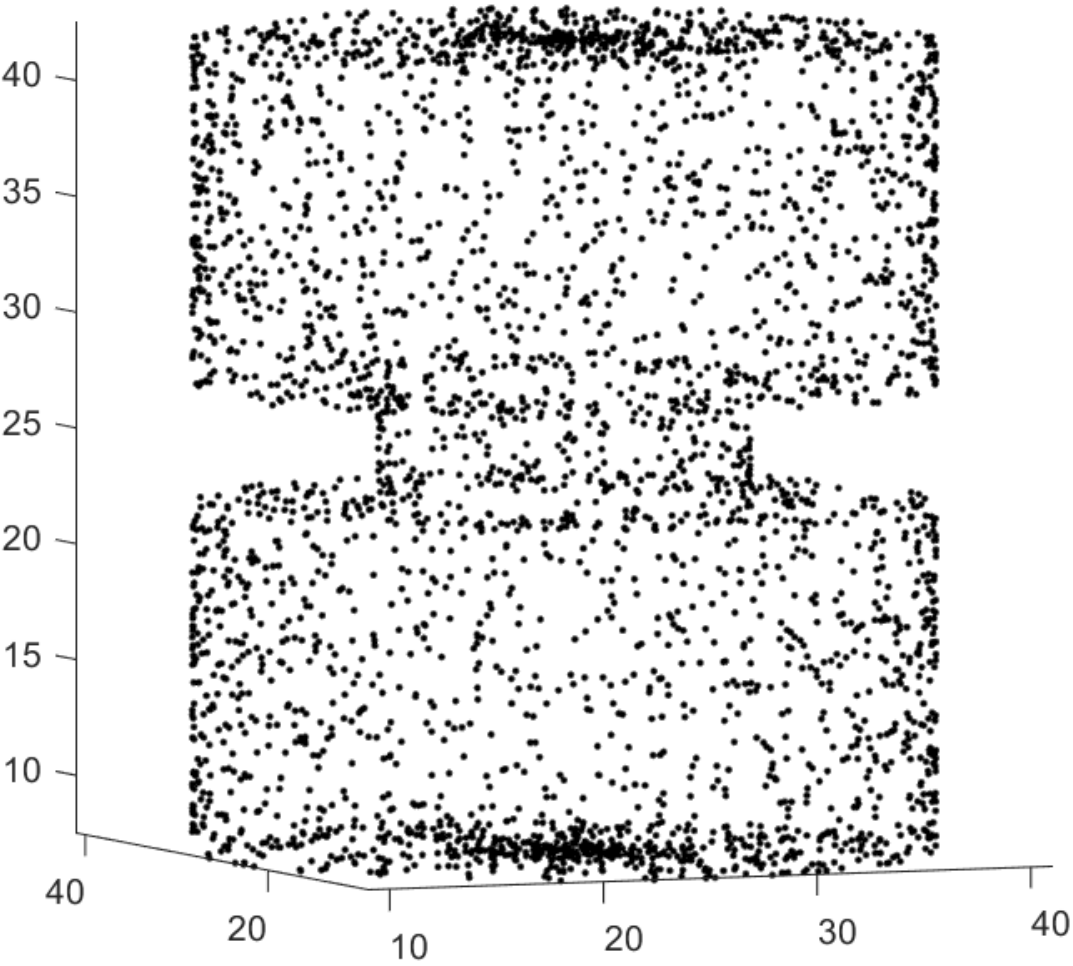}&
  \includegraphics[height=0.25\textwidth]{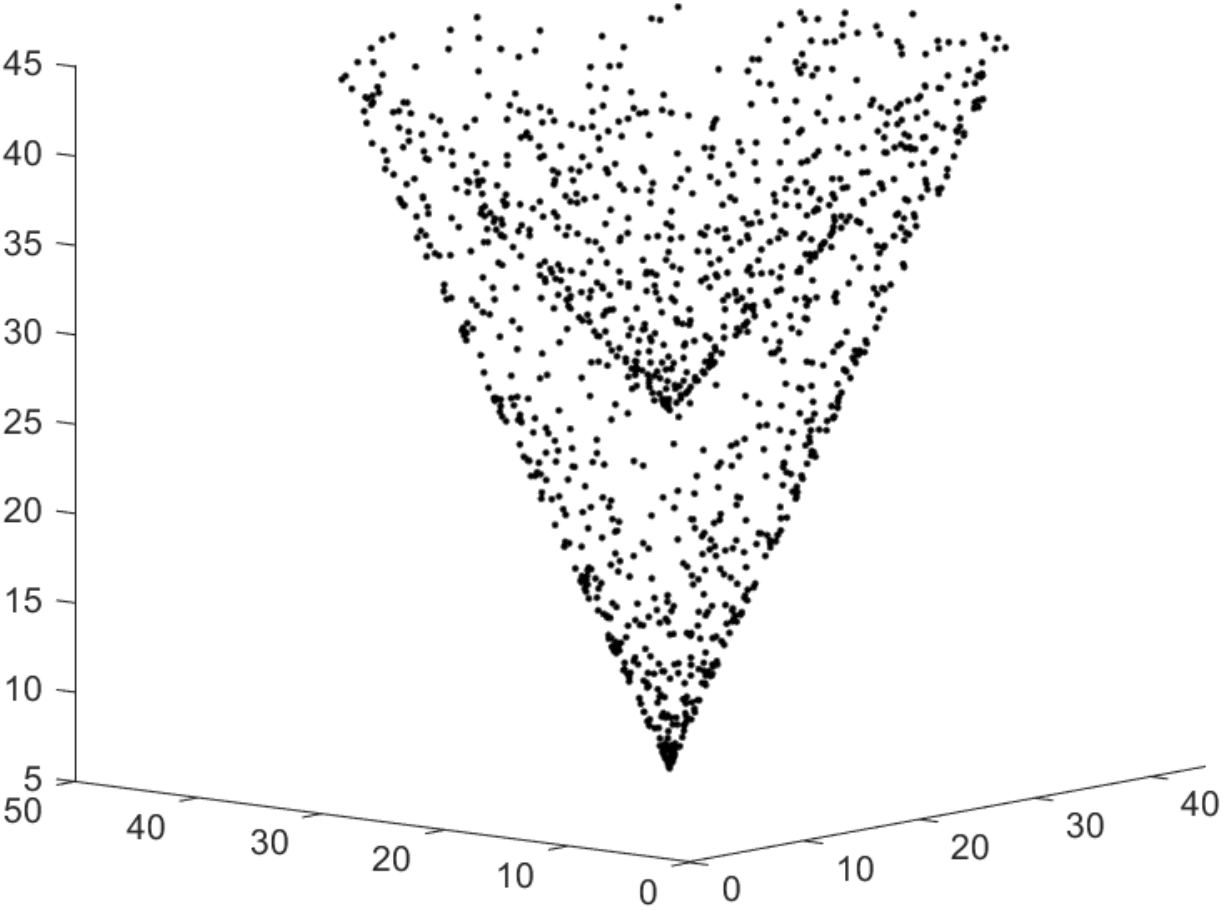}
   \end{tabular}
  \caption{Examples of three dimensional point cloud data. (a) A pyramid. (b) A yoyo. (c) An ice cream cone.}\label{fig.data3d}
\end{figure}

For the pyramid, the reconstructed surfaces with $\eta=0,5,10$ and the comparison of cross sections along $y=25$ (a middle section) are shown in Figure \ref{fig.pyramid}. In this case, we see limited improvements of capturing the vertices when the curvature constraint is included.
\begin{figure}
  \centering
      \begin{tabular}{cccc}
  (a) & (b) & (c) \\
  \includegraphics[width=0.28\textwidth]{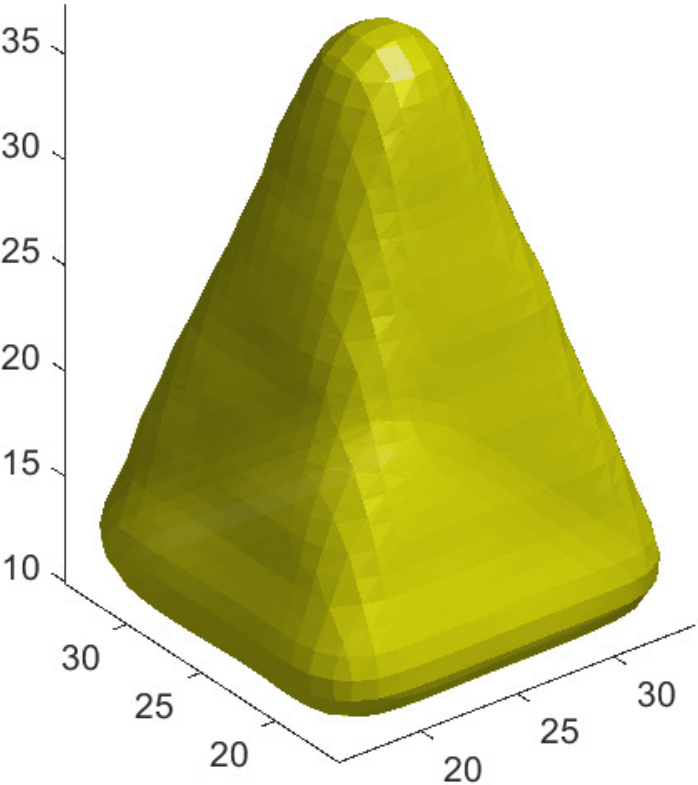}&
  \includegraphics[width=0.28\textwidth]{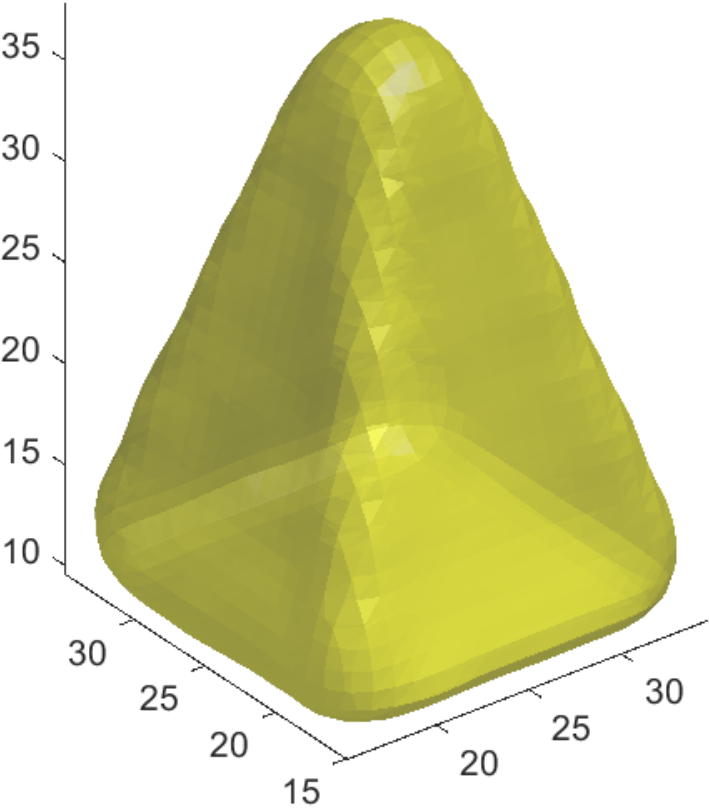}&
  \includegraphics[width=0.28\textwidth]{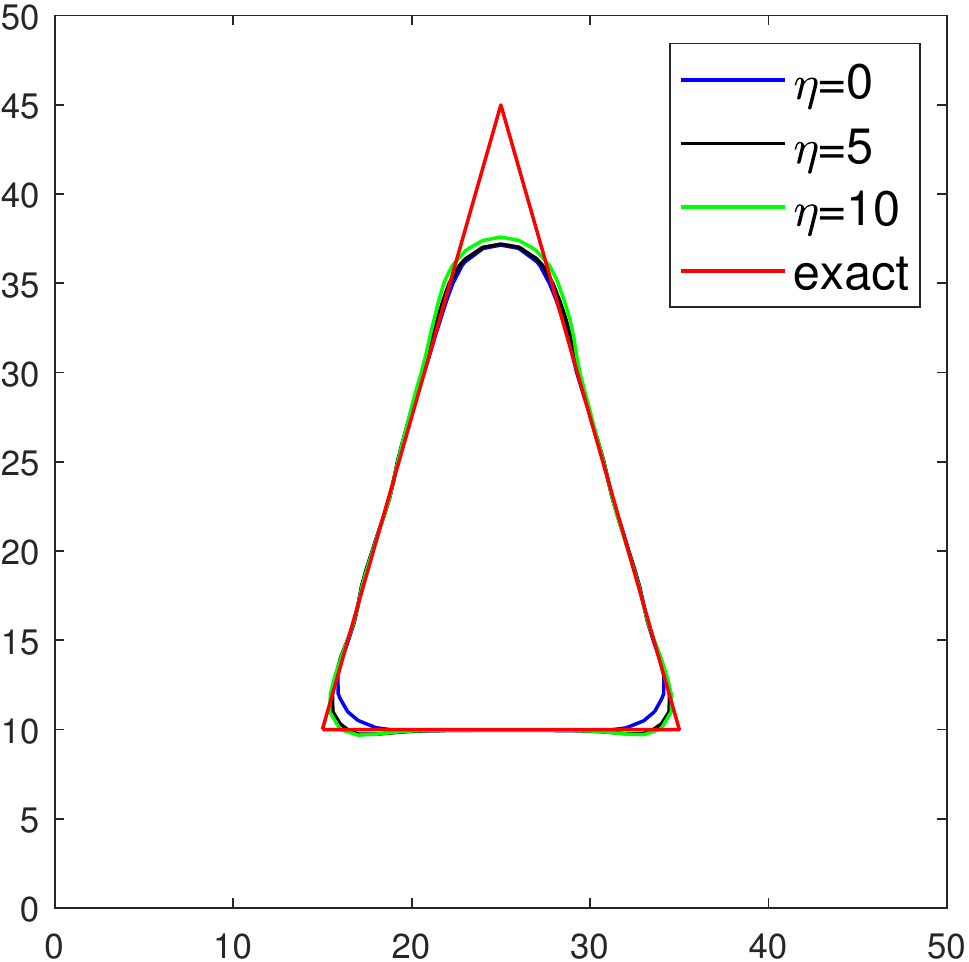}
   \end{tabular}
  \caption{Reconstruction of the pyramid by OSM with $s=2$: (a) Result with $\eta=0$. (b) Result with $\eta=10$. (c) Comparison of cross section along $y=25$.}\label{fig.pyramid}
\end{figure}

For the yoyo, the reconstructed surface with $\eta=0,5$ and the comparison of cross sections along $y=25$ are shown in Figure \ref{fig.part}. The advantage of the curvature term is obvious. With $\eta=0$, the solution attains the energy balance between surface area and distance to the data at some location away from the middle neck part. Since the curvature at that part is non-zero, given a positive $\eta$, the surface further evolves to capture the neck. For comparison, the reconstructed surface by the algorithm in \cite{EZL*12} is shown in Figure \ref{fig.part} (a). Similar to the result by OSM with $\eta=0$, \cite{EZL*12} fails to capture the neck part.
\begin{figure}
  \centering
      \begin{tabular}{cc}
  (a)&(b)  \\
  \includegraphics[width=0.4\textwidth]{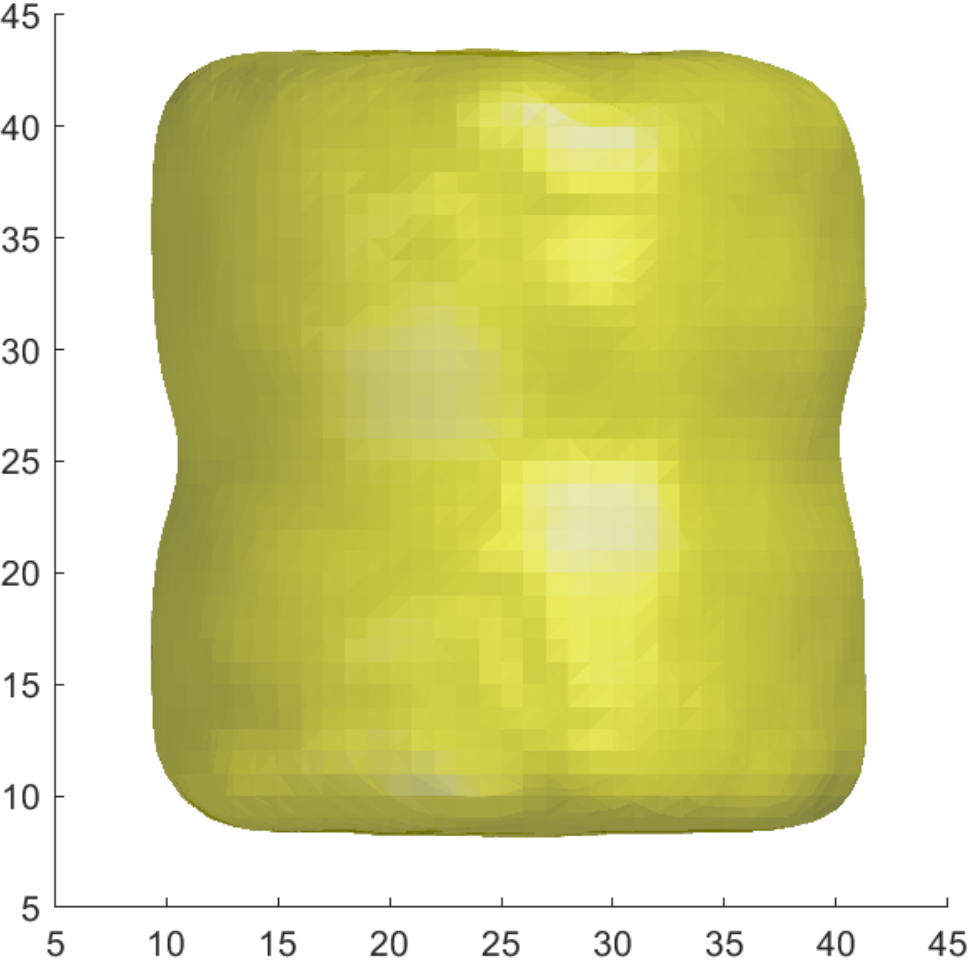}&
  \includegraphics[width=0.4\textwidth]{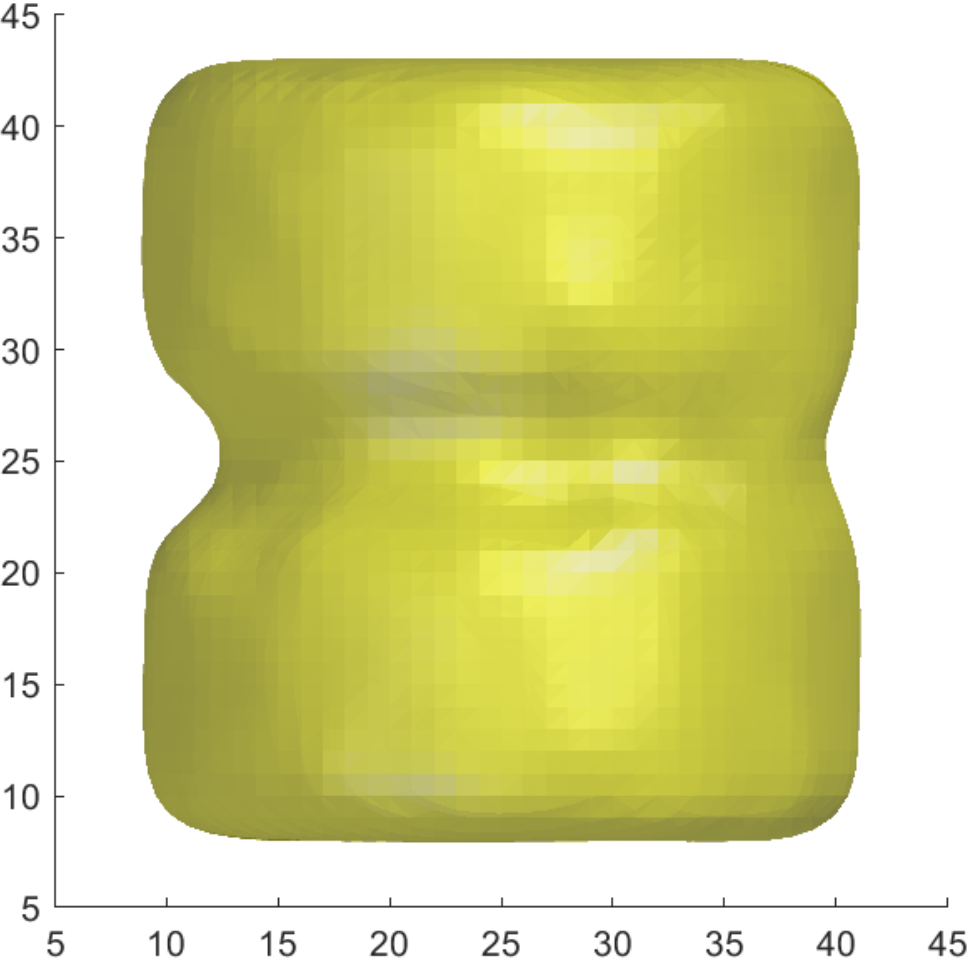}\\
  (c) & (d)\\
  \includegraphics[width=0.4\textwidth]{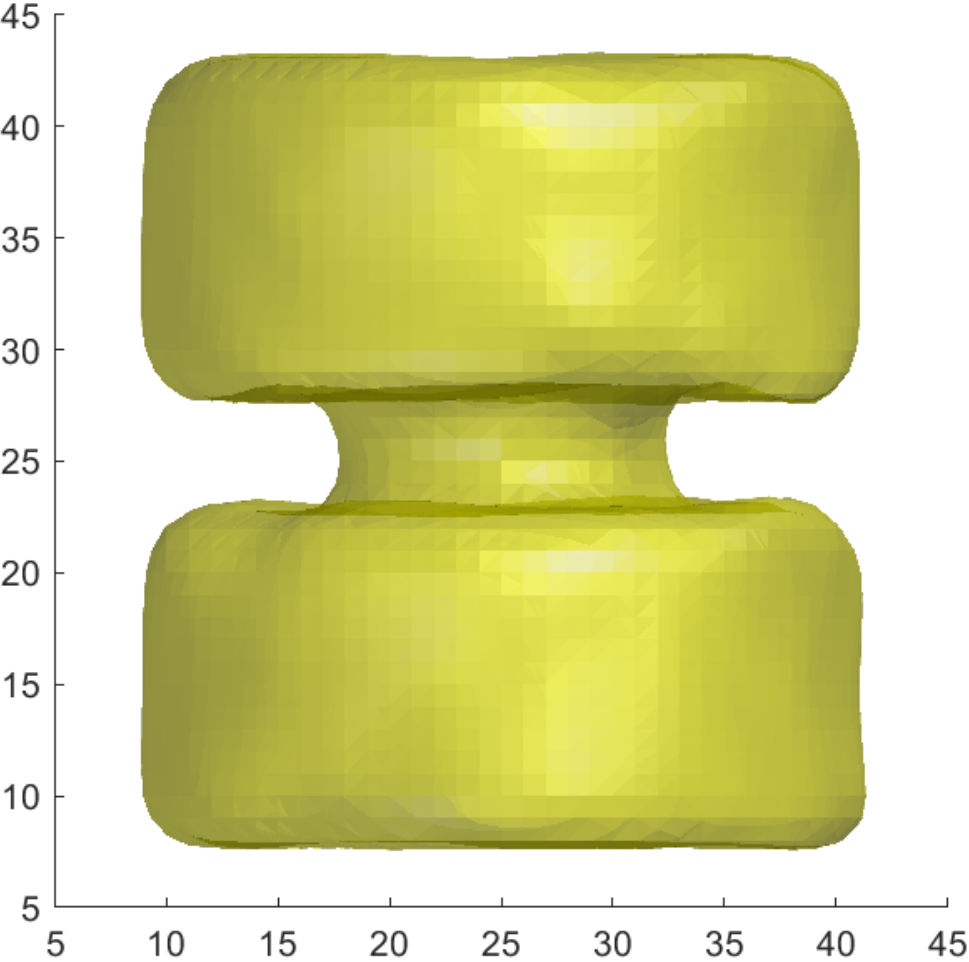}&
  \includegraphics[width=0.4\textwidth]{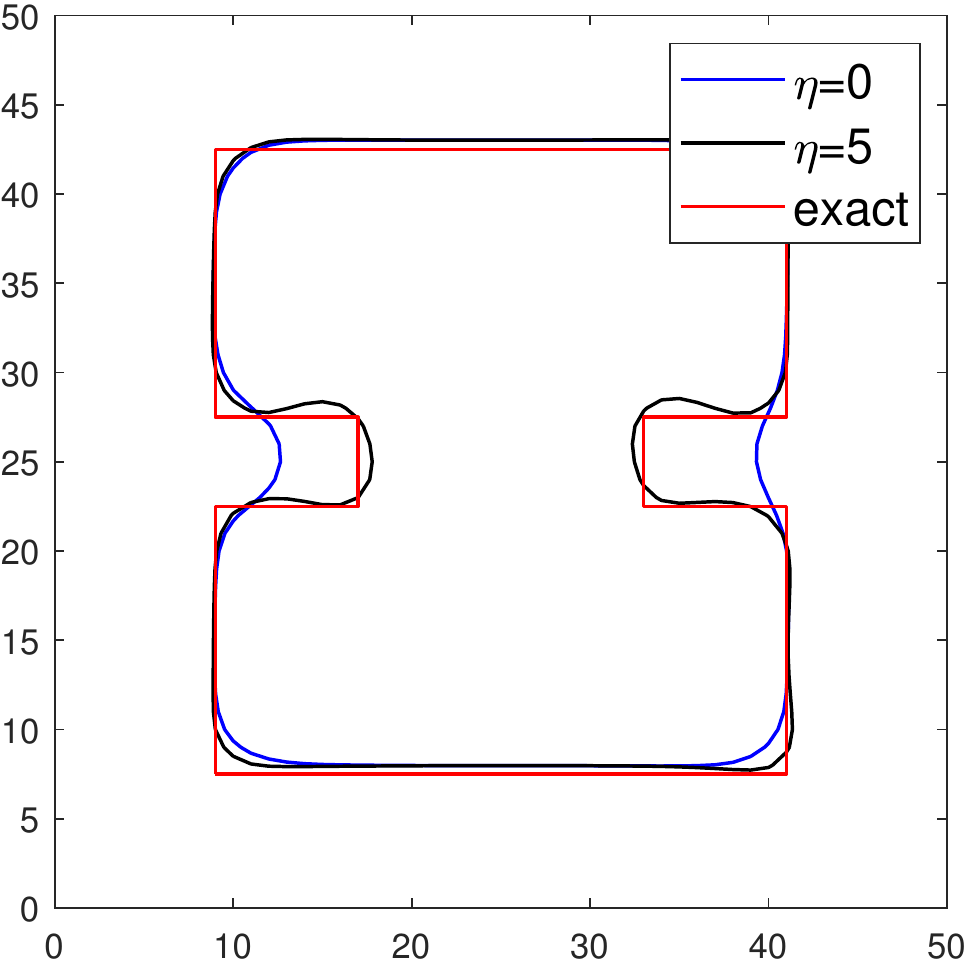}
  \end{tabular}
  \caption{Reconstruction of the yoyo by the algorithm from \cite{EZL*12} and OSM with $s=2$: (a) Result by the algorithm proposed in \cite{EZL*12} with $r_1=r_2=8, r_3=r_4=3$. (b)
  Result by OSM with $\eta=0$. (c) Result by OSM with $\eta=5$. (c) Comparison of cross sections of results by OSM along $y=25$.}\label{fig.part}
\end{figure}

The ice cream cone surface consists of two layers and its cross section looks like a boomerang.   For this example, if we use $\eta=0$, the solution shrinks to a point and then disappears. The reconstructed surfaces with $\eta=5,10$ and the comparison of cross sections along $y=25$ are shown in Figure \ref{fig.ice} (b)-(d). The effect of the value of $\eta$ on this ice cream cone is similar to that on the boomerang. Results with larger values of $\eta$ capture better the features of the underlying surface such as  corners. The reconstructed surface by the algorithm in \cite{EZL*12} is shown in Figure \ref{fig.ice} (a). \cite{EZL*12} recovers the bottom corner better but fails to reconstruct the upper concave part of the surface.
\begin{figure}
  \centering
      \begin{tabular}{cc}
  (a)&(b)  \\
  \includegraphics[width=0.4\textwidth]{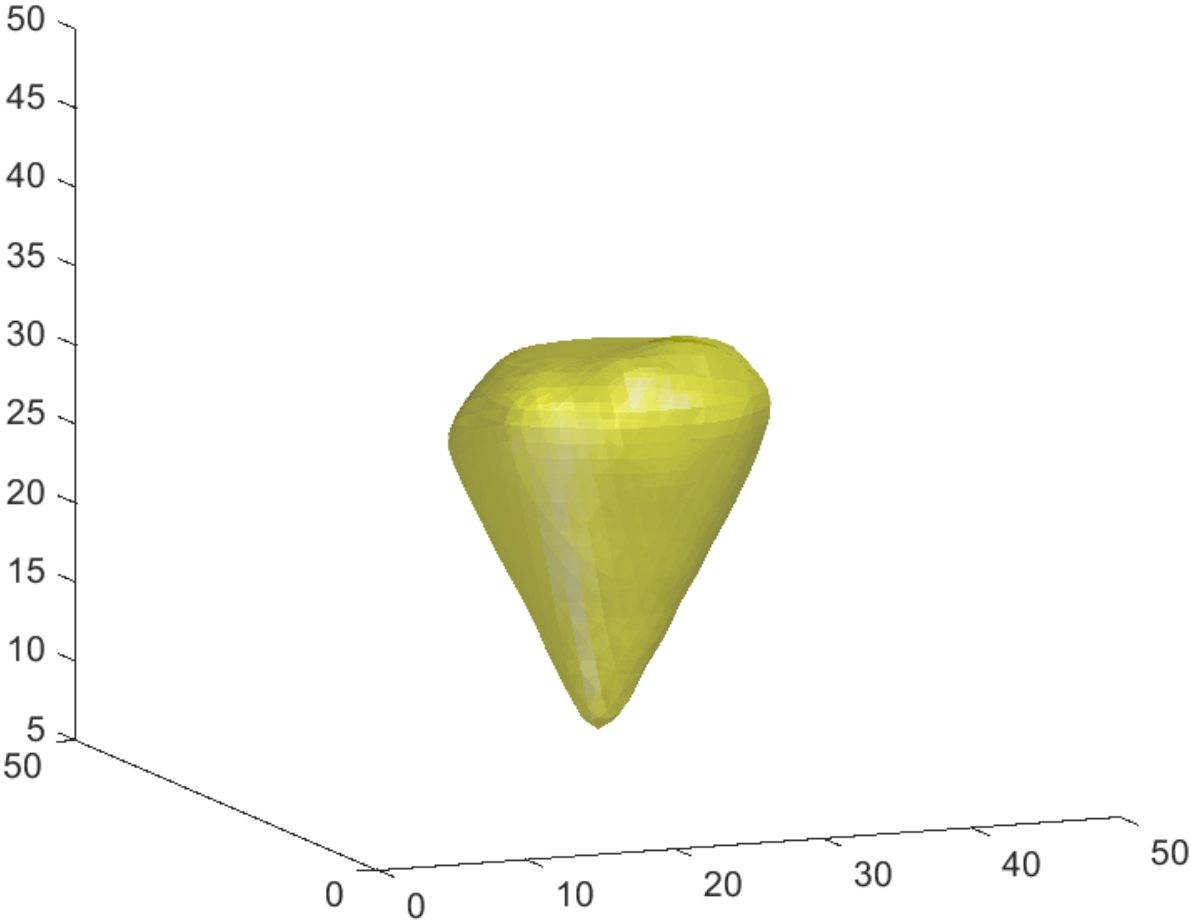}&
  \includegraphics[width=0.4\textwidth]{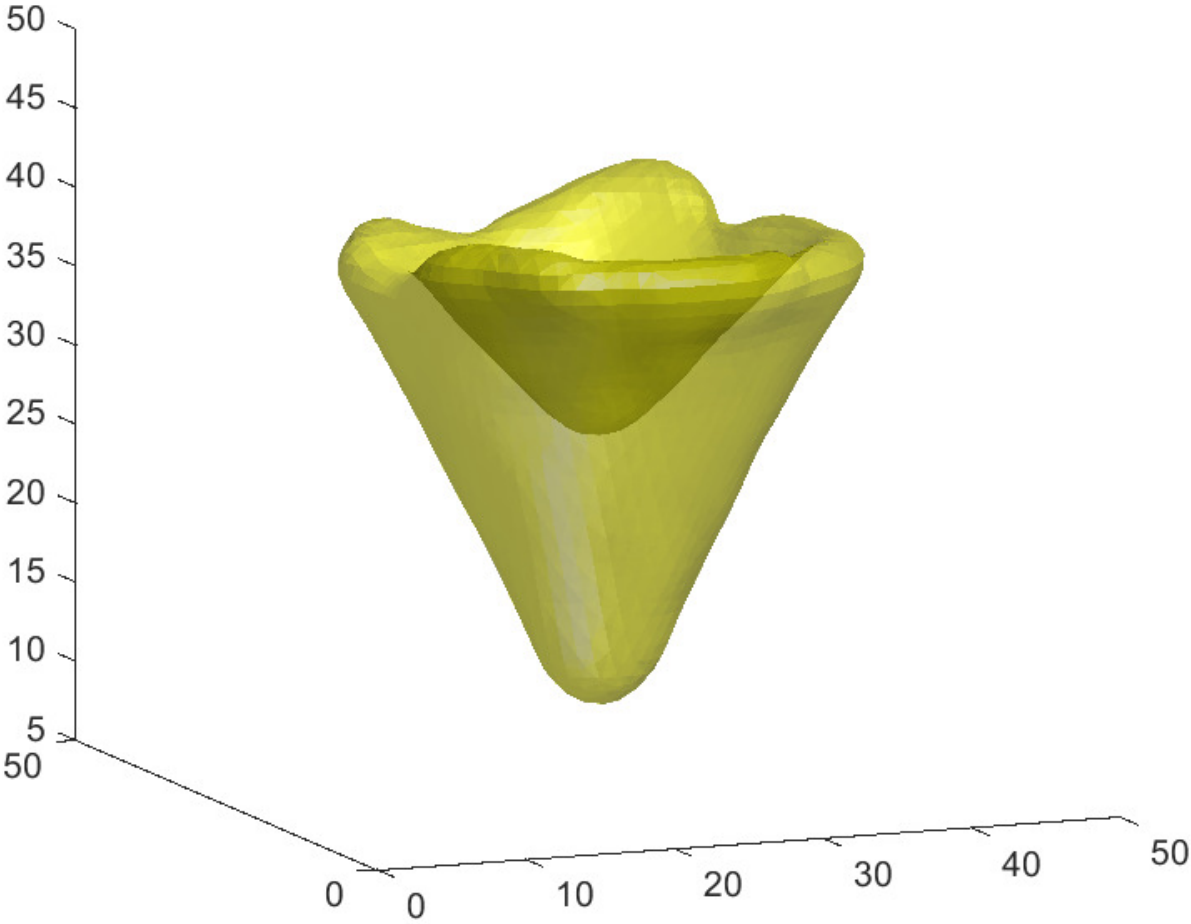}\\
  (c) & (d)\\
  \includegraphics[width=0.4\textwidth]{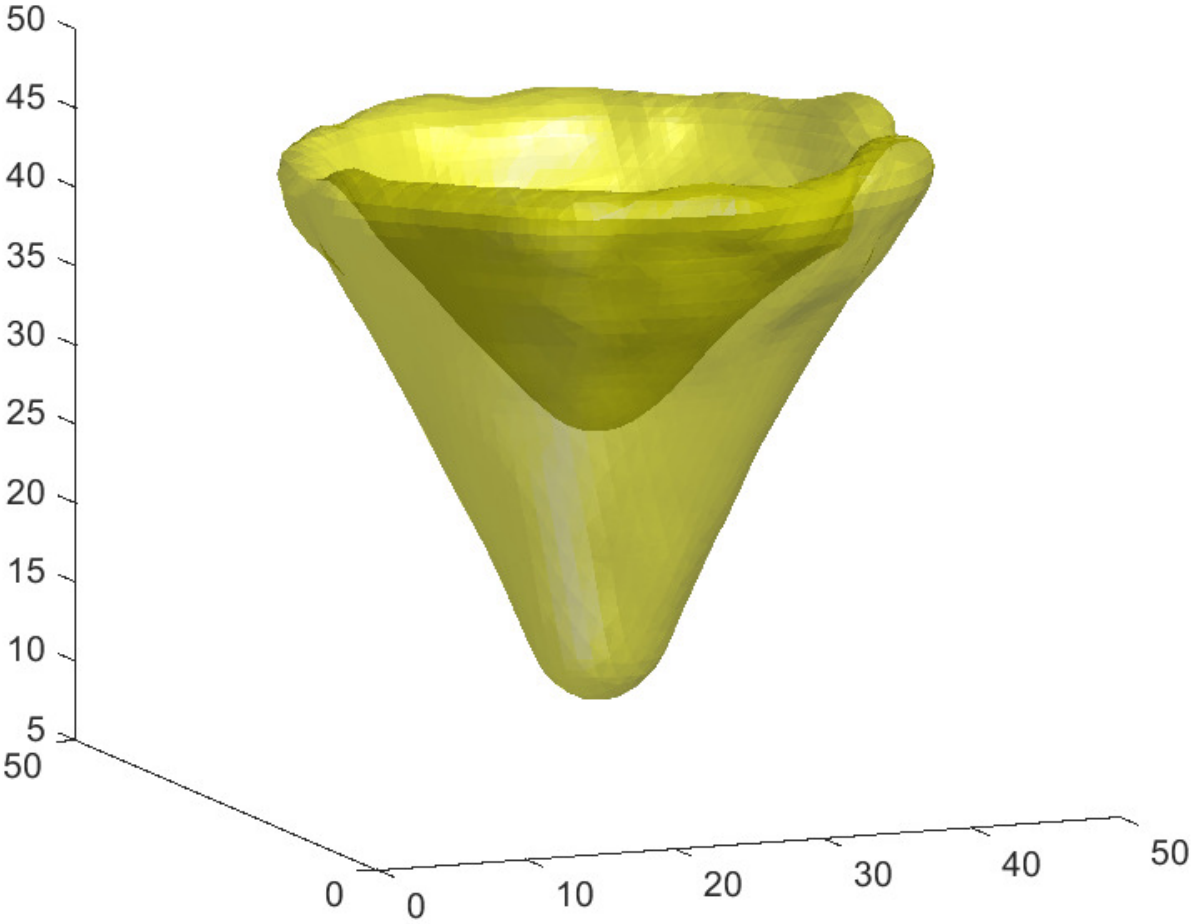}&
  \includegraphics[width=0.4\textwidth]{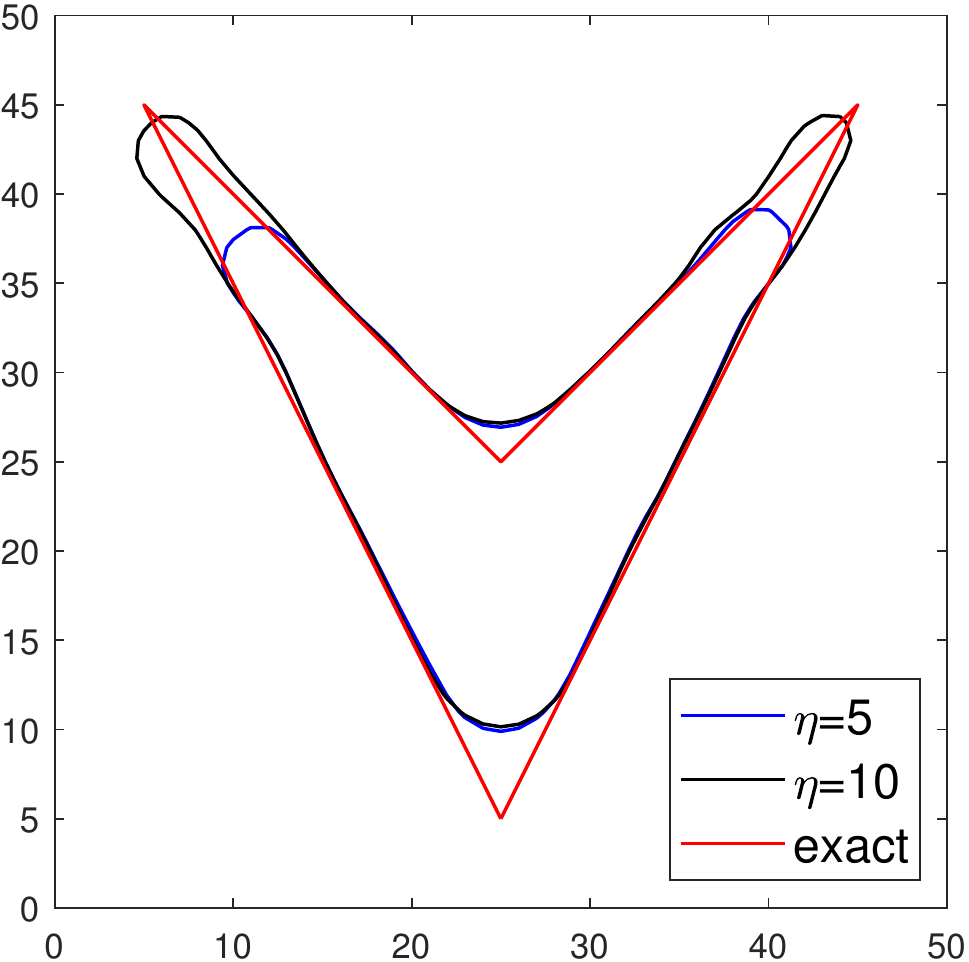}
  \end{tabular}
  \caption{Reconstruction of the ice cream cone by the algorithm from \cite{EZL*12} and OSM with $s=2$: (a) Result by the algorithm proposed in \cite{EZL*12} with $r_1=r_2=8, r_3=r_4=3$. (b)
  Result by OSM with $\eta=5$. (c) Result by OSM with $\eta=10$. (c) Comparison of cross sections of results by OSM along $y=25$.}\label{fig.ice}
\end{figure}

\section{Analytical Aspects of the Model} \label{sec:anal}

We consider the first variation of each term of the functional (\ref{eq.energy}).
The first variation of (\ref{eq.energy}) when $\eta=0$ is~\cite{zhao2000implicit}:
\begin{align}
\frac{1}{s}\big(\int_\Gamma d^s(\x)\,d\sigma\big)^{1/s-1}(sd^{s-1}\nabla d\cdot\mathbf{n}+d^s\kappa)\;,\label{eq.vard}
\end{align}
which shows the interaction between the data-dependent driving force, $d$, and the shape geometric feature, $\kappa$.  When $\Gamma$ is close to the point cloud, i.e., $d$ is small,  the shape of $\Gamma$ becomes flexible, i.e., $\kappa$ can be large. When $\Gamma$ is away from the point cloud, i.e., $d$ is large,  the shape of $\Gamma$ becomes rigid, and  $\kappa$ must be small.
For the regularization term, the effect of the mean-curvature $\kappa$ of $\Gamma$ is adjusted by the surface area. Notice that this term only focuses on the geometry of $\Gamma$, and the point cloud data
has no influence.
The first variation of the functional $\big(\int_\Gamma|\kappa(\mathbf{x})|^s\,d\sigma\big)^{1/s}$ can be derived as~\cite{droske2010higher}:
\begin{align}
\frac{1}{s}\big(\int_\Gamma\kappa(\mathbf{x})^s\,d\sigma\big)^{1/s-1}\times\begin{cases}
\text{div}_\Gamma(\delta(\kappa)\nabla_\Gamma\kappa)+\text{sign}(\kappa)|W|^2-\kappa|\kappa|&s= 1\\
\text{div}_\Gamma(|\kappa|^{s-2}s(s-1)\nabla_\Gamma\kappa)+s\kappa|\kappa|^{s-2}|W|^2-\kappa|\kappa|^s&s\geq 2.
\end{cases}	\label{eq.vark}
\end{align}
Here $\nabla_\Gamma$ is the tangent component of the gradient, and $\text{div}_{\Gamma}$ is its dual operator;  $W$ is the Weingarten map of $\Gamma$, and $|W|$ equals the  Gaussian curvature if $\Gamma$ is a 2D surface. Compared to (\ref{eq.vard}), the first variation related to the regularization term (\ref{eq.vark}) is more complicated. While the first variation (\ref{eq.vard}) connects the distance and curvature,  (\ref{eq.vark}) only depends on the geometric feature of the surface $\Gamma$.

When $s=1$ in the model (\ref{eq.energy}), we compare the cases with $\eta=0$ and $\eta\neq 0$.  We first note that $\kappa$ of the minimizer can not be constantly zero, since there is no compact minimal surface. When $\eta = 0$, i.e., without curvature constraint, a minimizer of (\ref{eq.energy}) satisfies the necessary condition: $\nabla d\cdot \mathbf{n}+d\kappa=\nabla\cdot(d\mathbf{n})=0$;  when $\eta\neq 0$, i.e., with the curvature constraint,  the optimality condition becomes  $\nabla\cdot(d\mathbf{n})+\eta[\text{div}_\Gamma(\delta(\kappa)\nabla_\Gamma\kappa)+ \text{sign}(\kappa) |W|^2-|\kappa|\kappa]=0$.   On the open subset of the minimizer where $\kappa>0$, this condition becomes $\nabla\cdot(d\mathbf{n})=\eta[\kappa^2-|W|^2]$, while on the region where $\kappa<0$, it is $\nabla\cdot(d\mathbf{n})=\eta[\kappa^2+|W|^2]$. Hence, the curvature regularization modifies the distance weighted area of the minimizing surface depending on the  local concavity/convexity and the Gaussian curvature. These modifications introduce more flexibility when fitting the point cloud, and our experiments show that they can help to improve reconstruction results.

When $s=2$, the first variation of the curvature regularization term $\left(\int_\Gamma\kappa^2 d\,\sigma\right)^{1/2}$ is
\begin{align}
\left(\int_\Gamma\kappa(\x)^2\,d\sigma\right)^{-1/2}\left(\Delta_\Gamma\kappa+\kappa G^2-\kappa^3/2\right)\;,\label{eq.vark2}
\end{align}
where $\Delta_\Gamma$ is the Laplace-Beltrami operator, and $G$ is the Gaussian curvature. Since (\ref{eq.vark2}) contains $\Delta_\Gamma\kappa$, we expect to see that our model with $s=2$ will be influenced by the locally averaged mean curvature, which leads to smoothing effects. Here we show this model's behavior in the following special case.
\begin{prop}\label{prop.s2}
Suppose the point cloud is sampled from a smooth closed surface $\Gamma$ with mean curvature $\kappa$ and Gaussian curvature $G$ satisfying:
\begin{align}
\kappa(G^2-\kappa^2/2)=0\;.\label{eq.cond1}
\end{align}
If the point cloud is sufficiently dense, i.e. the computed $d$ is very close to the exact distance function, then $\Gamma$ is a minimizer of (\ref{eq.energy}) only if it is a sphere of radius $\sqrt{2}$.
\end{prop}
\begin{proof}
Since $\Gamma$ passes through the point cloud, (\ref{eq.vard}) degenerates to $0$. Moreover,  by (\ref{eq.cond1}) together with (\ref{eq.vark}), the necessary condition for $\Gamma$ being a minimizer is that $\Delta_\Gamma\kappa=0$. Because $\Gamma$ is closed, $\kappa$ is a non-zero constant. By (\ref{eq.cond1}), this implies that $G$ is also constant; hence, we know that  $\Gamma$ can only be a sphere. Finally, since $G=1/r^2$ and $\kappa=1/r$ with $r$ the radius of the sphere, we can solve for the radius of $\Gamma$, which is $\sqrt{2}$.
\end{proof}
The curvature regularization term $\big(\int_\Gamma\kappa^2 d\,\sigma\big)^{1/2}$  inflates the membrane supported by the point cloud.  Mylar balloon~\cite{mladenov2007mylar}, which resembles slightly flattened sphere, satisfies the condition (\ref{eq.cond1}).  Proposition~\ref{prop.s2}  claims that, if the point cloud lies on a Mylar balloon,  the minimizer of the functional (\ref{eq.energy})  deviates from the underlying surface. 

The following result shows a two dimensional example  where the object is a circle, denoted by $C_0$, which is centered at the origin with radius $r_0$.   It is fair to assume that a local minimizer of $E_s(\Gamma)$ is a circle, denoted by $C$, with the same center and radius close to $r_0$. Denote the radius of $C$ by $r$. Then we have the following proposition:
\begin{prop}
  Under the setting of the above example, for $s=1$, $r=r_0$ is a local minimizer of $E_1(C)$ for any $\eta$. For $s=2$, $r=r_0$ is a local minimizer of $E_2(C)$ if $\eta\leq2r_0$ and $r=(r_0+\sqrt{r_0^2+12\eta})/6$ is a local minimizer if $\eta>2r_0$.
\end{prop}
\begin{proof}
  Note that for a circle with radius $r$ and the same center as $C_0$, $d=|r-r_0|$ and $\kappa=1/r$. $E_s(C)$ can be written as
\[
    E_s(C)=(2\pi)^{\frac{1}{s}} \left(|r-r_0|r^{\frac{1}{s}} +\eta r^{\frac{1}{s}-1}\right).
\]

  For $s=1$, $E_1(C)=(2\pi)^{\frac{1}{s}} \left(|r-r_0|r^{\frac{1}{s}} +\eta\right)$ of which $C_0$ is a local minimizer for any $\eta$.

  For $s=2$, $E_2(C)=(2\pi)^{\frac{1}{2}} \left(|r-r_0|r^{\frac{1}{2}} +\eta r^{-\frac{1}{2}}\right)$ whose subdifferential is
\[
  \partial E_2(C)=
    \begin{cases}
        (2\pi)^{\frac{1}{2}}\left(\frac{r_0}{2}r^{-\frac{1}{2}} -\frac{3}{2}r^{\frac{1}{2}} -\frac{\eta}{2}r^{-\frac{3}{2}}\right), & \mbox{ for } r<r_0,\\
        (2\pi)^{\frac{1}{2}}\left(\frac{3}{2}r^{\frac{1}{2}}-\frac{r_0}{2}r^{-\frac{1}{2}}-\frac{\eta}{2}r^{-\frac{3}{2}}\right), & \mbox{ for } r>r_0.
    \end{cases}
\]

  For $r<r_0$, it can be easily shown that if $12\eta\leq r_0^2$ and $(r_0+\sqrt{r_0^2-12\eta})/6<r<r_0$, $\partial E_2(C)<0$. If $12\eta>r_0$, $\partial E_2(C)<0$ for any $r<r_0$. In other words, if $r<r_0$ and $r$ is sufficiently close to $r_0$, $\partial E_2(C)<0$.

  For $r>r_0$, if $\eta\leq2r_0$, $\partial E_2(C)>0$ and thus $r=r_0$ is a local minimizer of $E_2(C)$. If $\eta>2r_0$, then $\partial E_2(C)<0$ for $r_0<r<(r_0+\sqrt{r_0^2+12\eta})/6$ and $\partial E_2(C)>0$ for $r>(r_0+\sqrt{r_0^2+12\eta})/6$. Thus $r=(r_0+\sqrt{r_0^2+12\eta})/6$ is a local minimizer of $E_2(C)$.
\end{proof}

These properties show that the minimizer of the model (\ref{eq.energy}) is not easy to be analyzed even in a simple case such as a circle, and the results heavily depend on the combination of $d$ and $\kappa$.   

\section{Conclusion}\label{sec:con}

In this paper, we explored the surface reconstruction models based on point cloud data with curvature constraints. The proposed model combines the distance from surface to the point cloud with a global curvature regularization. Introducing this high-order geometric information allows us to impose geometric features around the corners and to reconstruct concave  features of the point cloud better.
We find that the interactions between two terms of the functional are subtle.  For example, for $s=1$, since it allows sharp corners,  it  may be more relaxed around the corners and give shorter length reconstruction.  For the curvature term in model (\ref{eq.energy}), larger $s$ gives more weight to the part of the surface which has large curvature, i.e., corners. As a result, the corners of the reconstructed surface are smoothed and extruded out a little bit.
For a fast computation, instead of directly solving the complicated terms in its Euler-Lagrange equations, we use a new operator splitting strategy and minimize the energy by a semi-implicit scheme.  We also explore an augmented Lagrangian method, which has the advantage of having less parameters compared to other ADMM approaches.  Both methods produce reliable results in many cases, including those where the point cloud is noisy or sparse.
Comparison between OSM and ALM shows advances of OSM in flexibility, stability, and efficiency.   Comparing the results of OSM using $s=1$ and $s=2$, we find that OSM with $s=2$ provides better results when reconstructing features of the point clouds.
There are a number of extensions to be considered, including different curvature constraints and using additional information such as surface normal directions to facilitate the reconstruction.  Applications to segmentation and image inpainting can also be considered.

\section*{Acknowledgment}
The authors would like to thank Dr.Martin Huska at University of Bologna, Italy for the valuable discussions exploring different approaches of fast algorithm for high order functionals.

\bibliographystyle{abbrv}
\bibliography{cite_surfaceCurv}

\end{document}